\documentclass[10pt,a4paper]{article}  

\usepackage[utf8]{inputenc}
\PassOptionsToPackage{hyphens}{url}
\usepackage{xcolor}
\definecolor{uiolink} {HTML}{0B5A9D}
\usepackage[hidelinks, colorlinks=true, linkcolor=uiolink, citecolor=uiolink]{hyperref}
\usepackage{graphicx}
\usepackage{varioref, babel, listings, amsmath,amsthm, amssymb, mathtools}
\usepackage{array}
\usepackage{amsmath}	
\usepackage{amssymb}
\usepackage{rotating}
\usepackage{listings}
\usepackage{times}
\usepackage{enumitem}
\usepackage{cite}
\usepackage{amsfonts}
\usepackage{silence}
\usepackage{float}
\usepackage[normalem]{ulem}
\usepackage{csquotes}
\usepackage{tikz}
\usepackage{booktabs}
\usepackage{nameref}
\usepackage{algorithm}
\usepackage{algpseudocode}

\usepackage[left=1.15in,right=1.15in,top=1.15in,bottom=1.15in]{geometry} 

\usepackage{enumitem}

\usepackage[capitalize,nameinlink]{cleveref}[0.19]
\setlist[enumerate]{leftmargin=.5in}
\setlist[itemize]{leftmargin=.5in}

\renewcommand{\hat}{\widehat}
\renewcommand{\tilde}{\widetilde}

\Crefname{figure}{Figure}{Figures}

\crefformat{equation}{\textup{#2(#1)#3}}
\crefrangeformat{equation}{\textup{#3(#1)#4--#5(#2)#6}}
\crefmultiformat{equation}{\textup{#2(#1)#3}}{ and \textup{#2(#1)#3}}
{, \textup{#2(#1)#3}}{, and \textup{#2(#1)#3}}
\crefrangemultiformat{equation}{\textup{#3(#1)#4--#5(#2)#6}}
{ and \textup{#3(#1)#4--#5(#2)#6}}{, \textup{#3(#1)#4--#5(#2)#6}}{, and \textup{#3(#1)#4--#5(#2)#6}}

\Crefformat{equation}{#2Equation~\textup{(#1)}#3}
\Crefrangeformat{equation}{Equations~\textup{#3(#1)#4--#5(#2)#6}}
\Crefmultiformat{equation}{Equations~\textup{#2(#1)#3}}{ and \textup{#2(#1)#3}}
{, \textup{#2(#1)#3}}{, and \textup{#2(#1)#3}}
\Crefrangemultiformat{equation}{Equations~\textup{#3(#1)#4--#5(#2)#6}}
{ and \textup{#3(#1)#4--#5(#2)#6}}{, \textup{#3(#1)#4--#5(#2)#6}}{, and \textup{#3(#1)#4--#5(#2)#6}}

\crefdefaultlabelformat{#2\textup{#1}#3}

\title{Implicit regularization in AI meets generalized hardness of approximation in optimization -- Sharp results for diagonal linear networks}

\author{Johan S. Wind\thanks{Department of Mathematics, University of Oslo, Norway (\href{mailto:johanswi@math.uio.no}{johanswi@math.uio.no})} \and  Vegard Antun\thanks{Department of Mathematics, University of Oslo, Norway (\href{mailto:vegarant@math.uio.no}{vegarant@math.uio.no})} \and  Anders C.\ Hansen\thanks{Department of Applied Mathematics and Theoretical Physics, University of Cambridge, UK  (\href{mailto:ach70@cam.ac.uk}{ach70@cam.ac.uk})}}

\newcommand\N{\mathbb{N}}
\newcommand{\R}{\mathbb{R}}

\newcommand{\T}{^{\top}}
\newcommand{\PN}{P_{\Null(A)}}

\newcommand{\norm}[1]{\left\|#1\right\|}
\DeclareMathOperator*{\argmin}{argmin}
\DeclareMathOperator*{\argmax}{argmax}
\DeclareMathOperator{\rank}{rank}
\DeclareMathOperator{\Null}{\mathcal{N}}
\DeclareMathOperator*{\sign}{sign}

\DeclareMathOperator{\diag}{diag}
\DeclareMathOperator{\supp}{supp}

\makeatletter
\@namedef{subjclassname@2020}{
  \textup{2020} Mathematics Subject Classification}
\makeatother

\usepackage{color}

\DefineNamedColor{named}{Purple}{cmyk}{0.45,0.86,0,0}

\DefineNamedColor{named}{JungleGreen} {cmyk}{0.99,0,0.52,0}

\DefineNamedColor{named}{orange} {rgb}{1,0.55,0}

\theoremstyle{plain}
\newtheorem{theorem}{Theorem}[section]
\newtheorem{lemma}[theorem]{Lemma}
\newtheorem{proposition}[theorem]{Proposition}
\newtheorem{corollary}[theorem]{Corollary}
\newtheorem{claim}{Claim}

\theoremstyle{definition}
\newtheorem{definition}[theorem]{Definition}
\newtheorem{example}[theorem]{Example}
\newtheorem{remark}[theorem]{Remark}

\newtheorem{assumption}{Assumption}

\numberwithin{equation}{section}

\begin{document}

\maketitle

\begin{abstract}
Understanding the implicit regularization imposed by neural network architectures and gradient based optimization methods is a key challenge in deep learning and AI. In this work we provide sharp results for the implicit regularization imposed by the gradient flow of Diagonal Linear Networks (DLNs) in the over-parameterized regression setting and, potentially surprisingly, link this to the phenomenon of phase transitions in generalized hardness of approximation (GHA). GHA generalizes the phenomenon of hardness of approximation from computer science to, among others, continuous and robust optimization. It is well-known that the $\ell^1$-norm of the gradient flow of DLNs with tiny initialization converges to the objective function of basis pursuit. We improve upon these results by showing that the gradient flow of DLNs with tiny initialization approximates minimizers of the basis pursuit optimization problem (as opposed to just the objective function), and we obtain new and sharp convergence bounds w.r.t.\ the initialization size. Non-sharpness of our results would imply that the GHA phenomenon would not occur for the basis pursuit optimization problem -- which is a contradiction -- thus implying sharpness.  Moreover, we characterize {\it which} $\ell_1$ minimizer of the basis pursuit problem is chosen by the gradient flow whenever the minimizer is not unique. Interestingly, this depends on the depth of the DLN.

\end{abstract}

\setcounter{tocdepth}{2}
\tableofcontents
\section{Introduction}\label{s:intro}

During the past decade, deep learning has transformed a number of historically challenging problems in computer vision, natural language processing, game intelligence, etc. In many of these applications, the trained neural networks used to solve these problems are \emph{over-parameterized}. That is, the neural networks have far more parameters than the number of data points used for training. In this setting, a neural network can typically fit any training data -- including random labels \cite{zhang2017understanding} -- making it hard to explain why deep learning methods generalize so well \cite{DRnonvacuous17}. Moreover, the practical performance of neural networks often improves as the number of parameters grow \cite{kaplan2020scaling, tan2019efficientnet}. These observations have led to the study of the potential implicit regularization (sometimes called implicit bias) imposed by the gradient based methods and different network architectures \cite{arora2019implicit, neyshabur2017geometry, neyshabur2014search}. 

It may seem surprising that there is a link to \emph{generalized hardness of approximation} (GHA), as this phenomenon -- at a first glance -- may seem disconnected from implicit regularization. However, the GHA phenomenon (see \S \ref{sec:gha}), which first appeared in \cite{opt_big} (see also \cite{CSBook} Chapter 8) and analyzed \cite{opt_big, gazdag2022generalised, comp_stable_NN22} in connection with robust and convex optimization
\cite{NemirovskiLRob, Nemirovski_NPhard_Stable, Nemirovski_robust, Nemirovski_robust2}, typically stem from regularization problems (e.g. basis pursuit, Lasso, nuclear norm minimization etc.) Thus, after a second look, it seems natural that there is a link to implicit regularization. Indeed, while the current literature is far from being able to characterize the implicit regularization imposed by deep learning in general, there has recently been made considerable progress on a number of classical problems in scientific computing. For example, there has been a substantial interest in different flavours of matrix factorization/completion \cite{arora2018optimization, arora2019implicit, chou2020gradient, geyer2020low, gissin2019implicit, gunasekar2018implicit, gunasekar2017implicit, neyshabur2017geometry, neyshabur2014search, razin2020implicit, razin2021implicit, razin2022implicit, soudry2018implicit, stoger2021small, bah2022learning}, and (sparse) linear regression \cite{arora2018optimization, chou2021more, even2023s, gissin2019implicit, gunasekar2018implicit,  gunasekar2017implicit, li2021implicit, moroshko2020implicit, pesme2020online, pesme2021implicit,  pesme2023saddle, vaskevicius2019implicit, woodworth2020kernel}. These are all areas where the phenomenon of GHA occurs. 

In this paper, we determine the implicit regularization imposed by gradient descent/flow for so-called \emph{diagonal linear networks} in the over-parameterized regression setting. Moreover, we show that our results imply the existence of diagonal linear networks which can approximate solutions to certain $\ell^1$-regularized optimization problems to arbitrary accuracy. However, the phenomenon of GHA \cite{opt_big} (see also \cite{comp_stable_NN22})  for these $\ell^1$-regularized optimization problems will -- in general -- prevent the existence of any algorithm that can compute these neural networks. Thus, paradoxically, one can prove that the implicit regularization of certain deep learning methods yield solutions that cannot be computed by algorithms. This phenomenon is similar in spirit to hardness of approximation in computer science \cite{arora2009computational}, however, is based on analysis rather than discrete mathematics. 

\subsection{Analyzing the implicit regularization of linear networks}

A linear neural network is a function $\Psi_{\theta}\colon \R^N\to \R$ on the form 
\begin{equation}
    \Psi_{\theta}(x) = \psi_{\theta} \cdot x, 
\end{equation}
where $\cdot$ denotes the standard real vector inner product, and $\psi_{\theta}\in \R^N$ is a vector whose components are parametrized by $N$ parameters $\theta = (\theta_1,\ldots,\theta_N)$.  A parametrization for $\psi_{\theta}$ could, for example, be $\psi_{\theta} = \theta$ for $\theta \in \R^N$ \cite{pesme2020online}, or $\psi_{\theta} = \theta_{+}^{p} -\theta_{-}^{p}$, where $\theta = (\theta_+,\theta_-) \in \R^{2N}_{\geq 0}$, and the power $p>0$ notation means componentwise action on the vectors $\theta_+$ and $\theta_-$ \cite{ gunasekar2017implicit,  li2021implicit, moroshko2020implicit, pesme2021implicit, vaskevicius2019implicit,woodworth2020kernel}. However, the above mentioned  parametrizations are not the only choices used in the literature \cite{chou2021more, even2023s, pesme2023saddle}.

In this work, we analyze different classes of linear neural networks for regression problems with training data $\{(a_1,y_1), \ldots, (a_m,y_m)\}\subset \R^N \times \R$ and a squared error loss functional. For this problem one can write the loss function for a given network parametrization $\psi_{\theta}$ as 
\begin{equation}\label{eq:flow1}
L(\theta) = \frac{1}{2} \sum_{i=1}^{m} (\psi_{\theta}\cdot a_i - y_i)^2 = \frac{1}{2} \|A\psi_{\theta} - y\|_{2}^{2},
\end{equation}
where $A$ is the $m\times N$ matrix whose $i$'th row is $a_i$ and $y$ is the vector $(y_1,\ldots,y_m)$. We assume that the size of the training data $m < N$ and that the $\mathrm{rank}(A) = m$, i.e., that the training data is linearly independent. With these assumptions, the linear system $Ax =y$ has infinitely many solutions and the challenge is to determine the solutions found when minimizing $L$ using gradient-based methods. The solutions sought with these methods generally depend on the parametrization of $\psi_{\theta}$ and the initialization of $\theta$. As we want to understand the role of over-parametrization for these models, we focus on the setting with $d\geq N > m$ parameters.

A much used approach for analyzing the implicit regularization imposed by different parametrizations of $\psi_{\theta}$, is to formulate the problem of minimizing $L$ in \eqref{eq:flow1} as a gradient flow problem. That is, one considers the dynamical system 
\begin{equation} \label{eq:grad_flow1}
  \dot w(t) = -\nabla_{\theta} L(w(t)), \quad\text{with initial data }\quad w(0) = w_0 \in \R^d.
\end{equation}
Here $\nabla_{\theta} L$ denotes the gradient of $L$ with respect to $\theta$ for a given parametrization of $\psi_{\theta}$ and the function $w \colon [0,\infty) \to \R^d$ denotes the flow of the parameters. We use dot notation $\dot w$ for the componentwise derivative of $w$. It is common to view \eqref{eq:grad_flow1} as a continuous counterpart to gradient descent, as we recover the gradient descent steps if we discretize the above equation using the forward Euler method. A key motivation for \eqref{eq:grad_flow1} is that it is often easier to study than its discrete counterpart.

\subsection{Generalized hardness of approximation -- Phase transitions}\label{sec:gha}

\emph{Hardness of approximation} \cite{arora2009computational, arora1998proof, bellare1998free, feige1996interactive, johan1999clique, haastad2001some, sudan2009probabilistically} is a phenomenon in computer science whose discovery in the 1990s lead to a highly active research program yielding several Gödel and Nevanlinna Prizes. It can roughly be described as follows (subject to  $\mathrm{P}\neq \mathrm{NP}$): Given a combinatorial optimization that may be NP-hard,  one may still -- in polynomial time (denoted by P) -- compute an $\epsilon$-approximate solution to this problem for any $\epsilon >\epsilon_0$. However, for any $\epsilon < \epsilon_0$, there does not exist any algorithm (i.e., Turing machine) that can compute such an approximation in polynomial time. We say that the problem has a phase transition at $\epsilon_0$: 

\begin{tabular}{@{}>{\centering}m{0.32\textwidth}>{\centering\arraybackslash}m{0.65\textwidth}@{}}
        \textbf{
			Classical phase
			transition  at $\epsilon_0$ in hardness
            of approximation} \newline
        (Assuming $\mathrm{P}\neq \mathrm{NP}$)
        \vspace{0.5cm}
        &
\begin{tikzpicture}[xscale=1.3]
\draw [thick,->] (3,0) -- (9.3,0);
\draw (3,-.2) -- (3,.2);
\draw (6.5,-.2) -- (6.5, 1.4);
\node[align=center, below] at (3,-.2) {$0$};
\node at (9.5,0) {$\epsilon$};
\node[align=center, below] at (6.5,-.2) {$\epsilon_0$};
\node[align=center, above] at (7.7,0.0) 
{$\epsilon > \epsilon_0:$ \\ Computing \\ $\epsilon$-approx $\in$ P};
\node[align=center, above] at (4.8,0.0)
{$\epsilon < \epsilon_0:$ \\ Computing \\ $\epsilon$-approx is $\notin$ P};
\end{tikzpicture}
\end{tabular}
The phenomenon of hardness of approximation, hinges on the assumption that $\mathrm{P}\neq \mathrm{NP}$. However, if it turns out that $\mathrm{P}=\mathrm{NP}$, then the phenomenon may cease to exist in many cases, and one can design algorithms which in polynomial time can compute approximations to any accuracy.

The phenomenon of GHA
\cite{opt_big, CSBook, gazdag2022generalised, comp_stable_NN22} 
 is similar in spirit to hardness of approximation, however it is in general independent of whether or not $\mathrm{P}=\mathrm{NP}$.  The phenomenon is more general than hardness of approximation, in that it is not just centered around the complexity of a computation in a Turing model (see \Cref{breakdown}), but is about arbitrary classes of computational problems in any model. For example, one may ask whether or not we can compute $\epsilon$-approximate solutions to certain computational problems at all. It turns out that for certain computational problems, the answer depends on the accuracy sought. In particular, for these problems, there exists an $\epsilon^{\mathrm{s}} > 0$, such that no algorithm can compute an $\epsilon$-approximation for $\epsilon < \epsilon^{\mathrm{s}}$. However, for $\epsilon > \epsilon^{\mathrm{s}}$ computing such an approximation is possible (even quickly). Schematically, we can view this as a phase transition as well: 

\begin{tabular}{@{}>{\centering}m{0.32\textwidth}>{\centering\arraybackslash}m{0.65\textwidth}@{}}
 \textbf{Phase transitions at $\epsilon^{\mathrm{s}}$ for generalized hardness of approximation}
        \vspace{0.5cm}
        &
\begin{tikzpicture}[xscale=1.3]
\draw [thick,->] (0,0) -- (6.3,0);
\draw (0,.2) -- (0,-.2) node[below] {$0$};
\draw (3.5,1.4) -- (3.5, -0.2) node[below] {$\epsilon^{\mathrm{s}}$};
\node at (6.5,0) {$\epsilon$};
\node[align=center, above] at (5.0,0.0) 
{\parbox{0.25\textwidth}{\center $\epsilon > \epsilon^{\mathrm{s}}:$ \\  Computing \\ $\epsilon$-approx $\in S_1$ }};
\node[align=center, above] at (1.75,0.0)
{\parbox{0.25\textwidth}{\center $\epsilon < \epsilon^{\mathrm{s}}:$ \\  Computing \\ $\epsilon$-approx $\in S_2$ }};
\end{tikzpicture}
\end{tabular}
For example, one could have 
$
S_1 = P
$
(polynomial solvable)
and
$
S_2 = P^c,
$
as in classical hardness of approximation, or for example
\[
S_1 = P \text{ (polynomial solvable)}, \qquad S_2 = \text{ non-computable}.
\]

\begin{remark}[Strong and weak breakdown epsilons]\label{breakdown}
If $S_1$ denotes the set of computable problems, the $\epsilon^{\mathrm{s}}$ at the center of the above phase transition is called the \emph{strong breakdown epsilon}. In \cite{opt_big}, a weaker version called the \emph{weak breakdown epsilon}, is also considered, that takes into account the runtime of the algorithms. We do, however, not consider this weaker version in this manuscript. 
\end{remark}

A result that will be important to us in what follows is the following theorem which is a very specialized case of the result in \cite{opt_big}. The precise statement can be found in \cref{ss:Ex_non_comp}.

\begin{theorem}[Generalized hardness of approximation and phase transitions in basis pursuit]\label{GHA:BP}
Let $K$, $N$ and $m$ be integers, with $N \geq m$, $N \geq 2$, and let $\kappa = 10^{-K}$. Consider the optimization problem 
\begin{equation}\label{eq:BP123}
\min_{x\in \R^N} \|x\|_{1} \text{ subject to } Ax=y
\end{equation}
where $A \in \R^{m\times N}$, $y \in  \R^m$ and $\|\cdot\|_1$ denotes the $l^1$-norm. For the problem \eqref{eq:BP123}, there exists a class of inputs $\Omega_{K} \subset \R^{m\times N}\times \R^m$ of computable inputs for which the following hold simultaneously (where accuracy
is measured in the Euclidean norm):
\begin{enumerate}[label=(\roman*)]
    \item For $\epsilon < \kappa$, no algorithm can produce an $\epsilon$-approximate solution to a minimizer of \eqref{eq:BP123} for all inputs $(A,y) \in \Omega_{K}$.
    \item For $\epsilon > \kappa$, there exists an algorithm which can compute an $\epsilon$-approximate solution to a minimizer of \eqref{eq:BP123} for all inputs $(A,y) \in \Omega_{K}$.
\end{enumerate}
\end{theorem}

\begin{remark}[Generalized hardness of approximation in the sciences]
Note that GHA occurs in basis pursuit (with noise)  in the basic settings of compressed sensing -- e.g. when $A$ satisfies the robust nullspace property, see \cite{opt_big}). In these cases the phase transition happens typically at $K = \log(\delta^{-1})$, where $\delta >0$ is the noise parameter.  
\end{remark}

\section{Main results}
In this work we focus on network parametrizations of the form \begin{equation}\label{eq:param}
    \psi_{\theta} = \theta_{+}^{p} - \theta_{-}^{p}, \, \text{ where } \, \theta =(\theta_+,\theta_-) \in \R^{2N}_{\geq 0},
\end{equation}
and where $\theta_{\pm}^{p} = (\theta_{\pm,1}^{p}, \ldots, \theta_{\pm,N}^{p})$ is the vector where each component is raised to the power of $p \geq 2$. This parametrization is called a diagonal linear network, a name which is inspired from matrix factorization (see, e.g., \cite[Sec.\ 4]{woodworth2020kernel} for more on this connection).
Furthermore, we follow the gradient flow approach in \eqref{eq:grad_flow1}, with initial data $w_0 = \alpha {\bf 1} \in \R^{2N}$, where $\alpha > 0$ and ${\bf 1} = (1,\ldots, 1)$. For convenience, we abuse notation slightly and denote the dynamical system by   
\begin{equation} \label{eq:grad_flow}
  \dot \theta(t) = -\nabla L(\theta(t)), \quad\text{with}\quad \theta(0) = \alpha {\bf 1}_{2N},
\end{equation}
where the weights $\theta$ are now vector-valued functions of time $t\geq 0$. The model in \eqref{eq:param} with initialization $\alpha {\bf 1}$ has been studied many places in the literature \cite{gunasekar2017implicit, li2021implicit, moroshko2020implicit, pesme2021implicit, vaskevicius2019implicit, woodworth2020kernel}. Below the main theorems and in \Cref{ss:rel_work}, we expand more on how this relates to our work.   

The choice of $p$ in \eqref{eq:grad_flow} will be clear from the context, but the size of the initialization $\alpha > 0$ will be important to us. We, therefore, denote the solution vector at time $t$ by 
\begin{equation}\label{psi_def}
    \psi_{\alpha}(t) = \theta_{+}^{p}(t) - \theta_{-}^{p}(t), 
\end{equation}
where $\psi_{\alpha}(t)$ is a slight abuse of notation for $\psi_{\theta_\alpha}(t)$, as $\theta(t)$ depends on $\alpha$ according to \eqref{eq:grad_flow}.
We let $\psi_{\alpha}(\infty) \coloneqq \lim_{t\to\infty} \theta_{+}^{p}(t) - \theta_{-}^{p}(t)$ denote the solution vector at convergence. 
Next, consider the basis pursuit optimization problem \eqref{eq:BP123} from \Cref{GHA:BP}, and let 
\begin{equation}\label{eq:BP_argmin}
  \mathcal{U} = \mathcal{U}(A,y) = \argmin_{z \in \R^N} \|z\|_{1} \quad\text{subject to}\quad Az=y
\end{equation}
denote its set of minimizers, and let 
\begin{equation}\label{eq:BP_min}
R = R(A,y) = \min_{z \in \R^N} \|z\|_1 \quad\text{subject to}\quad Az=y
\end{equation}
denote its minimum value.
 We note that the solution set $\mathcal{U}$ always is non-empty, since the assumption $\rank (A) = m \le N$ implies that the set of feasible points is non-empty.

To determine the specific element that $\psi_{\alpha}(\infty)$ approximates in $\mathcal{U}$, we need some more notation.  
For non-negative vectors $z \in \R^{N}_{\geq 0}$, let 
\begin{equation}\label{eq:H_def}
H(z) = -\sum_{i = 1}^N |z_i| \ln (|z_i|), \quad \text{with the (usual) convention that}\quad 0\, \ln(0) \coloneqq 0,
\end{equation}
denote the entropy function of $z$. Moreover, for any $r > 0$ and $z\in \R^N$, we let $\norm{z}_{r} = (\sum_{i = 1}^N |z_i|^r)^\frac{1}{r}$. For completeness, we note that when $r \in (0,1)$, then $\|\cdot\|_r$ is a quasinorm with constant $2^{1/r -1}$, and for $r \geq 1$ this is the usual $\ell^r$-norm. In our main result, we shall see that the gradient flow of \eqref{eq:grad_flow} can get arbitrarily close to a unique minimizer of \eqref{eq:BP_argmin}, by choosing $\alpha > 0$ sufficiently small. The selected minimizer will be given by 
\begin{equation}\label{eq:W_p}
\mathcal{W}_p(A,y) = 
    \begin{cases}
      \argmax_{z \in \mathcal{U}(A,y)} H(z) &\text{ if } p = 2\\
      \argmax_{z \in \mathcal{U}(A,y)} \norm{z}_{2/p} &\text{ if } p \in (2,\infty)
    \end{cases}.
\end{equation}
Uniqueness of $\mathcal{W}_p$ is covered in \Cref{Wp_Gp_id}.

\subsection{The implicit regularization of gradient flow meets GHA}
Our first main result, characterizes the implicit regularization of the gradient flow of \eqref{eq:grad_flow} for the parametrization \eqref{eq:param} with tiny initialization. It also connects this flow to the phenomenon of GHA. 

\begin{theorem}\label{thm:grad_flow}
    Let 
    \[\Omega = \{(A,y) \in \R^{m\times N}\times \R^m : \mathrm{rank}(A)=m, y \neq 0\}.\]
We then have the following.
    \begin{enumerate}[label=(\roman*)] 
        \item \label{it:dep_bound} Let $(A,\cdot) \in \Omega$ and let $p \in [2,\infty)$. Then, there exist constants $C_1, C_2 > 0$, depending on $A$ and $p$, such that for any non-zero $y \in \R^m$ and any $\alpha > 0$, we have
            \begin{equation} \label{e:pe2}
    \norm{\psi_{\alpha}(\infty)-\mathcal{W}_p(A,y)}_2 \le C_1\norm{y}_2\left(\frac{\alpha^p}{\norm{y}_2}\right)^{C_2}.
            \end{equation}
            If $p > 2$, then $C_2 = 1$.
    \item\label{it:forall}  Suppose that there exist $p\geq 2$ and universal constants $C_1, C_2 > 0$ such that 
            \begin{equation} \label{e:forall}
    \norm{\psi_{\alpha}(\infty)-\mathcal{W}_p(A,y)}_2 \le C_1\norm{y}_2\left(\frac{\alpha^p}{\norm{y}_2}\right)^{C_2}, \quad\forall (A,y) \in \Omega.
            \end{equation}
Then, for any $\epsilon > 0$, there exists an algorithm $\Gamma_{\epsilon}$ that, given input $(A,y) \in \Omega$, computes solutions to the basis pursuit optimization problem \eqref{eq:BP_argmin} to accuracy $\epsilon$. That is 
        \[\inf_{x^* \in\mathcal{U}(A,y)}\|\Gamma_{\epsilon}(A,y) - x^*\|_2 \leq \epsilon, \quad \forall (A,y) \in \Omega.\]
            However, this contradicts the phenomenon of generalized hardness of approximation for the basis pursuit problem (\Cref{GHA:BP}). We conclude that there cannot exist any $p\geq 2$ and universal constants $C_1,C_2> 0$ for which \eqref{e:forall} holds. 
    \end{enumerate}
\end{theorem}

The above theorem extends the current literature in the following ways. 

\begin{enumerate}[label=(\Roman*), leftmargin=2em]
    \item \emph{We bound the gradient flow of $\psi_{\alpha}(\infty)$ against a distinct element in $\mathcal{U}(A,y)$.}

      It is well-known that diagonal linear networks $\psi_{\theta}= \theta_{+}^{p} - \theta_{-}^{p}$ with tiny uniform initialization $\theta = \alpha {\bf 1}$ yield solutions with small $\ell^1$-norm, i.e., that $\|\psi_{\alpha}(\infty)\|_1 - R$ is small, where $R$ is defined in \eqref{eq:BP_min} (we highlight \cite{woodworth2020kernel} and \cite{gunasekar2017implicit}[Cor.\ 2]). This parameterization is standard in the literature, although some works consider other equivalent parameterizations \cite{even2023s, chou2021more}, see \Cref{ss:rel_work} for more details.

      We stress that all results in the literature on implicit regularization prove bounds on the distance $\|\psi_{\alpha}(\infty)\|_1 - R$ for small values of $\alpha$. That is, convergence towards the minimum value of the basis pursuit problem. This is not the same as bounding the distance to any of the minimizers of this optimization problem. Indeed, if $x^* \in\mathcal{U}(A,y)$ is a minimizer of \eqref{eq:BP123} and $x$ is any vector satisfying $Ax=y$ and 
      \[ \|x\|_1 - \|x^*\|_1, \leq \epsilon \quad \text{does not imply that} \quad \|x-x^*\|_1 \leq \epsilon. \]
      That is, closeness to the minimum value (i.e., what is shown in \cite{woodworth2020kernel}) does not imply that one is close to a given minimizer. See, e.g., \cite[Lem.\ 3.10 in SI]{comp_stable_NN22}) for an example. A key technical contribution of this work is to introduce a new implicit regularizer $g_p$ (see \Cref{ss:prelim}), which allows us to bound the distance to a distinct minimizer of $\mathcal{U}(A,y)$. Interestingly, the choice of minimizer depends on the depth $p$ of the network.

    \item \label{conv_rate1}\emph{We improve upon the best known convergence rates for $\psi_{\alpha}(\infty)$ as $\alpha \to 0$. In particular we show that $\psi_{\alpha}(\infty)$ converges  at the rate $\mathcal{O}(\alpha^p)$ when $p > 2$, and a rate depending on $A$ when $p = 2$.}

      This improves on the known bounds on $\|\psi_{\alpha}(\infty)\|_1 - R$ with rate $\mathcal{O}\left(\frac{1}{\log\left(1 / \alpha\right)}\right)$ for $p = 2$ \cite{woodworth2020kernel,chou2021more} and $\mathcal{O}\left(\alpha^{p-2}\right)$ for $p \ge 3$ \cite{chou2021more}. Moreover, in \Cref{fig:positive_results} one can see numerical approximations to $\psi_{\alpha}(\infty)$, which confirms that the error decays at the rate $\alpha^p$ for different values of $p$. See \Cref{s:numerics} for more on how the numerical approximation is computed. We note that most previous works \cite{azulay2021implicit,gunasekar2017implicit,woodworth2020kernel} employ techniques which do not give rates of convergence, usually by checking the KKT optimality conditions for the basis pursuit problem in the limit $\alpha \to 0$.

    \item \emph{Utilizing well-established results in generalized hardness of approximation, we show that the dependence on $A$ in \ref{conv_rate1} is sharp for matrices $A$ with full rank. }

      Note that a consequence of this connection is -- informally speaking -- that a gradient based method (or any other algorithm) will for certain inputs not be able to compute an approximation to any solution of the basis pursuit problem beyond a certain accuracy. In particular, the implicit regularization of $\psi_{\theta} = \theta_{+}^{p} + \theta_{-}^{p}$, might converge towards a vector with small $\ell^1$-norm for small values of $\alpha$, but this vector might not be a minimizer of the basis pursuit optimization problem. This sheds new light on the common belief in the literature that these vectors are solutions of this optimization problem.

The dependence on $A$ in the constants $C_1$ and $C_2$, and subsequent blowup for certain $A$s, is closely related to a constant which we denote by $\chi_{B(A)}$, where $B(A)$ is a matrix that depends on the nullspace of $A$ (see \Cref{s:cond} for details). The constant $\chi_{B(A)}$, which appears in our expressions for $C_1$ and $C_2$, has appeared many places in the literature, for example, in linear programming problems \cite{todd1990dantzig, vavasis1996primal}, certain norms for the pseudo inverse \cite{Stewart89} and for differential equations \cite{vavasis1996stable, vavasis1994stable}.

\end{enumerate}

\begin{figure}
    \begin{center}
    
		\setlength{\tabcolsep}{3pt} 
            \begin{tabular}{@{}>{\footnotesize\centering}m{0.04\textwidth}>{\footnotesize\centering}m{0.30\textwidth}>{\footnotesize\centering}m{0.30\textwidth}>{\footnotesize\centering\arraybackslash}m{0.30\textwidth}@{}}   
                & $A_1$ is a $3\times 5$ matrix & $A_2$ is a $2\times 3$ matrix & $A_3$ is a $2\times 4$ matrix \\           
                \begin{turn}{90}\begin{footnotesize}$\quad\left\|\psi_{\alpha}(\infty) - \mathcal{W}_p(A,y)\right\|_2$\end{footnotesize}\end{turn}
                &\includegraphics[width=\linewidth]{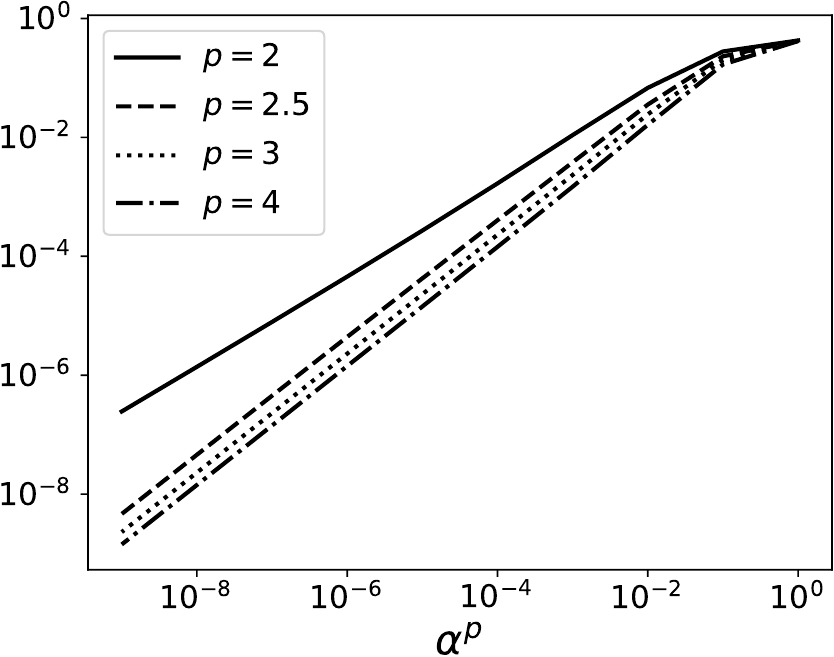} 
                &\includegraphics[width=\linewidth]{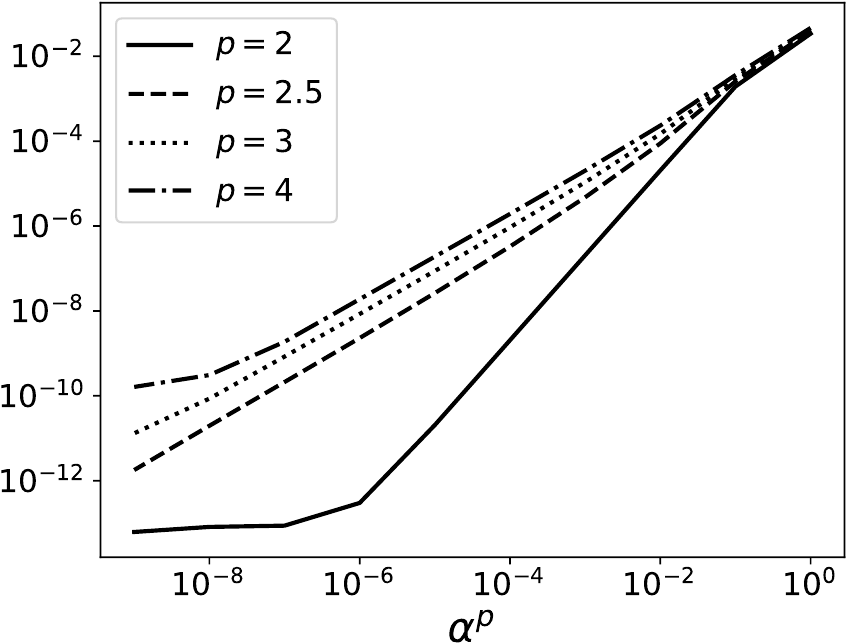}
                &\includegraphics[width=\linewidth]{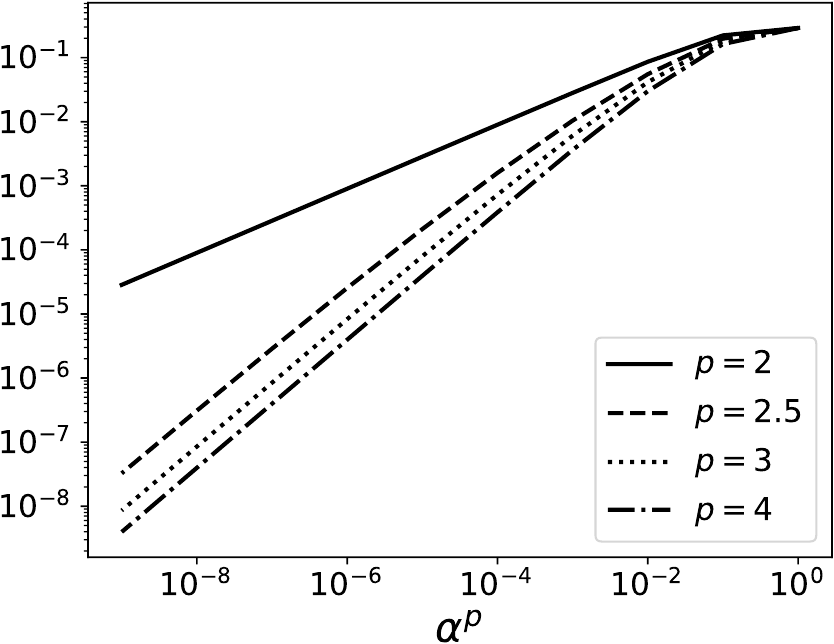}
            \end{tabular}
    \end{center}
    \vspace{-5mm}
    \caption{\label{fig:positive_results} We solve the gradient flow problem \eqref{eq:grad_flow} for different matrices $A_i$ and vectors, and different choices of $\alpha>0$ and $p\geq 2$. As we can see in the figure, the error scales with $\alpha$ according to the error bounds described by \Cref{thm:grad_flow}.}
\end{figure}

\begin{remark}[What if $\mathrm{rank}(A) < m$?]
  In the above theorem we assume that $\mathrm{rank}(A) = m$. If this is not the case, one can still consider the gradient flow of a given pair $(A,y)$, where $\mathrm{rank}(A) < m$. In this case, however, $\mathcal{W}_{p}(A,y)$ is not defined for $y \not\in \mathrm{range}(A)$ as the basis pursuit problem does not have feasible points. In these cases, the vector $\psi_{\alpha}(\infty)$ converges to a vector $\mathcal{W}_p(\tilde A,\tilde{y})$, where the matrix $\tilde A$ and vector $\tilde y$ are simple functions of $A$ and $y$. The precise formula for the pair $(\tilde A, \tilde y)$ can be found in \Cref{rank_A_l_m}.
\end{remark}

\subsection{Connecting the gradient flow to discrete computations}

In \Cref{thm:grad_flow} \ref{it:dep_bound}, we derive convergence estimates for the gradient flow $\psi_{\alpha}$ as $t \to \infty$. However, in practice we need to discretize the gradient flow, and run for a finite time $t$. Our next result -- which largely is used as a stepping stone for proving \Cref{thm:grad_flow} --  addresses these gaps.
\begin{theorem}\label{prop:main1}
    Let $A \in \R^{m\times N}$ with $\mathrm{rank}(A) = m$ and let $y \in \R^m$ be non-zero. Consider $p\geq 2$ and $\alpha > 0$.
    \begin{enumerate}[label=(\roman*)]
    \item\label{it:m1} There exist constants $K_1,K_2 > 0$, depending on $N,p,\alpha, y$ and $A$, such that for all $t\geq 0$, we have
    \begin{align*}
      \norm{\psi_{\alpha}(t)-\psi_{\alpha}(\infty)}_2 \le K_1\exp(-K_2t).
    \end{align*}
  \item\label{it:m2} Let $\widehat \theta_n = (\widehat \theta_{+,n}, \widehat \theta_{-,n})$ denote the $n$th step of gradient descent of \eqref{eq:flow1} with constant steplength $\eta>0$, and initial value  $\theta_{0}=\alpha{\bf 1}$ and parametrization as in \eqref{eq:param}. Let $t, \epsilon > 0$. Then there exists a constant $K_3 > 0$, depending on $N,p,\alpha,\epsilon, y$ and $A$, such that when $\eta \le K_3$ and $n = \lfloor t / \eta \rfloor$, we have
    \[\|\hat\psi_n - \psi_{\alpha}(t)\|_2 \le \epsilon,\]
    where $\hat\psi_n \coloneqq |\hat\theta_{+, n}|^p-|\hat\theta_{-, n}|^p$ is the solution vector computed with gradient descent.  
    \end{enumerate}
\end{theorem}
Formulas for $K_1,K_2$ and $K_3$ can be found in \Cref{psi_convergence,prop:grad_desc}.
The contributions of this theorem are as follows.

\begin{enumerate}[label=(\Roman*), leftmargin=2em]
\setcounter{enumi}{3}
    \item \emph{We show that $\lim_{t \to \infty}\psi_\alpha(t)$ convergences at an exponential rate.}

      It is standard in the literature to assume that $A\psi_\alpha(\infty) = y$ \cite{woodworth2020kernel,gunasekar2017implicit,azulay2021implicit}, and then prove results utilizing this assumption. To the best of our knowledge, only \cite{chou2021more} proved convergence of $\lim_{t \to \infty}\psi_\alpha(t)$ in a setting similar to ours. However, in \cite{chou2021more} one only proves existence of the limit, giving no rate of convergence.

    \item \emph{We discretize the gradient flow and show that if we choose the step length sufficiently small we can compute an $\epsilon$-approximation to the gradient flow.} 

    Note that a closer look at the constants $K_1$, $K_2$ and $K_3$ will reveal that these constants are straightforward to compute for given values of $N,p,\alpha, \epsilon$ and data $(A,y)$. Thus, the contribution of \Cref{prop:main1} \ref{it:m2} is that it allows us to design an algorithm which can approximate the gradient flow solution $\psi_{\alpha}(\infty)$ to any given accuracy $\tilde \epsilon > 0$, by choosing $t$ sufficiently large and $\eta > 0$ sufficiently small. This is a crucial step towards proving the impossibility result in \Cref{thm:grad_flow} \ref{it:forall}, which uses this algorithm to reach the contradiction. That being said, the proof \Cref{prop:main1} \ref{it:m2} largely utilizes standard results in the literature for initial value problems \cite{hairer1993solving, iserles2009first}.

\end{enumerate}

\begin{figure}[t]
    \begin{center}
		\setlength{\tabcolsep}{1pt} 
            \begin{tabular}{@{}>{\footnotesize\centering}m{0.04\textwidth}>{\footnotesize\centering}m{0.30\textwidth}>{\footnotesize\centering}m{0.30\textwidth}>{\footnotesize\centering\arraybackslash}m{0.30\textwidth}@{}}   
                & $p=3$ & $p=4$ & $p=5$\\
                \begin{turn}{90}\begin{footnotesize}$\quad\left\|\psi_{\alpha}(\infty) - \mathcal{W}_p(A,y)\right\|_2$\end{footnotesize}\end{turn}
                &\includegraphics[width=\linewidth]{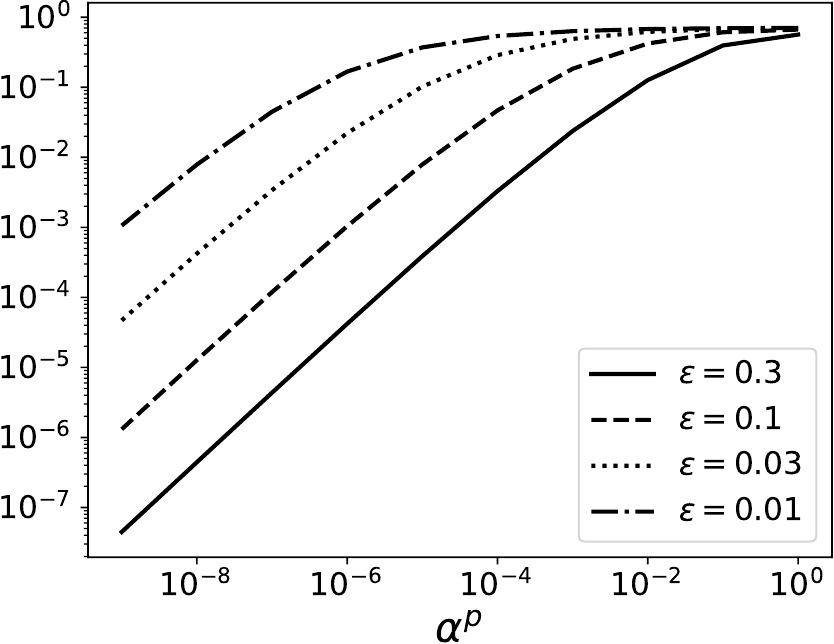} 
                &\includegraphics[width=\linewidth]{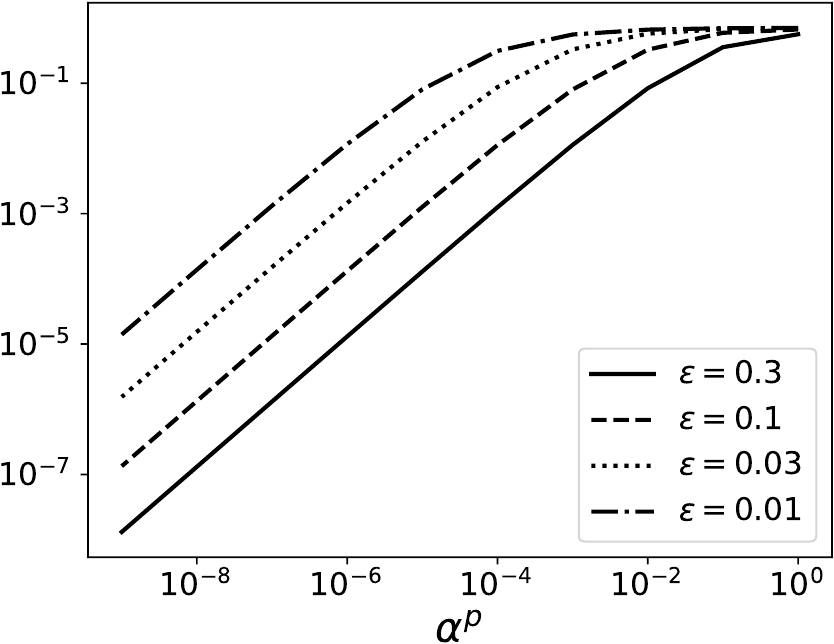}
                &\includegraphics[width=\linewidth]{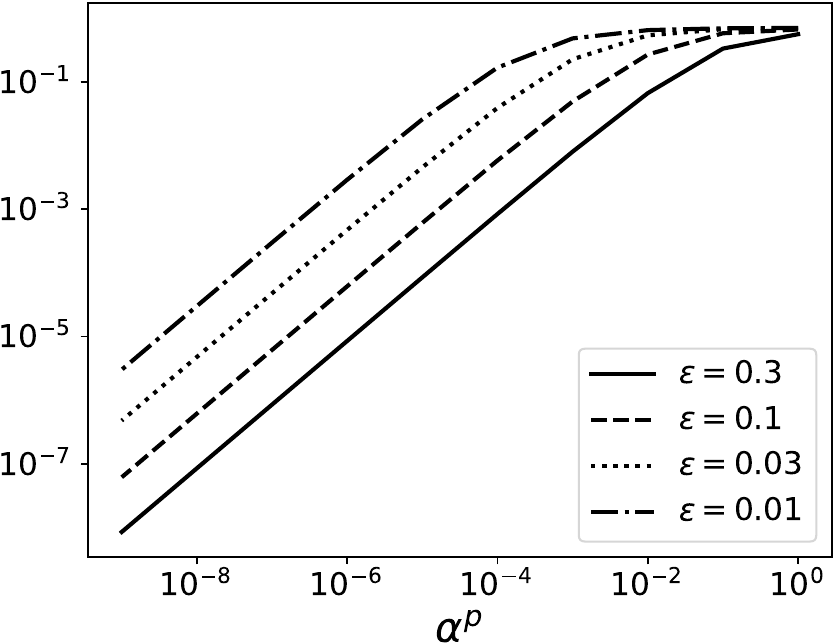}
            \end{tabular}
    \end{center}
    \vspace{-5mm}
    \caption{\label{fig:neg} We solve the gradient flow problem \eqref{eq:grad_flow} for the linear system in \eqref{eq:shift_example} for different choices of $\alpha>0$ and $p\geq 2$.}
\end{figure}

\subsection{Numerical examples}
\label{s:numerics}
In this section, we explore two numerical examples, which illustrate different aspects of \Cref{thm:grad_flow}. In the experiments we compute approximations to $\psi_{\alpha}(\infty)$ for  different data $(A,y)$ and hyperparameters $\alpha$ and $p$.  To compute the approximations to $\psi_{\alpha}(\infty)$, we choose a large value of $t> 0$ and compute an approximation $\hat\psi$ to a solution of the dynamical system \eqref{eq:grad_flow1} at time $t$. This approach is motivated by the findings in \Cref{prop:main1}, but to solve the dynamical system we choose to use an implicit Runge-Kutta solver for improved numerical accuracy. Codes for reproducing the figures in this paper are available at \url{https://github.com/johanwind/which_l1_minimizer}.

\begin{example}\label{ex:positive}
  In this example, we consider three different matrices $A_i$ and corresponding vectors $y_i$, $i=1,2,3$ (see \Cref{s:ex_details} for details). The last two matrices $A_i$ and vectors $y_i$ are designed so that the solution set $\mathcal{U}(A,y)$ consists of more than one element, and such that $\mathcal{W}_p(A_i,y_i)$ predicts different limits depending on the value of $p$. Since $\mathcal{W}_p(A,y) \in  \mathcal{U}(A,y)$ and $\mathrm{card}(\mathcal{U}(A,y)) > 1$, this allows us to check that the computed solution converges to the minimizer prescribed by the theorem, and not one of the other minimizers. In Figure \ref{fig:positive_results}, we compute approximations to $\psi_{\alpha}(\infty)$ for different values of $\alpha$. As we can see from the figure, the error decreases at the rate $\alpha^p$, for the different values of $p > 2$. Moreover, the final accuracy is $10^{-8}$, which indicates that the system converges to $\mathcal{W}_p(A,y)$, rather than any of the other minimizers. This agrees well with \Cref{thm:grad_flow}. 
\end{example}

\begin{example}
    In this example, we show how the constant $C_1$ in \Cref{thm:grad_flow} grows for different matrices $A$. To show this effect, we consider the linear system with  
    \begin{equation}\label{eq:shift_example}
        A = A(\epsilon) = \begin{bmatrix} 1 & 1-\epsilon \end{bmatrix} \quad\text{ and }\quad y=1.
    \end{equation}
For this system, the solution set $\mathcal{U}(A(\epsilon), y) = (1, 0)$ is single-valued for $\epsilon > 0$. In \Cref{fig:neg} we compute approximations $\psi_{\alpha}(\infty)$ for different values of $\alpha$ and $p$, in the same way as we did in \Cref{ex:positive}. In the figure we can see that all the approximations decay with the error rate $\alpha^p$, as $\alpha \to \infty$, but the constant in front of the different matrices are different. In particular, we see that $C_1$ grows for smaller values of $\epsilon$. 
\end{example}

\subsection{Related works}
\label{ss:rel_work}

This paper establishes new convergence results for diagonal linear networks for regression problems, and it connects these results to the phenomenon of generalized hardness of approximation for continuous optimization problems. Below we review connected works and provide some historical background.

\begin{itemize}[leftmargin=10pt]
  \setlength\itemsep{0em}

\item[] {\bf  \emph{Implicit regularization in AI:}} Some of the first systematic numerical investigations into the role over-parametrization plays in obtaining neural networks that generalize well can be found in \cite{neyshabur2017geometry, neyshabur2014search, zhang2017understanding}. Yet, older works focusing on how the optimization algorithm affects the generalization of neural networks also exist \cite{lecun2002efficient}. Early work in this direction for simple models mainly focused on the matrix factorization and completion problem, see e.g., work by Neyshabur, Tomioka, \& Srebro \cite{neyshabur2014search}, or work by Arora, Choen, Hu, \& Luo \cite{arora2019implicit}. As highlighted in the introduction, a large number of works \cite{arora2018optimization, arora2019implicit, chou2020gradient, geyer2020low, gissin2019implicit, gunasekar2018implicit, gunasekar2017implicit, neyshabur2017geometry, neyshabur2014search, razin2020implicit, razin2021implicit, razin2022implicit, soudry2018implicit, stoger2021small, bah2022learning, chou2021more, gunasekar2017implicit, li2021implicit, vaskevicius2019implicit, woodworth2020kernel, gissin2019implicit} have studied over-parametrization for simple models such as over-parameterized linear regression or matrix completion. 

\item[] {\bf  \emph{Generalized hardness of approximation and the Solvability Complexity Index (SCI) hierarchy:}}
The phenomenon of GHA in optimization first appeared in the work by Bastounis et al. in \cite{opt_big} (see also Adcock  \cite{CSBook} Chapter 8) for several continuous optimization problems routinely used in scientific computing and data science. These results were subsequently extended to the world of AI by Colbrook et al. in \cite{comp_stable_NN22} (see also \cite{paradox22}).  Recently, these results have also been applied in problems related to the design of optimal neural networks for linear inverse problems in \cite{gazdag2022generalised}. 
GHA belongs to the mathematical theory behind the SCI hierarchy which was introduced in \cite{Hansen_JAMS} and continued in the work by Ben-Artzi et al. \cite{Ben_Artzi2022, ben2022computing}, by Colbrook et al.  
\cite{colbrook2021computing, colbrook2019foundations, Colbrook_2019} and by Nevanlinna \cite{Hansen2016ComplexityII, CRAS, SCI} and co-authors. See also the work by S. Olver and M. Webb \cite{webb2021spectra}. The SCI hierarchy is directly related to 
to Smale's \cite{smale1981fundamental, smale1997complexity} program on foundations for computations and the seminal work by  McMullen \cite{mcmullen1985families} and 
Doyle \& McMullen \cite{doyle1989solving} on polynomial root-finding. 

\item[] {\bf  \emph{Diagonal linear neural networks:}}
The first indication that the implicit regularization of linear networks with tiny initialization could be connected to $l^1$-minimization can be found in \cite{gunasekar2017implicit}[Cor.\ 2]. This was followed by a more systematic treatment of the subject conducted by Woodworth et al.\ in \cite{woodworth2020kernel}. Herein, the authors found that if $p\geq 2$ is an integer, and if the converged solution $\psi_{\alpha}(\infty)$, satisfies $A\psi_{\alpha}(\infty) = y$, then  $\|\psi_{\alpha}(\infty)\|_1 \to R$ as $\alpha \to 0$. Azulay et al. \cite{azulay2021implicit} extends this work for $p=2$ for different initializations and the model $\psi_{\theta} = \theta_{+}^{(1)} \odot \theta_{+}^{(2)} - \theta_{-}^{(1)} \odot \theta_{-}^{(2)}$, where $\odot$ represents entry-wise multiplication. Similarly, in \cite{chou2021more} Chou, Maly \& Rauhut consider the parameterization $\psi_{\theta} =\theta_{+}^{(1)} \odot \cdots \odot \theta_{+}^{(p)} - \theta_{-}^{(1)} \odot \cdots \odot \theta_{-}^{(p)}$. Their main result says that if one chooses the initialization $\alpha{\bf 1}_{2pN}$, and $\alpha \leq h(R, \epsilon)$ where $h$ is a given function and $\epsilon > 0$, then $\|\psi_{\alpha}(\infty)\|_1 - R \leq \epsilon$. That is, by choosing $\alpha$ sufficiently small, the gradient flow can get arbitrarily close to the minimum of the basis pursuit problem.  

In \cite{vaskevicius2019implicit} Vaskevicius, Kanade \& Rebeschini consider the model $\psi_{\theta} = \theta_{+}^{2} - \theta^{2}_{-}$ with $\theta_0 = \alpha {\bf 1}$, and they analyze the gradient descent algorithm directly.  Their main result, is not on $\ell^1$-minimization, but ensures that one can approximate a sparse vector $x$, under certain conditions. In \cite{li2021implicit}, Li et al.\ extend the work of \cite{vaskevicius2019implicit} to the case with $p > 2$ and slightly different conditions on $A$. Moroshko et al.\ \cite{moroshko2020implicit} studies case $p\geq 2$ for classification problems with an exponential loss function, whereas Pesme, Pillaud-Vivien, \& Flammario \cite{pesme2021implicit} analyze the stochastic gradient descent algorithm for  $p=2$ and the squared loss function. 

\item[] {\bf  \emph{Robust optimization and optimization for sparse recovery:}} GHA is directly linked to robust optimization and the work by Ben-Tal, El~Ghaoui, and Nemirovski \cite{NemirovskiLRob, Nemirovski_NPhard_Stable, Nemirovski_robust, Nemirovski_robust2}. See also recent results on counter\-examples describing non-convergence of specific optimization algorithms \cite{bolte2022iterates, bolte2022curiosities}. In particular, GHA is inevitable for any robust optimization theory for computing minimizers of convex optimization problems. Both implicit regularization in AI and GHA are crucially linked to optimization for sparse recovery, where there is a myriad of work and we can only highlight certain works here, for example (in random order) : Juditsky, Kilin{\c{c}}{-}Karzan, Nemirovski \& Polyak. \cite{Juditsky_2012, Juditsky_2011}, Chambolle \& Pock \cite{chambolle_pock_2016, Chambolle_2011}, Figueiredo, Nowak \& Wright \cite{Mario, Mario2}, Nesterov \& Nemirovski \cite{Nesterov_Nemirovski_Acta}, Colbrook \& Adcock \cite{colbrook2022warpd, neyra2023nestanets, adcock2023restarts}.

\item[] {\bf  \emph{Over-parametrization in optimization:}}
Over-parametrization is also entering optimization as a technique for converting non-smooth problems, such as those based on $\ell^1$-minimization, to smooth optimization problems by introducing more variables. Often this comes at the expense of losing convexity. In \cite{hoff2017lasso} Hoff studied the LASSO problem with a parametrization $x=v\odot u$. Therein the author show that if $(\hat u,\hat v)$ is an optimal solution to the smooth problem, then $\hat x = \hat u \odot \hat v$ is an optimal value for the LASSO problem. This parametrization is studied further by Poon \& Peyré in \cite{ Poon_NeurIPS21, poon2023smooth} which use this technique to develop new optimization algorithms, whereas Zhao, Yang \& He study this technique in the setting of implicit regularization \cite{Zhao_IR22}.

\item[] {\bf  \emph{Hardness of approximation:}}
In 2001 Arora, Feige, Goldwasser, Lund, Lovász, Motwani, Safra, Sudan, \& Szegedy received the Gödel prize for a series of papers \cite{arora1998proof, arora1998probabilistic, feige1996interactive} on probabilistically checkable proofs and its connection to hardness of approximation. This discovery opened up new avenues for understating the difficulty of computing approximate solutions to NP-hard problems. For example, Arora \& Mitchell received the Gödel prize in 2010 for showing that a polynomial-time approximation scheme the Euclidean traveling salesman problem \cite{Arora_PTAS} and Guillotine subdivisions \cite{Mitchell99}, and Khot received in Nevanlinna Prize in 2014 for establishing the Unique Games problem, and showing how a solution to this problem would imply the inapproximability of several NP-hard optimization problems \cite{Khot}.

\end{itemize}

\section{Proofs of \Cref{thm:grad_flow} part \ref{it:dep_bound} \& \Cref{prop:main1} part \ref{it:m1} }
\subsection{Preliminaries}
\label{ss:prelim}
\subsubsection{Notation}
    Let $x$ be a $N$-dimensional real valued vector. 
We denote the support of $x$ by $\supp(x) = \{i \in \{1,\dots,N\} : x_i \ne 0\}$, and use $\overline{\supp(x)} = \{1,\ldots, N\}\setminus \supp(x)$ for its compliment.
Furthermore, we let $\mathcal{R}(x) = \left\{\diag(x)z : z \in [0,\infty)^N\right\} $ denote the signed orthant given by $x$. Note $\mathcal{R}(x)$ intersects the coordinate axes, since $z$ is allowed to be zero. Moreover, $\mathcal{R}(x)$ collapses to a single point if $x=0$. For a real number $t$, we let $\mathrm{sign}(t)$ denote the sign of $t$, with the convention that $\mathrm{sign}(0)=0$. For a vector $x$ we apply $\mathrm{sign}(x)$ componentwise to entry in $x$.
For a matrix $A \in \R^{m \times N}$ with $\mathrm{rank}(A) = r$ we denote the ordered singular values of $A$ by $\sigma_{1}(A) \geq \cdots \geq \sigma_{r}(A) = \sigma_{\min}(A) > 0$. We let $\|A\|_{\mathrm{op}}$ denote the operator norm of $A$, and note that $\|A\|_{\mathrm{op}} = \sigma_{1}(A)$. Furthermore, if $\rank(A) = m$, then 
\begin{align}\label{sigma_min_prop}
  \sigma_{\min}(A) = \inf_{\substack{z \in \R^m \\ \norm{z}_2 = 1}}\norm{A\T z}_2 > 0.
\end{align}
We denote the nullspace of $A$ by $\Null(A)$, and let $P_{\Null(A)} \in \R^{N\times N}$ denote the projection onto the nullspace of $A$. We use $f'$ to denote the derivative of a differentiable function $f\colon V \to \R$, $V\subseteq \R$. Finally, $\lfloor \cdot \rfloor$ denotes the floor function.

\subsubsection{Setup}\label{ss:setup}

Recall the setup from \Cref{s:intro}. Here $p \in [2,\infty)$ denotes the degree of homogeneity in \eqref{eq:param}, and $\alpha > 0$ denotes the initialization scale in \eqref{eq:grad_flow}. We consider a matrix $A \in \R^{m\times N}$ with $\rank(A) = m$, and $y \in \R^m$.
Following Woodworth et al.\ \cite{woodworth2020kernel}, we define the convex function $Q_p \colon \R^N \to \R$, whose minimizers are closely connected to the parameterization in \eqref{eq:param}, and the subsequent gradient flow problem \eqref{eq:grad_flow}. The function $Q_p$ is defined through scalar functions $q_p$, which in turn are defined via functions $h_p$. These are defined as follows:
\begin{equation}
    \begin{split}\label{eq:h_2_def}
      h_2(t) &= 2\sinh(t), \quad \text{for } t \in \R\\
      h_p(t) &= (1-t)^{-\frac{p}{p-2}}-(1+t)^{-\frac{p}{p-2}}, \quad \text{for } t \in (-1,1) \text{ and } p > 2.
    \end{split}
\end{equation}
The different functions $h_p$, as well as the functions $q_p$ and $g_p$ defined below, can be seen in \Cref{fig:setup}.
Note that for both $p=2$ and $p>2$, $h_p$ is smooth, strictly increasing and odd, i.e., $h_p(-t) = -h_p(t)$. Moreover, for $t \in \R$ we have $h'_2(t) \geq 2$, and for $p > 2$ we see that $h'(t) \geq 2p/(p-2)$ for $t \in (-1,1)$.

Next, we define
\begin{equation}\label{q_p_def}
  q_p(u) = \int_0^u h_p^{-1}(v)\ dv,
\end{equation}
where $h_{p}^{-1}$ denotes the inverse of $h_{p}$. Note that $h_{p}^{-1}$ is defined on all of $\R$, since the range of $h_{p}$ is $\R$ and $h_{p}$ is a strictly increasing. For $p=2$, this integral has an explicit solution \cite[Eq. (21)]{woodworth2020kernel},
\[
  q_2(u) = \int_0^u h_2^{-1}(v)\ dv = 2-\sqrt{u^2+4}+u\sinh^{-1}\left(\frac{u}{2}\right).
\]
Moreover, we observe that for $p \geq 2$, $q_p(0) = 0$ and  $q_{p}(-u) = q_{p}(u)$. Furthermore, $q_{p}$ is strictly convex since it is the integral of the strictly increasing function $h_p^{-1}$. See \Cref{prop:q_p_properties} for further details.

We shall see that the solution $\psi_{\alpha}(t) \in \R^N$ in \eqref{psi_def} to the gradient flow problem \eqref{eq:grad_flow}, will be closely related to the optimization problem 
\begin{align}\label{Vp_def}
  \mathcal{V}_p(A,y) \coloneqq \argmin_{z \in \R^N} Q_p(z) \quad\text{subject to}\quad Az = y.
\end{align}
\Cref{Vp_single_valued} tells us that $\mathcal{V}_p(A,y)$ is single valued. Here $Q_p \colon \R^N \to \R$ is the function 
\begin{equation}\label{Q_p_def}
  Q_p(z) = \alpha^p\sum_{i=1}^N q_p\left(\frac{z_i}{\alpha^p}\right),
\end{equation}
where $q_p$ is given by \eqref{q_p_def}, $\alpha > 0$ and $p \geq 2$. Note that $Q_p$ is a strictly convex function since $q_p$ is strictly convex and the sum in \eqref{Q_p_def} is non-negative.

\begin{figure}
    \begin{center}
		\setlength{\tabcolsep}{3pt} 
            \begin{tabular}{@{}>{\centering}m{0.30\textwidth}>{\centering}m{0.29\textwidth}>{\centering\arraybackslash}m{0.30\textwidth}@{}}   
                $h_p$ & $q_p$ & $g_p$ \\
                \includegraphics[width=\linewidth]{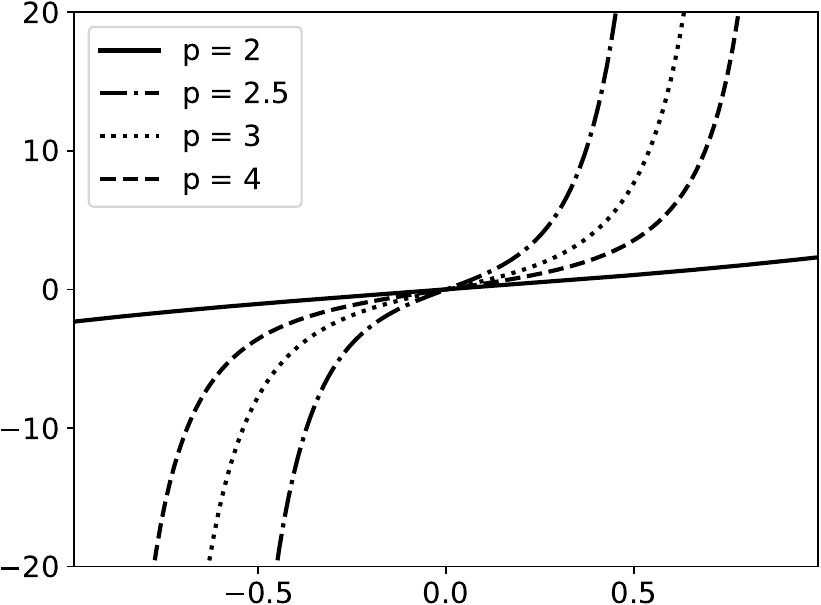} 
                &\includegraphics[width=\linewidth]{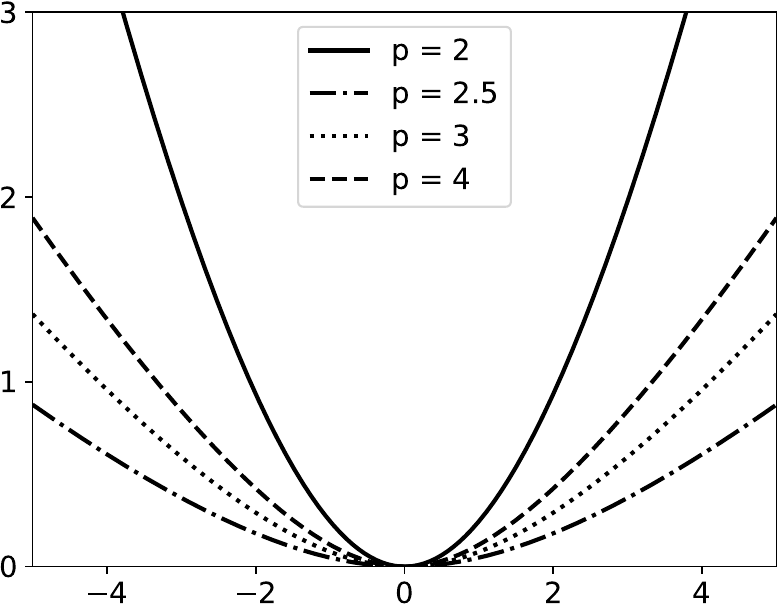}
                &\includegraphics[width=\linewidth]{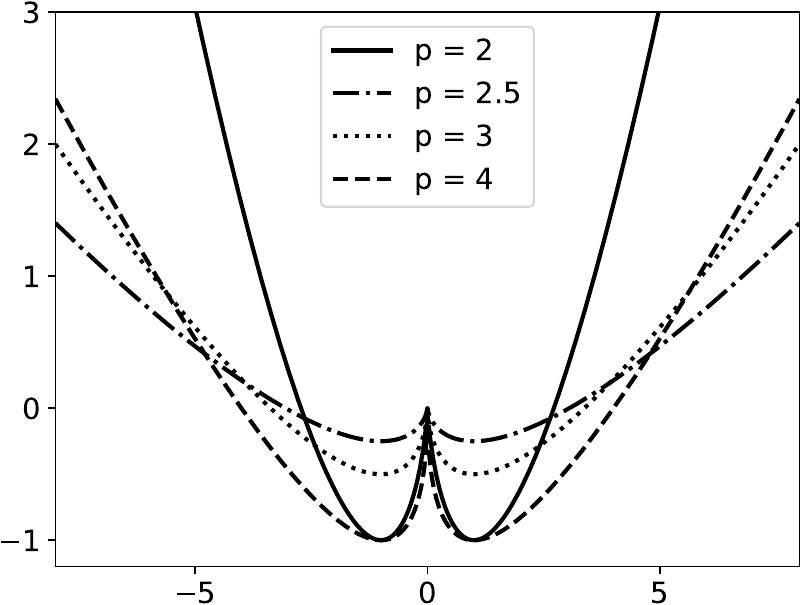}
            \end{tabular}
    \end{center}
    \vspace{-5mm}
    \caption{\label{fig:setup}The functions considered in the \Cref{ss:setup}.}
\end{figure}

Previous works \cite{chou2021more, woodworth2020kernel} have noted that $Q_p$ behaves as the $\ell_1$ norm when $\alpha$ goes to zero. Specifically, these works bound the distance between the minimum value $R(A,y)$ of \eqref{eq:BP_min} and the $\ell_1$-norm of the converged gradient flow  $\norm{\psi_\alpha(\infty)}_1$. In this work, we consider convergence of $\psi_\alpha(\infty)$ (as opposed to $\norm{\psi_\alpha(\infty)}_1$). This requires us to answer {\it which} $\ell_1$ minimizer $\psi_\alpha(\infty)$ convergences to, when there are several $\ell_1$ minimizers. To accomplish this we introduce a function $G_p$, defined below, which captures the behaviour of $Q_p$ when $\alpha \to 0$. The function $G_p$ is also key to getting sharp convergence rates in terms of $\alpha$.
For $u \in \R$, define the functions
\begin{equation}\label{eq:g_def}
    \begin{split}
    g_2(u) &= \begin{cases}
    |u|\ln\left(\tfrac{|u|}{e}\right) & \text{ for }|u| > 0 \\
    0 & \text{ for } u = 0
  \end{cases}, \\
    g_p(u) &= |u| - \tfrac{p}{2} |u|^{2/p}, \quad \text{for } p > 2.
    \end{split}
\end{equation}
Moreover, for $p \geq 2$ and $\alpha > 0$, let $G_p \colon \R^N \to \R$ be the function
\begin{equation}\label{G_p_def}
  G_p(z) = \alpha^p\sum_{i=1}^N g_p\left(\frac{z_i}{\alpha^p}\right).
\end{equation}
For \Cref{thm:grad_flow}, we need to bound $\norm{\psi_{\alpha}(\infty)-\mathcal{W}_p(A,y)}_2$. We will bound this distance through three intermediate points, defined as solutions of the following three optimization problems: 
\begin{alignat}{2}
  \label{eq:q*}
  q^* &= \argmin_{z \in \R^N} Q_p(z) \quad &&\text{ subject to }\quad Az = y,\\
  \label{eq:m*}
  m^* &= \argmin_{z \in \R^N} Q_p(z)       &&\text{ subject to }\quad z \in \mathcal{U}(A,y),\\
  \label{eq:g*}
  g^* &= \argmin_{z \in \R^N} G_p(z)       &&\text{ subject to }\quad z \in \mathcal{U}(A,y).
\end{alignat}
Feasibility of these problems are guaranteed by the assumption $\rank (A) = m$, and uniqueness is guaranteed by \Cref{cor:un_qm,cor:un_g} below. We will show that $\psi_{\alpha}(\infty) = q^* = \mathcal{V}_p(A,y)$ and $\mathcal{W}_p(A,y) = g^*$, and then bound the distances $\norm{q^*-m^*}_2$ and $\norm{m^*-g^*}_2$. This will in turn be used to prove \Cref{thm:grad_flow}.

\subsubsection{A few results related to the setup}

We start by establishing a few key facts about $q_p$ and $g_p$ defined in \Cref{q_p_def,eq:g_def}.

\begin{proposition}[Properties of $q_p$]\label{prop:q_p_properties}
  For any $p \ge 2$, $q_p \colon \R \to \R$ is an even, strictly convex, continuously differentiable function.
\end{proposition}
\begin{proof}
  Since $h_p$ is an odd function, so is $h_{p}^{-1}$. It follows that $q_p(-u) = \int_{0}^{-u} h_p^{-1}(v)\ dv = q_{p}(u)$, and hence $q_p$ is even. Next, notice that since $h_{p}$ is a strictly increasing and smooth function, so is $h^{-1}_{p}$. It follows that $q_p$ is strictly convex and smooth.
\end{proof}
An immediate consequence of the above proposition is the following corollary. 
\begin{corollary}\label{prop:Q_p_properties}
  For any $p \ge 2$, $Q_p \colon \R^N \to \R$ is an even, strictly convex, continuously differentiable function. It attains its minimum at the origin.
\end{corollary}

When bounding the distances between the minimizers $q^*,m^*,g^*$ defined in (\ref{eq:q*}-\ref{eq:g*}), we will need the following bound.
\begin{proposition}\label[proposition]{qg'_bound}
  Let $p \ge 2$, then $g_p'(u) \le q_p'(u) \le g_p'(u+1)$ for all $u \in (0,\infty)$.
\end{proposition}
\begin{proof}
    First consider the case when $p > 2$. Let $a \colon (-1,1) \to (-1,\infty)$ and $b \colon (-1,1) \to (0,\infty)$ be given by $a(t) = -1+ (1-t)^{-\frac{p}{p-2}}$ and $b(t) = (1-t)^{-\frac{p}{p-2}}$.
    Furthermore, observe that $h_p(t) = a(t)-a(-t) = b(t)-b(-t)$. The inverse of both functions are related to $g_p'$ in the following way
\begin{alignat*}{2}
  a^{-1}(u) &= 1-(u+1)^{\frac{2}{p}-1} = g_p'(u+1) \quad\quad && u \in (-1,\infty),\\
  b^{-1}(u) &= 1-u^{\frac{2}{p}-1} = g_p'(u) && u \in (0,\infty).
\end{alignat*}
    Next, notice that for $t \in (0,1)$ we have $a(t) \le h_p(t) \le b(t)$. Since $a,b,h_p$ are all increasing functions, this implies $a^{-1}(u) \ge h_p^{-1}(u) \ge b^{-1}(u)$ for $u \in (0,\infty)$. So for all $u \in (0,\infty)$ we have 
  \begin{align*}
    g_p'(u) = b^{-1}(u) \le h_p^{-1}(u) = q_p'(u) = h_p^{-1}(u) \le a^{-1}(u) = g_p'(u+1),
  \end{align*}
    which proves the claim for $p>2$.

    Next, consider the case $p = 2$. This time, let $a : \R \to (-1,\infty)$ and $b : \R \to (0,\infty)$ be given by
      $a(t) = \exp(t)-1$, and  $b(t) = \exp(t)$. As before, we observe that $h_p(t) = a(t)-a(-t) = b(t)-b(-t)$. Furthermore, the inverse of these functions are related to $g'_{p}$ as follows
  \begin{align*}
    a^{-1}(u) &= \ln(u+1) = g_2'(u+1), & u \in (-1,\infty)\\
    b^{-1}(u) &= \ln(u) = g_2'(u), & u \in (0,\infty)
  \end{align*}
  Again, for all $t \in (0,\infty)$ we have $a(t) \le h_p(t) \le b(t)$. Repeating the argument from the case $p > 2$, this implies for all $u \in (0,\infty)$ we have $g_2'(u) = b^{-1}(u) \le h_2^{-1}(u) = q_2'(u) \le a^{-1}(u) = g_2'(u+1)$.
\end{proof}

Next we ensure that the set of solutions in \Crefrange{eq:q*}{eq:g*} are single valued. Our first result, \Cref{l:unique_min}, establishes general conditions ensuring that the minimizer exists and is unique. This will be used to establish that $q^*$ and $m^*$ are well defined in the corollary below. The fact that $g^*$ is single-valued is treated in \Cref{cor:un_g}. 
Proving existence requires some attention to technical details. Note that strict convexity is not enough. Indeed, the function $x \mapsto \exp(x)$ defined on $\R$ is a strictly convex function on a closed and convex set, but it does not attain its infimum. Additionally, neither $Q_p$ nor $G_p$ is strongly convex, which would have allowed us to apply, for example, \cite[Thm.\ 2.2.10]{nesterov2018lectures}.

\begin{lemma}\label{l:QG_coercive}
  Let $p \ge 2$ and $\alpha > 0$, then $Q_p$ and $G_p$ are weakly coercive \cite[Def.\ 4.5]{andreasson2005introduction}.
\end{lemma}
\begin{proof}
  $G_p$ is weakly coercive since $g_p$ is weakly coercive. For all $z \in \R^n$ we have $Q_p(z) \ge G_p(z)$ since \Cref{qg'_bound} and $q_p(0) = g_p(0)$ give $q_p(u) \ge g_p(u)$ for all $u \in \R$. So $Q_p$ is also weakly coercive.
\end{proof}

\begin{proposition}\label{l:unique_min}
  Let $D \subset \R^N$ and $f \colon D \to \R$ be a continuous, weaky coercive, strictly convex function. Suppose $S \subset D$ is convex, closed and non-empty. Then the optimization problem
  \begin{align*}\label{eq:unique_min}
    \argmin_{z\in D} f(z) \quad\text{subject to}\quad z \in S
  \end{align*}
  has a unique solution. 
\end{proposition}
\begin{proof}
  The minimizer exists by \cite[Thm.\ 4.7]{andreasson2005introduction} and is unique by \cite[Prop.\ 4.11]{andreasson2005introduction}.
\end{proof}

\begin{corollary} \label{cor:un_qm}
  Let $A \in \mathbb{R}^{m\times N}$ have rank $m$. Then for any $y \in \R^m$, the minimizers $q^*$ and $m^*$ in \eqref{eq:q*} and \eqref{eq:m*} exist and are unique.
\end{corollary}
\begin{proof}
  This follows from \Cref{l:unique_min}, and we will check the conditions of the proposition. From \Cref{prop:Q_p_properties} we know $Q_p \colon \R^N \to \R$ is continuous, strictly convex and \Cref{l:QG_coercive} says it is weakly coercive. Moreover, the affine set $\{z \in \R^N : Az = y\}$ is closed and convex. The set $\mathcal{U}(A,y)$ is closed and convex by \Cref{M_s_lemma}.
\end{proof}

Since $G_p$ is not convex everywhere, we need a slightly modified proof for $g^*$.
\begin{corollary}\label{cor:un_g}
  Let $A \in \mathbb{R}^{m\times N}$ have rank $m$. Then for any $y \in \R^m$ the minimizer $g^*$ in  \eqref{eq:g*} exists and is unique.
\end{corollary}
\begin{proof}
  From \Cref{M_s_lemma} we know the set $\mathcal{U}(A,y)$ is closed, convex and satisfies $\mathcal{U}(A,y) \subset \mathcal{R}(s)$ for some $s \in \{-1,1\}^N$. Now, observe that $G_p$ restricted to $\mathcal{R}(s)$ is a continuous, strictly convex function which is weakly coercive by \Cref{l:QG_coercive}. It satsifies the conditions of \Cref{l:unique_min}, so $m^*$ exists and is unique.
\end{proof}

\begin{corollary}\label{Vp_single_valued}
  Let $A \in \mathbb{R}^{m\times N}$ have rank $m$, $p \ge 2$ and $y \in \R^m$, then $\mathcal{V}_p(A,y)$ exists and is unique.
\end{corollary}
\begin{proof}
  Completely analogous to \Cref{cor:un_qm}.
\end{proof}

\subsection{Proof of \Cref{prop:main1} part \ref{it:m1}}
\subsubsection{Relating \texorpdfstring{$\psi_{\alpha}$ to $\mathcal{V}_p$}{}}

Ultimately, one of our goals is to bound the distance $\|\psi_{\alpha}(t) - \mathcal{W}_p (A,y)\|_2$ as $t \to \infty$. A first step in this direction is to derive an explicit expression for $\psi_{\alpha}(t)$ at time $t \geq 0$. As the following lemma reveals, the description of $\psi_{\alpha}$ is closely connected to $h_p$ in \Cref{eq:h_2_def}.

\begin{lemma}[The dynamics of $\psi_{\alpha}$]\label[lemma]{psi_integral}
    Let $A \in \R^{m\times N}$ have rank $m$, and let $y \in \R^m$, $p \in [2,\infty)$ and $\alpha > 0$. 
    The state $\psi_{\alpha}(T)$ in \eqref{psi_def}, evolving by the gradient flow \eqref{eq:grad_flow}, at time $T \ge 0$, then satisfies the following componentwise equations
    \begin{align*}
      [\psi_{\alpha}(T)]_i &= \alpha^p h_p\left(C\int_0^T [r(t)]_i\ dt\right), \quad\quad\text{for } i=1,\ldots,N,
\end{align*}
where $r(t) = A\T(y-A\psi_{\alpha}(t))$, and $C = 4$ if $p = 2$ and $C = p(p-2)\alpha^{p-2}$ if $p > 2$.
\end{lemma}
\begin{proof}
    Throughout the proof let $i \in \{1,\dots,N\}$. Next, recall that the gradient flow \eqref{eq:grad_flow} for $t \in [0,\infty)$ is given by 
   \begin{align*}
    \frac{d}{dt}\theta_+(t) &= p\diag(\theta_+(t))^{p-1}r(t)\\
    \frac{d}{dt}\theta_-(t) &=-p\diag(\theta_-(t))^{p-1}r(t)\\
    \theta_+(0) &= \theta_-(0) = \alpha{\bf 1}\\
    [\psi_{\alpha}(t)]_i &= [\theta_+(t)]_i^p-[\theta_-(t)]_i^p,
  \end{align*}
    where the minus sign in \cref{eq:grad_flow} has been incorporated into the definition of $r(t)$.
  These equations can be solved by separation of variables. 
    If $p = 2$ we have that
    \begin{equation*}\label{eq:p2_theta}
        [\theta_+(T)]_i = \alpha\exp\left(2\int_0^T [r(t)]_i\ dt\right)\quad\text{ and }\quad
        [\theta_-(T)]_i = \alpha\exp\left(-2\int_0^T [r(t)]_i\ dt\right),
    \end{equation*}
    for $T \in [0,\infty)$, which gives the state
  \begin{align*}
    [\psi_{\alpha}(T)]_i &= \alpha^2\left(\exp\left(4\int_0^T [r(t)]_i\ dt\right)-\exp\left(-4\int_0^T [r(t)]_i\ dt\right)\right) = \alpha^2 h_2\left(4\int_0^T [r(t)]_i\ dt\right).
  \end{align*}
Next, if $p > 2$ we have that
    \begin{equation}\label{eq:pg2_theta}
      \begin{split}
        [\theta_+(T)]_i &= \left(\alpha^{2-p}-p(p-2)\int_0^T [r(t)]_i\ dt\right)^{-\frac{1}{p-2}}\\
        [\theta_-(T)]_i &= \left(\alpha^{2-p}+p(p-2)\int_0^T [r(t)]_i\ dt\right)^{-\frac{1}{p-2}}.
      \end{split}
    \end{equation}
    for $T \in [0,\infty)$.
    This gives the state
  \begin{align*}
    & [\psi_{\alpha}(T)]_i = \left(\alpha^{2-p}-p(p-2)\int_0^T [r(t)]_i\ dt\right)^{-\frac{p}{p-2}} - \left(\alpha^{2-p}+p(p-2)\int_0^T [r(t)]_i\ dt\right)^{-\frac{p}{p-2}}\\
                &= \alpha^p\left(\left(1-p(p-2)\alpha^{p-2}\int_0^T [r(t)]_i\ dt\right)^{-\frac{p}{p-2}} - \left(1+p(p-2)\alpha^{p-2}\int_0^T [r(t)]_i\ dt\right)^{-\frac{p}{p-2}}\right)\\
                &= \alpha^p h_p\left(p(p-2)\alpha^{p-2}\int_0^T [r(t)]_i\ dt\right).
  \end{align*}
\end{proof}

Our next result shows the close connection between $\psi_{\alpha} (t)$ and the optimization problem \cref{eq:q*}. It will be used in \cref{Apsi_to_y} to show that $\psi_{\alpha}(t) \to q^*$ as $t\to \infty$. 

\begin{lemma}[$A\psi_{\alpha}$ determines $\psi_{\alpha}$]\label[lemma]{Apsi_determines_psi}
    Let $A \in \R^{m\times N}$ have rank $m$ and let $y \in \R^m$, $p \in [2,\infty)$ and $\alpha > 0$. 
    The state $\psi_{\alpha}(t)$ in \eqref{psi_def} to the gradient flow problem \eqref{eq:grad_flow} at time $t \ge 0$ with data $y$ and initialization $\alpha$ satisfies
  \begin{align}\label{argmin_psi}
    \psi_{\alpha}(t) = \mathcal{V}_p(A,A\psi_{\alpha}(t)).
  \end{align}
\end{lemma}
\begin{proof}
    Fix $t \ge 0$. By the definition of $\mathcal{V}_p$ in \eqref{Vp_def} we have
    \begin{align*}
      \mathcal{V}_p(A,A\psi_{\alpha}(t)) = \argmin_{z \in \R^N} Q_p(z) \text{ subject to } Az = A\psi_{\alpha}(t).
    \end{align*}
  A solution $w$ to this optimization problem is optimal if and only if it satisfies the KKT conditions (\Cref{KKT}). Hence, it is sufficient that there exists $\lambda \in \R^m$ such that $\nabla Q_p(w) = A\T \lambda$ and $Aw = A\psi_{\alpha}(t)$. The latter condition is trivially satisfied by $\psi_{\alpha}(t)$. We will show $\nabla Q_p(\psi_{\alpha}(t)) = A\T\lambda$, for the $\lambda$ described below.

  Let $i \in \{1,\ldots, N\}$, and note
  \begin{equation}\label{eq:nabla_Q1}
    [\nabla Q_p(\psi_{\alpha}(t))]_i = q'_p\left(\frac{[\psi_{\alpha}(t)]_i}{\alpha^p}\right) = h_p^{-1}\left(\frac{[\psi_{\alpha}(t)]_i}{\alpha^p}\right).
  \end{equation}
  From \Cref{psi_integral} we know
  \begin{equation}\label{eq:nabla_Q2}
    [\psi_{\alpha}(t)]_i = \alpha^p h_p\left(C\int_0^t [r(u)]_i\ du\right)\quad \text{where}\quad  r(t) = A\T(y-A\psi_{\alpha}(t)),
  \end{equation}
  and $C$ is a constant depending on $p$ and $\alpha$. Combining \eqref{eq:nabla_Q1}
  and \eqref{eq:nabla_Q2}, yields
  \begin{align*}
    \nabla Q_p(\psi_{\alpha}(t))]_i &=   C\int_0^t [r(u)]_i\ du =  C\int_0^t \left[A\T(y-A\psi_{\alpha}(u))\right]_i\ du \\
                                    &=  C\int_0^t \left[ \sum_{j=1}^{m}(A\T)_{i,j}[(y-A\psi_{\alpha}(u))]_{j}\right]_i\ du = [A\T \lambda]_i, 
  \end{align*}
where $\lambda$ is a vector with components
  \[
    \lambda_j = C\int_0^t [y-A\psi_{\alpha}(u)]_j\ du\quad\text{for}\quad j \in \{1,\ldots,m\}.
  \]
\end{proof}

The following lemma will be used to show that $\mathcal{V}_p$ is continuous in its second argument in \Cref{Vp_lip} and provide loose bounds on $q^*$,$m^*$ and $\psi_\alpha$ in \Cref{loose}.
\begin{lemma}\label[lemma]{Vp_bound}
  Let $A \in \R^{m\times N}$ with rank$(A) = m$, $p \ge 2$ and $v \in \R^N$. Then
  \begin{align}\label{Vp_claim}
    \norm{\mathcal{V}_p(A,Av)}_\infty \le \norm{v}_1.
  \end{align}
\end{lemma}
\begin{proof}
  Let $w \coloneqq \mathcal{V}_p(A,Av)$. To prove \eqref{Vp_claim}, assume for contradiction that $\norm{w}_\infty > \norm{v}_1$. Then there exists $j \in \{1,\dots,N\}$ such that $|w_j| > \norm{v}_1$. Now, since $w$ is the minimizer over a set including $v$, we have that $Q_p(w) \le Q_p(v)$. Next, observe that for $t_1,t_2 \in \R$ with $|t_1| > |t_2|>0$, we have that $q_p(t_1)=q_p(|t_1|) > q_p(|t_2|)> q_p(0)$, since $q_p$ is a strictly convex and even function. In particular, we have that $q_p \left(\frac{w_j}{\alpha}\right) > q_p\left(\frac{\norm{v}_1}{\alpha}\right)$.
    Using these inequalities, we get 
    \begin{equation}\label{eq:Qp_ineq1}
  \begin{split}
    Q_p(w) &= \alpha^p \sum_{i=0}^N q_p\left(\frac{w_i}{\alpha^p}\right) \ge \alpha^p \left((N-1)q_p(0) + q_p\left(\frac{w_j}{\alpha^p}\right)\right)\\
    &> \alpha^p \left((N-1)q_p(0) + q_p\left(\frac{\norm{v}_1}{\alpha^p}\right)\right).
  \end{split}
  \end{equation}
    Now, since $q_p$ is convex, we apply Jensen's inequality between points $\frac{\norm{v}_1}{\alpha^p}$ and $0$ with weights $\frac{|v_i|}{\norm{v}_1}$, $i \in \{1,\dots,N\}$. This yields
    \begin{equation*}\label{eq:Qp_jensen}
    q_p\left(\frac{|v_i|}{\alpha^p}\right) \le \left(1-\frac{|v_i|}{\norm{v}_1}\right)q_p(0) + \frac{|v_i|}{\norm{v}_1}q_p\left(\frac{\norm{v}_1}{\alpha^p}\right).
  \end{equation*}
  Summing over $i$ we get
  \begin{equation}\label{eq:Qp_ineq2}
    \sum_{i=1}^N q_p\left(\frac{|v_i|}{\alpha^p}\right) \le (N-1)q_p(0)+q_p\left(\frac{\norm{v}_1}{\alpha^p}\right).
  \end{equation}
    Combining \eqref{eq:Qp_ineq1} and \eqref{eq:Qp_ineq2}, now yields
  \begin{align*}
      Q_p(w) > \alpha^p \sum_{i=1}^N q_p\left(\frac{|v_i|}{\alpha^p}\right)
    = \alpha^p \sum_{i=1}^N q_p\left(\frac{v_i}{\alpha^p}\right) = Q_p(v).
  \end{align*}
    which is a contradiction, and hence $\norm{w}_\infty \le \norm{v}_1$.
\end{proof}

\subsubsection{Bounding \texorpdfstring{$\norm{\psi_{\alpha}(t)-\psi_{\alpha}(\infty)}_2$}{}}
We will start by bounding $\norm{A\psi_{\alpha}(t)-y}_2$. Then, we will prove stability of $\mathcal{V}_p$ against perturbations in the second argument. Finally, these will be combined with \Cref{Apsi_determines_psi} to give a bound on $\norm{\psi_{\alpha}(t)-\psi_{\alpha}(\infty)}_2$.

\begin{proposition}[$A\psi_{\alpha}(t) \to y$]\label{Apsi_to_y}
  Let $A \in \R^{m\times N}$ with rank$(A) = m, y \in \R^m$. Let $\psi_{\alpha}(t)$ be the state of the gradient flow problem \eqref{eq:grad_flow} for some $p \ge 2, \alpha > 0$ at time $t$. Then for all $t \in [0,\infty)$
  \begin{align*}
    \norm{A\psi_{\alpha}(t)-y}_2 \le \norm{y}_2 \exp(-Ct)
  \end{align*}
  where $C = 2p^2\sigma_{\min}^2(A)\alpha^{2p-2}$. Specifically, $\lim_{t\to \infty} A\psi_{\alpha}(t) = y$.
\end{proposition}
\begin{proof}
  From \Cref{psi_integral} we get an explicit expression for the flow $\psi_{\alpha} = (\psi_{1},\ldots, \psi_{N})$. Differentiating the $i$th element gives
    \[
        \frac{d\psi_i}{dt}(t) = \alpha^p C_1 h'_p\left(C_1\int_0^t [r(u)]_i\ du\right) [r(t)]_i,\quad\text{where}\quad C_1 = \begin{cases} 
        4 & \text{if } p = 2\\
            p(p-2)\alpha^{p-2} & \text{if } p > 2
        \end{cases}
    \]
    and $r(t) = A\T(y-A\psi_{\alpha}(t))$.
    
    Next, we consider the dynamics of $\norm{y - A\psi_{\alpha}(t)}_2^2$, which is given by
  \begin{align*}
      \frac{d}{dt}\norm{y - A\psi_{\alpha}(t)}_2^2 &= -2\left[\frac{d\psi_{\alpha}}{dt}(t)\right]\T [A\T(y-A\psi_{\alpha}(t))] = -2\left[\frac{d\psi_{\alpha}}{dt}(t)\right]\T r(t)\\
                                      &= -2 \alpha^p C_1 \sum_{i=1}^N h'_p\left(C_1\int_0^t [r(u)]_i\ du\right) [r(t)]_i^2.
  \end{align*}
  Now, note $h'_2(u) = 2\cosh(u) \ge 2$ for all $u \in \R$ and that
    \[h'_p(u) = \frac{p}{p-2}\left((1-u)^{-\frac{2p-2}{p-2}}+(1+u)^{-\frac{2p-2}{p-2}}\right) \ge \frac{2p}{p-2} > 0, \quad \text{for } p>2 \text{ and } u \in (-1,1).\] 
  Now, let 
  \[C_2 = \begin{cases} 2 &\text{if }p=2 \\ \frac{2p}{p-2} & \text{if } p>2\end{cases}, \]
  then
  \begin{align*}
    -2 \alpha^p C_1 \sum_{i=1}^N h'_p\left(C_1\int_0^t [r(t')]_i\ dt'\right) [r(t)]_i^2
    \le -2 \alpha^p C_1 C_2 \sum_{i=1}^N [r(t)]_i^2 = -2 \alpha^p C_1 C_2 \norm{r(t)}_2^2.
  \end{align*}
  Expanding $r(t)$ and using \eqref{sigma_min_prop} we can further bound
  \begin{align*}
    -2\alpha^p C_1 C_2 \norm{r(t)}^2_2 = -2\alpha^p C_1 C_2 \norm{A\T(y-A\psi_{\alpha}(t))}^2_2 \le -2\alpha^p C_1 C_2\sigma_{\min}^2(A) \norm{y-A\psi_{\alpha}(t)}_2^2.
  \end{align*}
    Gather the constants into $C = \alpha^p C_1 C_2 \sigma_{\min}^2(A) > 0$ we have
\begin{equation*}
    \frac{d}{dt}\norm{y- A\psi_{\alpha}(t)}_2^2 \le -2C\norm{y-A\psi_{\alpha}(t)}_2^2.
\end{equation*}
By Grönwall's inequality, this implies
    \[
        \norm{y- A\psi_{\alpha}(t)}_{2}^{2} \leq \norm{y - A\psi_{\alpha}(0)}_{2}^{2}\exp(-2Ct).
    \]
    Applying a square root on both sides, and using the fact that $\psi_{\alpha}(0)=0$, then gives 
    \begin{align*}
    \norm{A\psi_{\alpha}(t)-y}_2 \le \norm{y}_2\exp(-Ct). 
  \end{align*}
  Tracing back the factors of $C$ we get $C = 2p^2\sigma_{\min}^2(A)\alpha^{2p-2}$ for $p \ge 2$. Note that while $C_1$ and $C_2$ were divided into cases $p = 2$ and $p > 2$, the expression for $C$ holds for both cases.
\end{proof}

To prove stability in the second argument of $\mathcal{V}_p$, we will need the following bound on the condition number of the Hessian of $Q_p$.

\begin{lemma}\label{kappa_bound}
  Let $z \in \R^N$, $Q_p$ be as defined in \eqref{Q_p_def}, $p \ge 2$, $\alpha > 0$ and $\kappa(A) = \frac{\norm{A}_\mathrm{op}}{\sigma_{\min}(A)}$ for positive definite matrices $A$, then $\kappa(\nabla^2 Q_p(z)) \le \frac{1}{2}\left(\frac{1}{\alpha^p}\norm{z}_\infty+2\right)^\frac{2p-2}{p}$.
\end{lemma}
\begin{proof}
  From the definitions of $Q_p$ and $q_p$ we calculate the hessian of $Q_p$. At a point $z \in \R^N$, it is a diagonal matrix whose $i$th diagonal element is
  \begin{align*}
    \nabla^2 Q_p(z)_{ii} &= \alpha^{-p}q''\left(\frac{z_i}{\alpha^p}\right) = \alpha^{-p}(h_p^{-1})'\left(\frac{z_i}{\alpha^p}\right) = \frac{1}{\alpha^{p}h_p'(h_p^{-1}(\frac{z_i}{\alpha^p}))}.
  \end{align*}
  Since $h_p'$ is positive, the claim is equivalent to
  \[\kappa(\nabla^2 Q_p(z)) = \frac{\max_{i \in \{1,\dots,N\}}h_p'(h_p^{-1}(\frac{z_i}{\alpha^p}))}{\min_{i \in \{1,\dots,N\}}h_p'(h_p^{-1}(\frac{z_i}{\alpha^p}))} \le \frac{1}{2}\left(\frac{1}{\alpha^p}\norm{z}_\infty+2\right)^\frac{2p-2}{p}.\]
  First, consider the case $p = 2$. Note that $h_2'(t) = e^t+e^{-t} \le \left|e^{|t|}-e^{-|t|}\right|+2 = |h_2(t)|+2$ for $t \in \R$. So $h_2'(h_2^{-1}(u)) \le |u|+2$ for $u \in \R$. Combining with $h_2'(t) \ge 2$, we have for each $i \in \{1,\dots,N\}$ that $2 \le h_2'(h_2^{-1}\left(\frac{z_i}{\alpha^p}\right)) \le \frac{1}{\alpha^p}\norm{z}_\infty+2$, which implies the claim for $p = 2$.

  Next, let $p > 2$ and $t \in (-1,1)$, then
  \begin{align*}
    h_p'(t) &= \frac{p}{p-2}\left((1-|t|)^{-\frac{2p-2}{p-2}}+(1+|t|)^{-\frac{2p-2}{p-2}}\right)
    \le \frac{p}{p-2}\left((1-|t|)^{-\frac{2p-2}{p-2}}+1\right)\\
            &\le \frac{p}{p-2}\left((|h_p(t)|+1)^{\frac{2p-2}{p}}+1\right)
            \le \frac{p}{p-2}(|h_p(t)|+2)^{\frac{2p-2}{p}}, 
  \end{align*}
  where the last inequality used $u^x+1 \le (u+1)^x$ when $u \ge 0$ and $x \ge 1$. Inserting $t = h_p^{-1}(u)$ with $u \in \R$, we deduce $h_p'(h_p^{-1}(u)) \le \frac{p}{p-2}(|u|+2)^{\frac{2p-2}{p}}$. Combining with $h_p'(u) \ge \frac{2p}{p-2}$, we have for each $i \in \{1,\dots,N\}$ that $\frac{2p}{p-2} \le h_p'(h_p^{-1}\left(\frac{z_i}{\alpha^p}\right)) \le \frac{p}{p-2}(\frac{1}{\alpha^p}\norm{z}_\infty+2)^{\frac{2p-2}{p}}$, which implies the claim.
\end{proof}

The following bound proves that for bounded inputs, $\mathcal{V}_p$ is Lipschitz continuous in its second argument.
\begin{lemma}\label{Vp_lip}
  Let $A \in \R^{m\times N}$ with $\mathrm{rank}(A) = m$, $p \ge 2$ and $a,b \in \R^m$. Then
  \begin{align*}
    \norm{\mathcal{V}_p(A,a)-\mathcal{V}_p(A,b)}_2 \le C \norm{a-b}_2,
  \end{align*}
  where $C \coloneqq \frac{1}{2\sigma_{\min}(A)}\left(\frac{\sqrt N}{\alpha^p\sigma_{\min}(A)}\max\{\norm{a}_2,\norm{b}_2\}+2\right)^\frac{2p-2}{p}$. In particular, $\mathcal{V}_p$ is continuous with respect to its second argument.
\end{lemma}
\begin{proof}
  Let $z \in \R^m$. By the KKT conditions (\Cref{KKT}) for the definition of $\mathcal{V}_p(A,z)$ \eqref{Vp_def}, we can find $\lambda \in \R^m$ such that
  \begin{align*}
    \nabla Q(\mathcal{V}_p(A,z)) = A\T\lambda \text{ and } A\mathcal{V}_p(A,z) = z.
  \end{align*}
  Multiplying by $\PN$ gives $\PN\nabla Q(\mathcal{V}_p(A,z)) = 0$.

  Inserting $a$ and $b$ for $z$, we find $v \coloneqq \mathcal{V}_p(A,a)$ and $w \coloneqq \mathcal{V}_p(A,b)$ satisfying $\PN\nabla Q(v) = \PN\nabla Q(w) = 0$ and $Av = a, Aw = b$. Combining, and integrating the line segment between $v$ and $w$, we get
  \begin{align*}
    0 &= \PN(\nabla Q(w)-\nabla Q(v)) = \PN\int_0^1 \nabla^2 Q(v+t(w-v)) (w-v)\ dt = \PN H (w-v).
  \end{align*}
  Here $H \coloneqq \int_0^1 \nabla^2 Q(v+t(w-v))\ dt$. Assume $H$ is positive definite, which will be proved later. Additionally, $w$ and $v$ satisfy $A(w-v) = b-a$. Since $H$ is positive definite, solving for $w-v$ gives a unique solution, which is
  \begin{align*}
    w-v = JA\T(AA\T)^{-1}(b-a), \text{ where } J \coloneqq \PN^\perp-\PN(\PN H \PN)^\dagger \PN H\PN^\perp.
  \end{align*}
  Here $^\dagger$ denotes the pseudoinverse.

  Next, let $D \coloneqq \PN H \PN+\PN^\perp H \PN^\perp$. Note that $D-\sigma_{\min}(H)I = \PN(H-\sigma_{\min}(H)I)\PN + \PN^\perp (H-\sigma_{\min}(H)I)\PN^\perp$ is positive semi-definite, so $\sigma_{\min}(D) \ge \sigma_{\min}(H)$. Additionally, we will need the identity $DJ = (\PN^\perp-\PN)H\PN^\perp$. This follows from the identity $\PN H \PN(\PN H \PN)^\dagger = \PN$, which is true since $H$ is positive definite. Now
  \begin{align*}
    \norm{J}_\mathrm{op} &= \norm{D^{-1} DJ}_\mathrm{op} \le \norm{D^{-1}}_\mathrm{op}\norm{DJ}_\mathrm{op} = \norm{D^{-1}}_\mathrm{op}\norm{(\PN^\perp-\PN)H\PN^\perp}_\mathrm{op}\\
                                   &\le \norm{D^{-1}}_\mathrm{op}\norm{\PN^\perp-\PN}_\mathrm{op}\norm{H}_\mathrm{op}\norm{\PN^\perp}_\mathrm{op} \le \norm{H^{-1}}_\mathrm{op}\norm{H}_\mathrm{op} = \kappa(H).
  \end{align*}
  We just proved that $\norm{J}_\mathrm{op} \le \kappa(H)$, so
  \begin{align*}
    \norm{\mathcal{V}_p(A,b)-\mathcal{V}_p(A,a)}_2 &= \norm{w-v}_2 = \norm{JA\T(AA\T)^{-1}(b-a)}_2\\
                                                   &\le \frac{\norm{J}_\mathrm{op}}{\sigma_{\min}(A)}\norm{b-a}_2 \le \frac{\kappa(H)}{\sigma_{\min}(A)}\norm{b-a}_2.
  \end{align*}

  We finish the proof by showing that $H$ is positive definite and $\kappa(H) \le \sigma_{\min}(A)C$. Fix $t \in [0,1]$. Introduce the shorthand $H_t \coloneqq \nabla^2 Q(v+t(w-v))$ such that $H = \int_0^1 H_t\ dt$. By \Cref{kappa_bound}
  \begin{align}\label{kappa_Ht_bound}
    \kappa(H_t) = \kappa(\nabla^2 Q(v+t(w-v))) \le \frac{1}{2}\left(\frac{1}{\alpha^p}\norm{v+t(w-v)}_\infty+2\right)^\frac{2p-2}{p}.
  \end{align}
Note that 
  \Cref{Vp_bound} gives 
  \[
    \norm{v}_\infty = \norm{\mathcal{V}_p(A,a)}_\infty = \norm{\mathcal{V}_p(A,AA\T(AA\T)^{-1}a)}_\infty \le \norm{A\T(AA\T)^{-1}a}_1 \le \frac{\sqrt N}{\sigma_{\min}(A)}\norm{a}_2.
  \]
  Similarly, $\norm{w}_\infty \le \frac{\sqrt N}{\sigma_{\min}(A)}\norm{b}_2$. By convexity of $\norm{\cdot}_\infty$, we have 
  \[
    \norm{v+t(w-v)}_\infty \le \frac{\sqrt N}{\sigma_{\min}(A)}\max\{\norm{a}_2,\norm{b}_2\}.
  \]
  Inserting into \eqref{kappa_Ht_bound} yields
  \begin{align*}
    \kappa(H_t) \le \frac{1}{2}\left(\frac{1}{\alpha^p}\frac{\sqrt N}{\sigma_{\min}(A)}\max\{\norm{a}_2,\norm{b}_2\}+2\right)^\frac{2p-2}{p} = \sigma_{\min}(A)C.
  \end{align*}
  Now
  \begin{align*}
    \norm{H}_\mathrm{op} &= \norm{\int_0^1 H_t\ dt}_\mathrm{op} \le \int_0^1 \norm{H_t}_\mathrm{op}\ dt = \int_0^1 \kappa(H_t)\sigma_{\min}\left(H_t\right)\ dt\\
                         &\le \int_0^1 \sigma_{\min}(A)C\sigma_{\min}\left(H_t\right)\ dt \le \sigma_{\min}(A)C\sigma_{\min}\left(\int_0^1 H_t\ dt\right) = \sigma_{\min}(A)C\sigma_{\min}(H), 
  \end{align*}
  which shows that $\sigma_{\min}(H) > 0 \implies H$ is positive definite and $\kappa(H) = \frac{\norm{H}_\mathrm{op}}{\sigma_{\min}(H)} \le \sigma_{\min}(A)C$ as claimed.
\end{proof}

\subsubsection{\Cref{prop:main1} part \ref{it:m1} -- Precise statement }
With these results at hand, we are now ready to state and prove a precise version of \Cref{prop:main1} part \ref{it:m1}.
\begin{theorem}\label{psi_convergence}
  Let $A \in \R^{m\times N}$ with rank$(A) = m, y \in \R^m$. Let $\psi_{\alpha}(t)$ be the state of the gradient flow problem \eqref{eq:grad_flow} for some $p \ge 2, \alpha > 0$ at time $t$. Then for all $t \in [0,\infty)$
  \begin{align*}
    \norm{\psi_{\alpha}(t)-\psi_{\alpha}(\infty)}_2 \le K_1\exp(-K_2t),
  \end{align*}
  where $K_1 = \alpha^p\left(\frac{2\sqrt N\norm{y}_2}{\sigma_{\min}(A)\alpha^p}+2\right)^\frac{3p-2}{p}$ and $K_2 = 2p^2\sigma_{\min}^2(A)\alpha^{2p-2}$.
\end{theorem}
\begin{proof}
  Let $t \ge 0$. By \Cref{Apsi_determines_psi} we have $\mathcal{V}_p(A,A\psi_{\alpha}(t)) = \psi_{\alpha}(t)$. Taking the limit, we get $\psi_{\alpha}(\infty) = \lim_{t \to \infty} \mathcal{V}_p(A,A\psi_{\alpha}(t)) = \mathcal{V}_p(A,A\psi_{\alpha}(\infty))$, since $A\psi_{\alpha}(t)$ converges by \Cref{Apsi_to_y}, and $\mathcal{V}_p$ is continuous in its second argument by \Cref{Vp_lip}.

  \Cref{Apsi_to_y} gives $\norm{A\psi_{\alpha}(t)-y}_2 \le \norm{y}_2 \implies \norm{A\psi_{\alpha}(t)}_2 \le 2\norm{y}_2$ and $\norm{A\psi_{\alpha}(\infty)}_2 = \norm{y}_2$. Together, this means we can apply \Cref{Vp_lip} with $a = A\psi_{\alpha}(t)$ and $b = A\psi_{\alpha}(\infty)$ to get
  \begin{align}\label{eq:psi_diff_intermediate}
    \norm{\psi_{\alpha}(t)-\psi_{\alpha}(\infty)}_2 \le \frac{1}{2\sigma_{\min}(A)}\left(\frac{2\sqrt N}{\alpha^p\sigma_{\min}(A)}\norm{y}_2+2\right)^\frac{2p-2}{p} \norm{A\left(\psi_{\alpha}(t) - \psi_{\alpha}(\infty)\right)}_2.
  \end{align}
  Now, by \Cref{Apsi_to_y} we have $A\psi_{\alpha}(\infty) = y$ and $\norm{A\psi_{\alpha}(t)-y}_2 \le \norm{y}_2\exp(-K_2t)$ where $K_2 = 2p^2\sigma_{\min}^2(A)\alpha^{2p-2}$. Combining, that is $\norm{A(\psi_{\alpha}(t)-\psi_{\alpha}(\infty))}_2 \le \norm{y}_2\exp(-K_2t)$. Inserting into \eqref{eq:psi_diff_intermediate}, we get
  \begin{align*}
    \norm{\psi_{\alpha}(t)-\psi_{\alpha}(\infty)}_2 \le \frac{1}{2\sigma_{\min}(A)}\left(\frac{2\sqrt N}{\sigma_{\min}(A)\alpha^p}\norm{y}_2+2\right)^\frac{2p-2}{p} \norm{y}_2\exp(-K_2t) \le K_1\exp(-K_2t),
  \end{align*}
  where $K_1 = \alpha^p\left(\frac{2\sqrt N\norm{y}_2}{\sigma_{\min}(A)\alpha^p}+2\right)^\frac{3p-2}{p}$.
\end{proof}

\subsection{Proof of \Cref{thm:grad_flow} part \ref{it:dep_bound}}

\subsubsection{A constant by Todd and Stewart}\label{s:cond}
For matrices $B \in \R^{m\times N}$ with rank $m$, we let 
\[\chi_{B} = \sup\{\|(BDB^{\top})^{-1}BD\|_{\mathrm{op}}: \text{ where } D \text{ is a $N\times N$ positive definite diagonal matrix} \}.\]
The quantity $\chi_{B}$ appears many places in the literature \cite{Stewart89, todd1990dantzig, vavasis1996stable, vavasis1994stable, vavasis1996primal}, and it has been proved independently by Todd \cite{todd1990dantzig} and Stewart \cite{Stewart89} that $\chi_{B}$ is finite. Moreover, for any non-singular matrix $Q \in \R^{m\times m}$ we have that $\chi_{QB} = \chi_{B}$. That is,  $\chi_{B}$ is invariant to left multiplication by non-singular matrices \cite{vavasis1996primal}. In particular, this means that $\chi_{B}$ depends on the nullspace of $B$, rather than $B$ itself.  

Now, for a matrix $A \in \R^{m\times N}$ let $m' = \mathrm{dim}(\mathcal{N}(A))$ denote the dimension of the nullspace of $A$ and assume that $m' > 0$. Let $\tilde B \in \R^{N\times m'}$ be a matrix whose columns form a basis for the nullspace of $A$. We then define the quantity 
\begin{equation}
    \mathcal{K}(A) = \begin{cases} 1 & \text{if }\mathcal{N}(A) \text{ is trivial} \\
      \chi_{\tilde B^{\top}}+1 &\text{if } \mathcal{N}(A) \text{is non-trivial} 
 \end{cases}.
\end{equation}
It follows from the discussion above, that $\mathcal{K}(A)$ is independent of how we choose the columns of $B$. 
In the lemma below, we shall see that $\mathcal{K}(A)$ naturally appears. Before we state the theorem, we recall that $\mathcal{N}(A)$ denotes the nullspace of $A$ and $P_{\mathcal{N}(A)}$ denotes the projection onto the nullspace.  
\begin{lemma}\label[lemma]{K_lemma}
    Let $A \in \R^{m\times N}$ and let $D \in \R^{N \times N}$ be a diagonal matrix with strictly positive diagonal. Let $u,v \in \R^N$ and suppose that 
  \begin{align}\label{K_lemma_condition}
    P_{\Null(A)}D (u+P_{\Null(A)}v) = 0.
  \end{align}
  Then
  \begin{align}\label{eq:K_bound}
    \norm{u+P_{\Null(A)}v}_2 \leq \mathcal{K}(A)\norm{u}_2.
  \end{align}
\end{lemma}
\begin{proof}
  First, observe that if $\Null(A) = \{0\}$, then $\mathcal{K}(A) = 1$ and the statement is trivially true. Therefore, assume the nullspace of $A$ is non-trivial and let $m' = \dim (\Null(A)) > 0$ denote the dimension of the nullspace. Let $B \in \R^{N\times m'}$ be a matrix whose columns form an orthonormal basis for $\Null(A)$. Observe that $B\T B = I$ and $B B\T = P_{\Null(A)}$.
Next, note that $B\T D B$ is a symmetric positive definite matrix and hence invertible. Rewriting \eqref{K_lemma_condition} to $BB\T D (u+BB\T v) = 0$ and solving with respect to $B\T v$ gives
  \[ B\T v = -(B\T D B)^{-1}B\T Du. \]
  Inserting this into the left hand side of \eqref{eq:K_bound} yields
  \[
    \norm{u+BB\T v}_2 = \norm{u-B(B\T D B)^{-1}B\T Du}_2 \le \norm{I-B(B\T D B)^{-1}B\T D}_{\mathrm{op}}\norm{u}_2.
  \]
  Next, we bound the operator norm
  \begin{align*}
    \norm{I-B(B\T D B)^{-1}B\T D}_\mathrm{op} &\le \norm{B(B\T D B)^{-1}B\T D}_{\mathrm{op}}+1 \le \norm{(B\T D B)^{-1}B\T D}_{\mathrm{op}}+1\\
                                              &\leq \chi_{B\T}+1 = \mathcal{K}(A). 
  \end{align*}
  Here we used the fact that $\|B\|_{\mathrm{op}} = 1$ for the second inequality.
\end{proof}

\subsubsection{Loose bounds for \texorpdfstring{$q^*, m^*$ and $\psi_{\alpha}$}{}}\label{loose}
We use the general bound on $\mathcal{V}_p(A,Av)$ from \Cref{Vp_bound} to get loose bounds on $q^*$,$m^*$ and $\psi_\alpha$. This will be useful for example in case $\alpha$ is large in \Cref{qm_bound} and $p = 2$ in \Cref{thm:grad_flow}.

\begin{lemma}\label[lemma]{q_infty_bound}
  Let $A \in \R^{m\times N}$ with rank$(A) = m$ and $y \in \R^m$. Let $q^*$ and $m^*$ be given by \eqref{eq:q*} and \eqref{eq:m*}. Then
  \begin{align*}
    \norm{q^*}_\infty \le \norm{m^*}_1 \le \frac{\sqrt N}{\sigma_{\min}(A)}\norm{y}_2.
  \end{align*}
\end{lemma}
\begin{proof}
  Applying \Cref{Vp_bound} with $v = m^*$ yields $\norm{q^*}_\infty = \norm{\mathcal{V}_p(A,y)}_\infty = \norm{\mathcal{V}_p(A,Am^*)}_\infty \le \norm{m^*}_1$.

  Let $w = A\T (AA\T)^{-1}y$. Now -- by using \Cref{sigma_min_prop} -- we see that \[Aw = v \implies \norm{y}_2 \ge \sigma_{\min}(A)\norm{w}_2.\] Additionally, since $m^* \in \mathcal{U}(A,v)$, it minimizes the $\ell_1$ norm, so we have
  \begin{align*}
    \norm{m^*}_1 \le \norm{w}_1 \le \sqrt N\norm{w}_2 \le \frac{\sqrt N}{\sigma_{\min}(A)}\norm{y}_2.
  \end{align*}
\end{proof}

\begin{proposition}\label{Vp_bound2}
  Let $A \in \R^{m\times N}$ with $\mathrm{rank}(A) = m$, $p \ge 2$ and $v \in \R^m$. Then
   $ \norm{\mathcal{V}_p(A,v)}_\infty \le \sqrt N\sigma_{\min}^{-1}(A)\norm{v}_2$.
\end{proposition}
\begin{proof}
  Apply \Cref{Vp_bound} with $w = A\T (AA\T)^{-1} v$ to get
  \begin{align*}
    \norm{\mathcal{V}_p(A,v)}_\infty &= \norm{\mathcal{V}_p(A,Aw)}_\infty \le \norm{w}_1 = \norm{A\T (AA\T)^{-1}v}_1\\
                                     &\le \sqrt N\norm{A\T (AA\T)^{-1}}_\mathrm{op}\norm{v}_2 \le \frac{\sqrt N}{\sigma_{\min}(A)}\norm{v}_2.
  \end{align*}
\end{proof}

\begin{corollary}\label{psi_bound}
  Let $A \in \R^{m\times N}$ with rank$(A) = m, y \in \R^m$. Let $\psi_{\alpha}(t)$ be the state of the gradient flow problem \eqref{eq:grad_flow} for some $p \ge 2, \alpha > 0$ at time $t$. Then for all $t \in [0,\infty)$
  \begin{align*}
    \norm{\psi_{\alpha}(t)}_\infty \le \frac{2\sqrt N}{\sigma_{\min}(A)}\norm{y}_2.
  \end{align*}
\end{corollary}
\begin{proof}
  By \Cref{Apsi_to_y} we have for all $t \ge 0$ that $\norm{A\psi_{\alpha}(t)-y}_2 \le \norm{y}_2$. So by the reverse triangle inequality $\norm{A\psi_{\alpha}(t)}_2 \le 2\norm{y}_2$. By \Cref{Apsi_determines_psi} we can write $\psi_{\alpha}(t) = \mathcal{V}_p(A,A\psi_{\alpha}(t))$. We apply \Cref{Vp_bound2} with $v = A\psi_{\alpha}(t)$ to get
\begin{align*}
  \norm{\psi_{\alpha}(t)}_\infty &= \norm{\mathcal{V}_p(A,A\psi_{\alpha}(t))}_\infty \le \frac{\sqrt N}{\sigma_{\min}(A)}\norm{A\psi_{\alpha}(t)}_2 \le \frac{2\sqrt N}{\sigma_{\min}(A)}\norm{y}_2.
\end{align*}
\end{proof}

\subsubsection{Proving \texorpdfstring{$\mathcal{W}_p(A,y) = g^*$ and $\psi_\alpha(\infty) = q^*$}{}}

\begin{proposition}\label[proposition]{Wp_Gp_id}
  Let $A \in \R^{m\times N}$ with rank$(A) = m$, $y \in \R^m$ and $p\geq 2$. 
    Let $g^*$ and $\mathcal{W}_{p}(A,y)$ be given by \eqref{eq:g*} and \eqref{eq:W_p}. The minimizer in \eqref{eq:W_p} is unique, and we have the equality $g^* = \mathcal{W}_p(A,y)$.
\end{proposition}
\begin{proof}
  From \cref{M_s_lemma}, we know that $\mathcal{U}(A,y)$ is a compact and convex set. Using the Extreme Value Theorem and the continuity of $H$ in \cref{eq:H_def} and $\|\cdot\|_{2/p}$, $p > 2$, we can see that the maximum in \cref{eq:W_p} is always attained. We will show that $\mathcal{W}_p(A,y) = g^*$, which together with \Cref{cor:un_g} will show that $\mathcal{W}_p(A,y)$ is single valued.
We first prove the claim for $p > 2$. Observe that $t \mapsto R - C t^{2/p}$, $t\geq 0$ is a strictly decreasing function whenever $C> 0, R\in \R$ and $p>2$. This implies that
    \begin{align*}
      \mathcal{W}_p(A,y) =&\argmax_{z \in \mathcal{U}(A,y)} \norm{z}_{2/p}
      = \argmin_{z \in \mathcal{U}(A,y)} -\norm{z}_{2/p}
      = \argmin_{z \in \mathcal{U}(A,y)} R - C \norm{z}_{2/p}^{2/p}, 
    \end{align*}
    for $p > 2$. 
    Now let $C = \tfrac{p}{2}\alpha^{2-p}$ and $R=R(A,y)$ be the minimum of the basis pursuit problem in \eqref{eq:BP_min}.
    By definition we have that $\|z\|_1 = R$ for all $z \in \mathcal{U}(A,y)$. Furthermore, we know that $\alpha^p g_p\left(\frac{u}{\alpha^p}\right) = |u|-\frac{p}{2}\alpha^{p-2}|u|^{2/p}$ for $u\in\R$. Combining these facts, we get
    \begin{align*}
      \mathcal{W}_p(A,y) 
      &= \argmin_{z \in \mathcal{U}(A,y)} \|z\|_1 - \frac{p}{2}\alpha^{p-2} \norm{z}_{2/p}^{2/p} = \argmin_{z \in \mathcal{U}(A,y)}\sum_{i=1}^{N} |z_i|-\frac{p}{2}\alpha^{p-2}|z_i|^{2/p}  \\
      &= \argmin_{z \in \mathcal{U}(A,y)} \sum_{i=1}^{N} \alpha^{p}g_p \left(\frac{z_i}{\alpha^p}\right) = \argmin_{z \in \mathcal{U}(A,y)} G_{\alpha,p}(z).
    \end{align*}
    This proves the claim for $p > 2$.

    Next, we consider the case where $p = 2$. Then for any constant $K \in \R$, we have that 
    \begin{align*}
      \mathcal{W}_2(A,y) = \argmax_{z \in \mathcal{U}(A,y)} H(z) = \argmin_{z \in \mathcal{U}(A,y)} -H(z) - K 
    \end{align*} 
    Now, take  $K=\ln(e\alpha^2) R$, where $R$ is as in $\eqref{eq:BP_min}$, and observe that $\alpha^2 g_2\left(\frac{u}{\alpha^2}\right) = |u|\big(\ln|u|-\ln(e\alpha^2)\big)$.
    Then 
    \begin{align*}
      \mathcal{W}_2(A,y)   
      &= \argmin_{z \in \mathcal{U}(A,y)} -H(z) - \ln(e\alpha^2)\norm{z}_1 
      = \argmin_{z \in \mathcal{U}(A,y)} \sum_{i=1}^N \left( |z_i|\ln(|z_i|) - \ln(e\alpha^2)|z_i|\right) \\
      &= \argmin_{z \in \mathcal{U}(A,y)} \sum_{i=1}^N \alpha^p g_p\left(\frac{z_i}{\alpha^p}\right)
      = \argmin_{z \in \mathcal{U}(A,y)} G_p(z),
  \end{align*}
  proves the case for $p = 2$.
\end{proof}

\begin{proposition}\label{psi_infty_eq_q}
  Let $A \in \R^{m\times N}$ with rank$(A) = m$, $y \in \R^m, \alpha > 0$ and $p \geq 2$. Let $\psi_{\alpha}(\infty)$ be as defined in \eqref{psi_def} and $q^*$ be given by \eqref{eq:q*}. Then $\psi_{\alpha}(\infty) = q^*$.
\end{proposition}
\begin{proof}
  By \Cref{Apsi_determines_psi}, continuity of the second argument of $\mathcal{V}_p$ by \Cref{Vp_lip}, $A\psi_{\alpha}(t) \to y$ by \Cref{Apsi_to_y}, and finally the definition of $q^*$, we have
  \begin{align*}
    \psi_{\alpha}(\infty) = \lim_{t \to \infty} \psi_{\alpha}(t) = \lim_{t \to \infty} \mathcal{V}_p(A,A\psi_{\alpha}(t)) = \mathcal{V}_p(A,y) = q^*.
  \end{align*}
\end{proof}

\subsubsection{Bounding the distance between \texorpdfstring{$g^*$ and $m^*$}{g* and m*}}
\label{ss:bound_gm}
In order to prove \Cref{thm:grad_flow}, we need to bound the distance $\norm{q^*-g^*}_2$. To do so, we use $m^*$ as an intermediate point. In this section, we bound the distance between $g^*$ and $m^*$, and in \Cref{ss:bound_qm} we bound the distance between $m^*$ and $q^*$.
To bound $\norm{m^*-g^*}_2$ in \Cref{gm_bound}, we need two lemmas describing certain properties of the set $\mathcal{U}(A,y)$. The first lemma says that it lies in a single signed orthant.

\begin{lemma}\label[lemma]{M_s_lemma}
  Let $A \in \R^{m\times N}$ with $\rank(A)=m$, $y \in \R^m$, and let $R=R(A,y)$ be the minimum of the basis pursuit problem \eqref{eq:BP_min}.
  Then there exists $s \in \{-1,1\}^N$ such that the following set equality holds
  \[
    \mathcal{U}(A,y) = \{z \in \mathcal{R}(s) : Az = y,\ s\T z = R\}.
  \]
  That is, the set of $\ell_1$ minimizers $\mathcal{U}(A,y)$ is contained in a signed orthant. Moreover, the set $\mathcal{U}(A,y)$ is closed, convex and bounded. 
\end{lemma}

\begin{proof}
  First, we claim we can find $s$ such that $\mathcal{U}(A,y) \subset \mathcal{R}(s)$. Assume for a contradiction that no such $s$ exists. Then there exist $a,b \in \mathcal{U}(A,y)$ and $k \in \{1,\dots,N\}$ such that $a_k b_k < 0$. Let $c = \frac{1}{2}(a+b)$. Clearly $Ac = y$. Furthermore, 
  \begin{align*}
    2\norm{c}_1 = \sum_{i = 1}^N|a_i+b_i| \le |a_k+b_k| + \sum_{\substack{i=1\\i\ne k}}^N |a_i|+|b_i| < \norm{a}_1+\norm{b}_1 = 2R.
  \end{align*}
  This contradicts $R$ being the minimum $\ell_1$ norm as defined in \eqref{eq:BP_min}.
Hence, we can pick $s \in \{-1,1\}^N$ such that $\mathcal{U}(A,y) \subset \mathcal{R}(s)$. Note $s\T z = \norm{z}_1\ \forall z \in \mathcal{R}(s)$. Using this we may characterize the set
  \begin{align*}
    \mathcal{U}(A,y) &= \{z \in \R^N : Az = y, \norm{z}_1 = R\} = \{z \in \mathcal{R}(s) : Az = y, \norm{z}_1 = R\}\\
                     &= \{z \in \mathcal{R}(s) : Az = y, s\T z = R\}.
  \end{align*}
The fact that $\mathcal{U}(A,y)$ is closed and convex, is seen from the equality constraints above. Boundedness follows from the fact that the $\ell^1$-ball is bounded.
\end{proof}

The next lemma shows the connection between the support of $g^*$ and the supports of elements in $\mathcal{U}(A,y)$. 
\begin{lemma}\label[lemma]{g0_M0_lemma}
  Let $A \in \R^{m\times N}$ with rank$(A) = m$ and $y \in \R^m$. 
    Let $g^*$ be given by \eqref{eq:g*}.
    If $g^*_k = 0$ for some $k \in \{1,\dots,N\}$, then $z_k = 0$ for all $z \in \mathcal{U}(A,y)$.
\end{lemma}
\begin{proof}
  Let $k \in \{1,\ldots, N\}$, and assume for contradiction that $g^*_k = 0$ and there is a $z \in \mathcal{U}(A,y)$, with $z_k \ne 0$. We will show this contradicts the optimality of $g^*$. Consider $\lambda \in (0,1)$, and let $w(\lambda) = (1-\lambda)g^* + \lambda z$. From \Cref{M_s_lemma} it is clear that $w(\lambda) \in \mathcal{U}(A,y)$ for all $\lambda \in (0,1)$.
Let $f(x) = \alpha^p g_p\left(\frac{x}{\alpha^p}\right)$  for $x \in \R$, and notice that $f$ is a strictly convex function, when restricted to one of the intervals $[0,\infty)$ or $(-\infty, 0]$. Moreover, $f$ is differentiable on the open intervals $(0,\infty)$ and $(-\infty, 0)$, but $\lim_{x \to 0^+} \frac{f(x)-f(0)}{x} = \lim_{x \to 0^+} \frac{f(-x)-f(0)}{x} = -\infty$.
Next, let $i \in \supp(g^*)$, and observe that by choosing $\lambda$ sufficiently small, we can ensure that $\lambda |z_i-g^*_i| \le \tfrac{1}{2}|g^*_i|$. For any such choice of $\lambda$, we have $|w_i(\lambda)|\geq |g_{i}^{*}| - \lambda |z_i -g_{i}^{*}| \geq \tfrac{1}{2}|g^{*}_{i}|$, which implies $w_i(\lambda)$ and $g_{i}^{*}$ have the same sign. Using the convexity and differentiability of $f$ on the intervals $(0,\infty)$ and $(-\infty, 0)$ we see that $f(w(\lambda)_i) - f(g_{i}^*) > f'(g_i^*)(w(\lambda)_i - g_i^*)=f'(g_i^*)\lambda(z_i - g_i^*) $ for $\lambda>0$ sufficiently small and $i \in \supp(g^*)$. Next, notice that 
  \begin{align*}
      \frac{G_p(w(\lambda))-G_p(g^*)}{\lambda} &= \sum_{i=1}^N \frac{f(w(\lambda)_i)-f(g^*_i)}{\lambda} \\
      &> \sum_{i\in \supp(g^*)}  f'(g_i^*)(z_i - g_i^*) + \sum_{i \in \overline{\supp(g^*)}} \frac{f(\lambda z_i)-f(0)}{\lambda}.
\end{align*}
Recall $\lim_{x \to 0^+} \frac{f(x)-f(0)}{x} = \lim_{x \to 0^+} \frac{f(-x)-f(0)}{x} = -\infty$. If $z_i \neq 0$ for $i \in \overline{\supp(g^*)}$, it is clear that we can make $\frac{G_p(w(\lambda))-G_p(g^*)}{\lambda} < 0$ by choosing $\lambda>0$ sufficiently small. This contradicts the optimality of $g^*$.
\end{proof}

We are now ready to bound the distance $\norm{m^*-g^*}_2$.
\begin{proposition}\label[proposition]{gm_bound}
    Let $A \in \R^{m\times N}$ with $\mathrm{rank}(A) = m$, $y \in \R^m$, $\alpha > 0$ and $p\geq 2$.
    Let $m^*$ and $g^*$ be given by \eqref{eq:m*} and \eqref{eq:g*}.
  Then there exists a constant $C_A > 0$ only depending on $A$ such that
  \begin{align*}
    \norm{m^*-g^*}_2 \le C_A\alpha^p.
  \end{align*}
\end{proposition}
\begin{proof}
  Let $\tilde{G}_p \colon \R^N \to \R$ be defined by 
  \begin{align*}
    \tilde{G}_p(z) = \alpha^p\sum_{i \in \supp(g^*)} g_p\left(\frac{z_i}{\alpha^p}\right), \quad \text{for } z \in \R^N.
  \end{align*}
    Since $g_p(0) = 0$, we know from \Cref{g0_M0_lemma} that $\tilde{G}_p(z) = G_p(z)$ for all $z \in \mathcal{U}(A,y)$. Furthermore, $\tilde{G}_p$ is differentiable at $g^*$. The same is not necessarily true for $G_p$, since $g_p$ is not differentiable at 0.

    By \Cref{M_s_lemma} we know there is an $s \in \{-1,1\}^N$ such that
  \begin{align}\label{gm_M_s_def}
    \mathcal{U}(A,y) &= \{z \in \mathcal{R}(s) : Az = y,\ s\T z = R\},
  \end{align}
    where $R$ is as in \cref{eq:BP_min}. Let $\tilde{A} \coloneqq \begin{bmatrix}A\\s\T\end{bmatrix}$. Then we may rewrite 
        \begin{align}\label{eq:M_def_mod}
    \mathcal{U}(A,y) = \left\{z \in \R^N : \tilde{A}z = \tilde{A}m^*,\ z_i s_i \ge 0\ \forall i \in \{1,\dots,N\}\right\}.
  \end{align}
  As stated above, by \Cref{g0_M0_lemma}, we have $\tilde{G}_p(z) = G_p(z)$ for all $z \in \mathcal{U}(A,y)$. Hence, using \eqref{eq:M_def_mod} to rewrite to a form which we can formulate KKT conditions for, we have
\begin{alignat*}{2}
    g^* &= \argmin_{z \in \R^N} G_p(z) &&\text{ subject to }\quad z \in \mathcal{U}(A,y)\\
        &= \argmin_{z \in \R^N} \tilde{G}_p(z) &&\text{ subject to }\quad z \in \mathcal{U}(A,y)\\
        &= \argmin_{z \in \R^N} \tilde{G}_p(z) &&\text{ subject to }\quad \tilde{A}z = \tilde{A}m^*,\ z_i s_i \ge 0\ \forall i \in \{1,\dots,N\}.
  \end{alignat*}
  Technically, we cannot apply \Cref{KKT} (KKT conditions) yet. Since it requires the function to be continuously differentiable on the whole feasible set. As $g_p$ is not differentiable at zero, $\tilde G_p$ is not differentiable at points $z \in \R^N$ such that there exists an $i \in \supp(g^*)$ with $z_i = 0$. We solve this problem by slightly shrinking the feasible set, so it satisfies $z_is_i \ge \frac{1}{2}|g^*_i|$. Since the minimizer $g^*$ satisfies this constraint, it remains the minimizer.
  \begin{alignat}{2}
    g^* &= \argmin_{z \in \R^N} \tilde{G}_p(z) &&\text{ subject to }\quad \tilde{A}z = \tilde{A}m^*,\ z_i s_i \ge 0\ \forall i \in \{1,\dots,N\}\nonumber\\
        &= \argmin_{z \in \R^N} \tilde{G}_p(z) &&\text{ subject to }\quad \tilde{A}z = \tilde{A}m^*,\ z_i s_i \ge \frac{1}{2}|g^*_i|\ \forall i \in \{1,\dots,N\}\label{g_kkt}.
  \end{alignat}
  Now, because $g^*$ is a minimizer of \eqref{g_kkt}, and $\tilde{G}_p$ is convex and continuously differentiable over the feasible set, $g^*$ has to satisfy the necessary KKT conditions stated in \Cref{KKT}. Specifically, there exist $\lambda^g \in \R^N$ and $\mu^g \in [0,\infty)^N$ such that
  \begin{alignat}{2}
    \nabla \tilde{G}_p(g^*) &= \tilde{A}\T\lambda^g + \diag(s)\mu^g && \label{g_kkt_grad0_first}\\
    (g^*_i s_i-\frac{1}{2}|g^*_i|) \mu^g_i &= 0 &&\quad \forall i \in \{1,\dots,N\}. \label{g_kkt_complementary_slackness}
  \end{alignat}
  Next, we separate the nullspace component of \eqref{g_kkt_grad0_first} by multiplying with $P_{\Null(\tilde{A})}$. We also simplify \eqref{g_kkt_complementary_slackness} using $g^*_i s_i = |g_i^*|$. We have
  \begin{align}
    P_{\Null(\tilde A)}(\nabla \tilde{G}_p(g^*) - \diag(s)\mu^g) &= 0\label{g_kkt_grad0}\\
    g^*_i \mu^g_i &= 0 \quad \forall i \in \{1,\dots,N\}.\label{g_mu0}
  \end{align}

  From \cref{eq:M_def_mod}, we have
  \begin{align*}
    m^*  = \argmin_{z \in \R^N} Q_p(z) \quad \text{ subject to }\quad \tilde{A}z = \tilde{A}m^*,\ z_i s_i \ge 0\ \forall i \in \{1,\dots,N\},
  \end{align*}
  and thus, by repeating the arguments above, there exists $\mu^m \in [0,\infty)^N$ such that
  \begin{align}
    P_{\Null(\tilde A)}(\nabla Q_p(m^*) - \diag(s)\mu^m) &= 0\label{m_kkt_grad0}\\
    m^*_i \mu^m_i &= 0 \quad \forall i \in \{1,\dots,N\}\nonumber.
  \end{align}

  Next, we claim there exists some $w \in \mathcal{R}(s)$, with $\supp(w) = \{1,\ldots,N\}$, satisfying
  \begin{align}
    \nabla G_p(w) &= \nabla Q_p(m^*)-\diag(s)\mu^m \label{w_grad}\\
    \norm{w-m^*}_\infty &\le \alpha^p\nonumber.
  \end{align}

    To prove this claim, we note that the $i$th equation in \cref{w_grad} is given by
  \begin{align*}
    g_p'\left(\frac{w_i}{\alpha^p}\right) = q_p'\left(\frac{m^*_i}{\alpha^p}\right)-\mu^m_i s_i.
  \end{align*}
    From \cref{gm_M_s_def}, we know $m^* \in \mathcal{U}(A,y) \subset \mathcal{R}(s)$, which implies $m^*_i s_i \ge 0$. 

  We first consider the case $m^*_i > 0$. Then $s_i = 1$ and $\mu^m_i = 0$, we need to find $w_i \in (0,\infty)$ such that $g_p'\left(\tfrac{w_i}{\alpha^p}\right) = q_p'\left(\tfrac{m^*_i}{\alpha^p}\right)$. By \Cref{qg'_bound} we have 
  \begin{align*}
    g_p'\left(\frac{m^*_i}{\alpha^p}\right) \le q_p'\left(\frac{m^*_i}{\alpha^p}\right) \le g_p'\left(\frac{m^*_i+\alpha^p}{\alpha^p}\right).
  \end{align*}
  Now, since $g_p'$ is continuous on $(0,\infty)$, the Intermediate Value Theorem says there exists $w_i \in [m^*_i,m^*_i+\alpha^p] \subset (0,\infty)$ such that $g_p'\left(\frac{w_i}{\alpha^p}\right) = q_p'\left(\frac{m^*_i}{\alpha^p}\right)$, and we have $|w_i-m^*_i| \le \alpha^p$ as desired. The case $m^*_i < 0$ is similar.

Next, we look at the case $m^*_i = 0$ and $s_i = 1$. Noting that $q'_p(0) = 0$, we need to find $w_i \in (0,\infty)$ such that $g_p'\left(\frac{w_i}{\alpha^p}\right) = -\mu^m_i s_i \in (-\infty,0]$. It is a simple case by case analysis to show that the range of $g'_p\big((0,1]\big) = (-\infty,0]$ for any $p \ge 2$. This implies there exists a $w_i \in (0,\alpha^p]$ such that $g_p'\left(\frac{w_i}{\alpha^p}\right) = -\mu^m_i s_i$. It is clear that $|w_i-m^*_i| = w_i \le \alpha^p$. The case where $m^*_i = 0$ and $s_i = -1$ follows a similar approach. This proves the claim.

  Next, we want to apply \Cref{K_lemma} to bound the distance $\|w-g^*\|_2$. Therefore, we define the $N\times N$ diagonal matrix $D$, whose $i$th diagonal entry is given by
  \begin{align*}
    D_{ii} = \begin{cases}
      \frac{g'_p\left(\frac{w_i}{\alpha^p}\right)+\mu^g_i s_i}{w_i} &\text{ if } i \in \overline{\supp(g^*)}\\
    \frac{g'_p\left(\frac{w_i}{\alpha^p}\right)-g'_p\left(\frac{g^*_i}{\alpha^p}\right)}{w_i-g^*_i} &\text{ if } i \in \supp(g^*) \text{ and } w_i \ne g^*_i\\
      1 &\text{ if } i \in \supp(g^*) \text{ and } w_i = g^*_i.
    \end{cases}
  \end{align*}
We claim that $D_{ii} > 0$ for all $i \in \{1,\dots,N\}$. To see this, start by observing $g_iw_i \geq 0$, since $g_i,w_i \in \mathcal{R}(s_i)$. Furthermore, since $g_p$ is strictly convex when restricted to either $(-\infty,0]$ or $[0,\infty)$, we have
\begin{equation} \label{eq:D_ii_ineq2}
  \frac{g'_p\left(\frac{w_i}{\alpha^p}\right)-g'_p\left(\frac{g^*_i}{\alpha^p}\right)}{w_i-g^*_i} > 0. 
\end{equation}
Now, since $g'_p(0) = 0$, and $w_i\in \mathcal{R}(s_i)$, it is clear from \eqref{eq:D_ii_ineq2} that $D_{ii} = (g'_p\left(\frac{w_i}{\alpha^p}\right)+\mu^g_i s_i)/w_i > \mu^g_is_i/w_i \ge 0$, when $i \in \overline{\supp(g^*)}$. If $i \in \supp(g^*)$ it is clear from \eqref{eq:D_ii_ineq2}, that $D_{ii} > 0$. This proves the claim. 

Next, from \eqref{g_mu0} we know $\mu^g_i = 0\ \forall i \in \supp(g^*)$. This makes it straightforward to verify
  \begin{align*}
    D(w-g^*) = \nabla G_p(w) - \nabla\tilde{G}_p(g^*) + \diag(\mu^g)s.
  \end{align*}
  From \cref{m_kkt_grad0} and \cref{w_grad}, we have $P_{\Null(\tilde A)}\nabla G_p(w) = 0$. Combining this with \eqref{g_kkt_grad0}, yields 
  \begin{align}\label{eq:PDwg}
    0 = P_{\Null(\tilde A)}(\nabla G_p(w) - \nabla\tilde{G}_p(g^*) + \diag(\mu^g)s) = P_{\Null(\tilde A)}D(w-g^*).
  \end{align}

  Since $m^*,g^* \in \mathcal{U}(A,y)$, we know $m^*-g^* \in \Null(\tilde{A})$. Using this in combination with \eqref{eq:PDwg}, gives
  \begin{align*}
    0 = P_{\Null(\tilde A)}D(w-g^*) = P_{\Null(\tilde A)}D(w-m^*\ +\ m^*-g^*) = P_{\Null(\tilde A)}D(w-m^*+P_{\Null(\tilde A)}(m^*-g^*)).
  \end{align*}
  Applying \Cref{K_lemma} with $u=w-m^*$ and $v = m^*-g^*$ yields
  \begin{align*}
    \norm{w-g^*}_2 \le \mathcal{K}(\tilde{A})\norm{w-m^*}_2.
  \end{align*}
  Finally, we bound
  \begin{align*}
    \norm{m^*-g^*}_2 &\le \norm{m^*-w}_2+\norm{w-g^*}_2 \le (\mathcal{K}(\tilde{A})+1)\norm{w-m^*}_2\\
                     &\le \sqrt N(\mathcal{K}(\tilde{A})+1)\norm{w-m^*}_\infty \le \sqrt N(\mathcal{K}(\tilde{A})+1) \alpha^p\\
                     &\le \sqrt N\left(\max_{s \in \{-1,1\}^N}\mathcal{K}\left(\begin{bmatrix}A\\s\T\end{bmatrix}\right)+1\right) \alpha^p.
  \end{align*}
\end{proof}

\subsubsection{Bounding the distance between \texorpdfstring{$q^*$ and $m^*$}{q* and m*}}
\label{ss:bound_qm}
The next lemma will be used in \Cref{qm_bound}, which presents the concrete bound on $\norm{q^*-m^*}_2$.

\begin{lemma}\label[lemma]{l2_l1_bound}
  Let $A \in \R^{m\times N}$ with $\mathrm{rank}(A) = m$ and $y \in \R^m$. Let $q^*$ and $m^*$ be given by \eqref{eq:q*} and \eqref{eq:m*}, respectively. Then
  \begin{equation*}
    \norm{q^*-m^*}_2 \le C_A(\norm{q^*}_1-\norm{m^*}_1)
  \end{equation*}
  where 
  \begin{equation}\label{eq:def_C_A}
    C_A = \max_{\substack{s \in \{-1,1\}^N \\ P_{\Null(A)} s \ne 0}} \frac{1}{\norm{P_{\Null(A)} s}_2}\mathcal{K}\left(\begin{bmatrix}A\\s\T\end{bmatrix}\right)
  \end{equation}
  if $A$ has non-trivial nullspace, and $C_A = 0$ if $\Null(A) = \{0\}$.
\end{lemma}
\begin{proof}
  Recall 
  \[ 
    q^* = \argmin_{z \in \{x \in \R^N : Ax=y\}} Q_p(z)
    \quad\text{and }\quad
    m^* = \argmin_{z \in \mathcal{U}(A,y)} Q_p(z).
  \]
  Throughout the proof, we will be working with a vector $s \in \R^N$, whose $i$th entry is given by
  \[
    s_i  \coloneqq
    \begin{cases}
      \sign(m^*_i) &\text{if } m^*_i \ne 0\\
      \sign(q^*_i) &\text{if } q^*_i \ne 0 \text{ and } m^*_{i}=0\\
      1 &\text{if } q^*_i = 0 \text{ and } m^*_{i}=0
    \end{cases}.
  \]
  Furthermore, let $R_{m^*} \coloneqq s^{\top}m^*$ and $R_{q^*} \coloneqq s^{\top}q^*$, and let
  \[ \mathcal{S}_r \coloneqq \{z \in \R^N : Az = y,\ s\T z = r,\ z_i s_i \ge 0\ \forall i \in \overline{\supp(m^*)}\},\quad\text{for } r\in \mathbb{R}. \]
  It is clear that $m^* \in \mathcal{S}_{R_{m^*}}$ and $q^* \in \mathcal{S}_{R_{q^*}}$.

  \begin{claim}\label{c:argmin_claims}
    We claim
    \begin{align}\label{m_W_and_q_W_claim}
      q^* = \argmin_{z \in \mathcal{S}_{R_{q^*}}} Q_p(z)
      \quad \text{and}\quad
      m^* = \argmin_{z \in \mathcal{S}_{R_{m^*}}} Q_p(z).
    \end{align}
  \end{claim}

  Start by observing that both minimizers in \eqref{m_W_and_q_W_claim} must be unique according to \cref{l:unique_min}. Moreover, since $\mathcal{S}_{R_{q^*}} \subset \{x\in \R^{N} : Ax = y\}$, and $q^* \in \mathcal{S}_{R_{q^*}}$, it is clear that $q^*$ is the minimizer in the first optimization problem in \eqref{m_W_and_q_W_claim}.

    To prove that $m^*$ is the minimizer in the second optimization problem in \eqref{m_W_and_q_W_claim} we argue by contradiction. Assume there is some $z^* \in \mathcal{S}_{R_{m^*}}$ where $z^* \ne m^*$ and $Q_p(z^*) < Q_p(m^*)$. Now, choose $\lambda \in (0,1)$ sufficiently small, so that $|m^*_i| > \lambda |z^*_i-m^*_i|$ for all $i \in \supp(m^*)$ and let $w = (1-\lambda)m^* + \lambda z^*$. By linearity $w \in \mathcal{S}_{R_{m^*}}$, and by strict convexity of $Q_p$ we have $Q_p(w) < (1-\lambda)Q_p(m^*)+\lambda Q_p(z^*) < Q_p(m^*)$. Furthermore, for all $i \in \supp(m^*)$ and for our choice of $\lambda$, we have that
  \begin{equation}
    \begin{split}\label{eq:signsplit}
      |(1-\lambda)m^*_i+\lambda z^*_i|-|m^*_i| &= |m_{i}^{*} + \lambda(z_{i}^{*} - m^{*}_i)| - |m_{i}^{*}| \\ &= \sign(m^*_i)\lambda(z^*_i-m^*_i) = \lambda s_i(z^*_i-m^*_i).
    \end{split}
  \end{equation}
  Now, since $z^* \in \mathcal{S}_{R_{m^*}}$, we know $|z^*_i| = s_i z^*_i \geq 0$ for $i\in \overline{\supp(m^*)}$. This implies
  \begin{equation} \label{eq:wmm*}
    \begin{split}
      \norm{w}_1-\norm{m^*}_1 
                            &= \sum_{i\in \overline{\supp (m^*)}} \lambda|z^*_i| + \sum_{i \in \supp (m^*)} |(1-\lambda)m^*_i+\lambda z^*_i|-|m^*_i| \\
                            &= \sum_{i\in \overline{\supp{m^*}}} \lambda s_i z_{i} + \sum_{i\in \supp (m^*)} \lambda s_i (z_{i}^{*} - m_{i}^{*}) =  \lambda s^{\top}(z^*-m^*) = 0.
    \end{split}
  \end{equation}
  That is, $\|w\|_1 = \|m^*\|_1$, which implies $w \in \mathcal{U}(A,y)$ with $Q_p(w) < Q_p(m^*)$. However, $m^*$ is defined as the minimizer of \cref{eq:m*}. This is a contradiction, which proves the claim.

  \begin{claim}\label{c:ms_qs_claim}
    We claim
    $  0 \leq s\T (q^*-m^*) \le \norm{q^*}_1-\norm{m^*}_1$.
  \end{claim}
  We start with the leftmost inequality, and assume for contradiction that $s\T m^* > s\T q^*$. Choose $\lambda \in (0,1)$ such that $|m_{i}^{*}| \geq \lambda |q_{i}^{*} - m_{i}^{*}| $ for all $i \in \supp(m^*)$ and let $w = (1-\lambda)m^* + \lambda q^*$. By linearity we have $Aw = y$. Moreover, by using the same arguments as in \cref{eq:signsplit} and \cref{eq:wmm*}, we see
  \[\|w\|_1 - \|m^*\|_1 = \lambda s^{\top}(q^* - m^*).\]
  By assumption we have $s^{\top}(q^* - m^*) < 0$, which implies $\norm{w}_1 < \norm{m^*}_1$. This contradicts the fact that $m^* \in \mathcal{U}(A,y)$, so we conclude $s\T m^* \leq s\T q^*$. To get the rightmost inequality in the claim, we bound
  \[s\T (q^*-m^*) = s\T q^* - \norm{m^*}_1 \le \norm{s}_\infty\norm{q^*}_1 - \norm{m^*}_1 = \norm{q^*}_1 - \norm{m^*}_1.\]
  This proves the claim. 

Next, let 
  \[\tilde{A} = \begin{bmatrix}A\\s\T\end{bmatrix} \quad \text{ and }
  \quad 
  \tilde{y} = \begin{bmatrix} y \\ R_{m^*} \end{bmatrix},\] 
  and observe
  \begin{alignat}{2}
      & \argmin_{z \in \R^N} Q_p(z) &&\text{ subject to } z \in \mathcal{S}_{R_{m^*}}\nonumber\\
    = & \argmin_{z \in \R^N} Q_p(z) &&\text{ subject to } \tilde{A}z = \tilde{y} \text{ and } z_i s_i \ge 0 \text{ for } i \in \overline{\supp(m^*)}\label{m_kkt}.
  \end{alignat}
  From \Cref{c:argmin_claims}, we know $m^*$ is a minimizer of \eqref{m_kkt}. It follows that $m^*$ must satisfy the KKT optimality conditions in \Cref{KKT}. This means there exist $\lambda \in \R^m$ and $\mu \in [0,\infty)^N$, with $\supp(\mu) \subset \overline{\supp(m^*)}$, such that
  \begin{alignat}{2}
    \nabla Q_p(m^*) &= \tilde{A}\T\lambda + \diag(s)\mu,&&\label{raw_kkt}\\
    \tilde{A}m^* &= \tilde y, && \nonumber \\
    \mu_i m^*_i &= 0 &&\quad \text{for } i \in \overline{\supp{m^*}}\nonumber,\\
    s_i m^*_i &\ge 0 &&\quad \text{for } i \in \overline{\supp{m^*}}\nonumber.
  \end{alignat}
  In particular, this implies $\mu_i m^*_i = 0$ for all $i \in \{1,\ldots,N\}$.

  Now, from \cref{raw_kkt} we have $\nabla Q_p(m^*)-\diag(s)\mu = \tilde A^{\top}\lambda \in \Null (A)^{\perp}$. Let $\xi \coloneqq \nabla Q_p(m^*)-\diag(s)\mu$ for notational convenience, and note
  \[ \xi_i =  q_p'\left(\frac{m^*_i}{\alpha^p}\right)-s_i\mu_i, \quad\text{for }i \in \{1,\ldots,N\}.\]

  By using the same arguments as above for $q^*$ in \cref{m_W_and_q_W_claim}, with $\hat{y}^{\top} = \begin{bmatrix} y^{\top} & R_{q^*} \end{bmatrix}$ instead of $\tilde{y}$, we know there exist $\nu \in \R^m$ and $\gamma \in [0,\infty)^N$, with $\supp (\gamma) \subset \overline{\supp (q^*)}$, such that $\gamma_iq^{*}_{i} = 0$ for $i\in \{1,\ldots, N\}$ and $s_iq_{i}^* \geq 0$ for $i \in \overline{\supp(m^*)}$. Furthermore, we let $\eta \coloneqq \nabla Q_p(q^*) - \diag (s)\gamma = \tilde{A}^{\top} \nu \in \Null(\tilde{A})^{\perp}$, and note
  \[ \eta_i =  q_p'\left(\frac{q^*_i}{\alpha^p}\right)-s_i\gamma_i, \quad\text{for }i \in \{1,\ldots,N\}.\]

  \begin{claim}\label{c:prop_D_claim}
    Let $D$ be a $N\times N$ diagonal matrix with diagonal elements 
    \begin{equation}\label{eq:D_def}
      D_{ii} = \begin{cases}
        \frac{\xi_i-\eta_i}{m^*_i-q^*_i} &\text{ if } m^*_i \ne q^*_i\\
        1 &\text{ otherwise}
      \end{cases},
      \quad\text{for } i \in \{1,\ldots,N\}.
    \end{equation}
    We claim $D_{ii} > 0$ and $D(m^*-q^*) = \xi -\eta$. 
  \end{claim}

  It is straightforward to see $D(m^*-q^*) = \xi -\eta$ from \cref{eq:D_def}, so we concentrate on proving $D_{ii} > 0$ for $i=1,\ldots,N$. Consider $i \in \{1,\ldots, N\}$. The claim is trivially true if $m^*_i = q^*_i$. Therefore, assume $m^*_i \ne q^*_i$. Notice that $q_p'$ is strictly increasing, since $q_p$ is strictly convex. It follows that if $m_{i}^{*}$ and $q^{*}_{i}$ are non-zero, then $\mu_i = \gamma_i = 0$, and 
  \[
    D_{ii} 
    = \frac{q'_p\left(\frac{m^*_i}{\alpha^p}\right)-q'_p\left(\frac{q^*_i}{\alpha^p}\right)}{m^*_i-q^*_i} > 0, 
  \]
  since $q_p'$ is strictly increasing.

  Next, observe $q'_p(0) = 0$. Indeed, since $q_p$ is a strictly convex and even function, we know $q'_p$ must be odd, which implies $q'_p(0) = 0$. Assume $m^*_i$ is non-zero and $q^*_i = 0$. Then $\mu_i = 0$, and 
  \[
    D_{ii} 
    = \frac{q'_p\left(\frac{m^*_i}{\alpha^p}\right)}{m^*_i}+ s_i \frac{\gamma_i}{m^*_i} = \frac{q'_p\left(\frac{m^*_i}{\alpha^p}\right)}{m^*_i}+\frac{\gamma_i}{|m^*_i|} > 0,
  \]
  since $q_p'$ is strictly increasing and odd, and $\gamma_i \geq 0$. A similar argument proves the case where $q^*_{i}$ is non-zero and $m^{*}_i = 0$. From this, we conclude that \Cref{c:prop_D_claim} holds. 

  \begin{claim}\label{c:elem_in_null}
    Let $\hat s = P_{\Null (A)}s$, and let $d \in \Null(A)$ be given by
    \begin{align*}
      d = \begin{cases}
        \frac{s\T (q^*-m^*)}{\norm{\hat {s}}_2^2} \hat s & \text{ if } \hat s \ne 0\\
        0 & \text{ if } \hat s = 0.
      \end{cases}.
    \end{align*}
    We claim $q^* - m^* - d \in  \Null (\tilde A)$.
  \end{claim}

First, notice that $d \in \Null(A)$ by construction, and $q^*-m^* \in \Null(A)$ since $Am^*=Aq^*=y$. Hence, it is sufficient to prove $s^{\top}(q^* - m^* - d) = 0$. If $\hat s \ne 0$, then $s\T (q^*-m^*-d) = s\T (q^*-m^*)(1-\tfrac{s^{\top}\hat s}{\norm{\hat s}_2^2}) = 0$. Otherwise, if $\hat s = 0$, then $s \in \Null(A)^{\perp}$, which implies $s\T (q^*-m^*-d) = 0$ since $q^*-m^*-d \in \Null(A)$. This implies $\tilde{A}(q^*-m^*-d) = 0$, which proves the claim.

  Now, since $\xi,\eta \in \Null(\tilde{A})^{\perp}$, we may use \Cref{c:prop_D_claim} and \Cref{c:elem_in_null} to deduce
  \[
    0
    = P_{\Null (\tilde A)} (\eta - \xi) 
    = P_{\Null (\tilde A)}D(q^* - m^*) 
    = P_{\Null (\tilde A)}D(d+P_{\Null (\tilde A)}(q^* - m^* - d)).
  \]
  It follows from \Cref{K_lemma} with $u=d$ and $v=q^* - m^* - d$, that 
  \begin{equation}\label{eq:norm_bound444}
    \norm{m^*-q^*}_2 \le \mathcal{K}(\tilde{A})\norm{d}_2.
  \end{equation}
  If $d \ne 0$, we know from \Cref{c:ms_qs_claim} that
  \begin{equation}\label{eq:d_bound55}
    \norm{d}_2 = \frac{s\T (q^*-m^*)}{\norm{\hat s}_2} \le \frac{\norm{q^*}_1-\norm{m^*}_1}{\norm{\hat s}_2}.
  \end{equation}
  Combining \cref{eq:norm_bound444} and \cref{eq:d_bound55} gives
  $\norm{m^*-q^*}_2 \le  C_A(\norm{q^*}_1-\norm{m^*}_1)$
  where $C_A$ is given by \cref{eq:def_C_A}.
\end{proof}

We need the following lemma for proving \Cref{qm_bound} with $p = 2$.
\begin{lemma}\label[lemma]{entropy_diff_bound}
  Let $u,v \in [0,e^{-1}]^N$, then
  \begin{align}
    \sum_{i=1}^N |u_i\ln(u_i)-v_i\ln(v_i)| \le \sqrt N\norm{u-v}_2\ln\left(\frac{N}{\norm{u-v}_2}\right)
  \end{align}
  where we use the convention that $0 \ln 0 = 0$.
\end{lemma}
\begin{proof}
  We start by considering the problem in one variable. Let $a,b \in [0,e^{-1}]$ and let $f(x) = -x\ln(x)$ for $x\geq 0$. Assume that $a\geq b$. Since $f$ is subadditive on $[0,e^{-1}]$, we have that $f(b) - f(a)\leq f(b-a)$. By symmetry this yields that  $|f(b)-f(a)|\leq f(|b-a|)$.

    Using this fact, together with Jensen's inequality with uniform weights (for concave functions), we see that
\begin{align*}
    \sum_{i=1}^N |u_i\ln(u_i)-v_i\ln(v_i)|  &\le \sum_{i=1}^N f(|u_i-v_i|) 
    \le N f\left(\frac{\norm{u-v}_1}{N}\right)\\
    = \norm{u-v}_1\ln\left(\frac{N}{\norm{u-v}_1}\right) &\le \sqrt N\norm{u-v}_2\ln\left(\frac{N}{\norm{u-v}_2}\right).
  \end{align*}
\end{proof}

We are now ready to bound $\norm{q^*-m^*}_2$ directly.
\begin{proposition}\label[proposition]{qm_bound}
  Let $A \in \R^{m\times N}$ with rank$(A) = m, y \in \R^m$, $\alpha > 0$ and $p\geq 2$. 
  Let $q^*$ and $m^*$ be given by \eqref{eq:q*} and \eqref{eq:m*}.
  Then, if $p > 2$ there exists a constant $C>0$ only depending on $A$ and $p$ such that
  \begin{align*}
    \norm{q^*-m^*}_2 &\le C\alpha^p.
  \end{align*}
  Furthermore, if $p = 2$ there exist constants $C_1,C_2 > 0$ only depending on $A$ such that
  \begin{align*}
    \norm{q^*-m^*}_2 &\le C_1\norm{y}_2\left(\frac{\alpha^2}{\norm{y}_2}\right)^{C_2}.
  \end{align*}
\end{proposition}
\begin{proof}
    If $\Null(A) = \{0\}$, we know from \cref{l2_l1_bound} that $m^*=q^*$, and the proposition holds immediately. Therefore, assume throughout that the nullspace of $A$ is non-trivial.

    We have $Q_p(q^*) \le Q_p(m^*)$, since $q^*$ since the feasible set in \cref{eq:q*} is a superset of the feasible set in \cref{eq:m*}. From \cref{prop:q_p_properties} we know that $q_p$ is an even function. By using this fact, together with the Fundamental Theorem of Calculus and \cref{qg'_bound}, we get

{ \allowdisplaybreaks
    \begin{align}
      0 &\le Q_p(m^*)-Q_p(q^*) = \alpha^p \sum_{i=1}^N \left( q_p\left(\frac{|m^*_i|}{\alpha^p}\right)-q_p\left(\frac{|q^*_i|}{\alpha^p}\right) \right)\nonumber\\
      &= \alpha^p \sum_{\substack{i=1 \\ |m^*_i| \ge |q^*_i|}}^N \int_{\frac{|q^*_i|}{\alpha}}^{\frac{|m^*_i|}{\alpha^p}} q'_p(x)\ dx
      - \alpha^p \sum_{\substack{i=1 \\ |m^*_i| < |q^*_i|}}^N \int_{\frac{|m^*_i|}{\alpha}}^{\frac{|q^*_i|}{\alpha^p}} q'_p(x)\ dx \nonumber\\
    \le\ &\alpha^p \sum_{\substack{i=1 \\ |m^*_i| \ge |q^*_i|}}^N \int_{\frac{|q^*_i|}{\alpha^p}}^{\frac{|m^*_i|}{\alpha^p}} g'_p(x+1)\ dx
    - \alpha^p \sum_{\substack{i=1 \\ |m^*_i| < |q^*_i|}}^N \int_{\frac{|m^*_i|}{\alpha^p}}^{\frac{|q^*_i|}{\alpha^p}} g'_p(x)\ dx \nonumber\\
    =\ &\alpha^p \sum_{\substack{i=1 \\ |m^*_i| \ge |q^*_i|}}^N g_p\left(\frac{|m^*_i|}{\alpha^p}+1\right)-g_p\left(\frac{|q^*_i|}{\alpha^p}+1\right)
    + \alpha^p \sum_{\substack{i=1 \\ |m^*_i| < |q^*_i|}}^N g_p\left(\frac{|m^*_i|}{\alpha^p}\right)-g_p\left(\frac{|q^*_i|}{\alpha^p}\right) \label{qg'_sums}
  \end{align}
}
  Next, we consider the case $p > 2$. Recall that $\|x\|_s \leq N^{1/s - 1/q}\|x\|_{q}$ for all $x\in \R^{N}$, whenever $0 < s < q\leq \infty$, see e.g., \cite[Eq.\ A.3]{Foucart13}. Furthermore, consider $a,b \geq 0$ and observe that since $t\mapsto |t|^{2/p}$ is subadditive, we have that 
  \begin{align*}
    |g_p(a)-g_p(b) - (a-b)| = \frac{p}{2}|a^{2/p}-b^{2/p}| \le \frac{p}{2}|a-b|^{2/p} \le p|a-b|^{2/p}.
  \end{align*}
  Using these inequalities, we get
  \begin{align*}
    &\alpha^p \sum_{\substack{i=1 \\ |m^*_i| \ge |q^*_i|}}^N g_p\left(\frac{|m^*_i|}{\alpha^p}+1\right)-g_p\left(\frac{|q^*_i|}{\alpha^p}+1\right)
    + \alpha^p \sum_{\substack{i=1 \\ |m^*_i| < |q^*_i|}}^N g_p\left(\frac{|m^*_i|}{\alpha^p}\right)-g_p\left(\frac{|q^*_i|}{\alpha^p}\right)\\
    \le\ & \norm{m^*}_1-\norm{q^*}_1 + p\alpha^{p-2}\sum_{i=1}^N |m^*_i-q^*_i|^{2/p}
    \le \norm{m^*}_1-\norm{q^*}_1 + Np\alpha^{p-2}\norm{m^*-q^*}_2^{2/p}.
  \end{align*}
  It follows that $\norm{m^*}_1-\norm{q^*}_1 + Np\alpha^{p-2}\norm{m^*-q^*}_2^{2/p} \ge 0$. Combining this with \Cref{l2_l1_bound}, yields
  \begin{align*}
    \norm{m^*-q^*}_2 \le C_A(\norm{q^*}_1-\norm{m^*}_1) \le C_ANp\alpha^{p-2}\norm{m^*-q^*}_2^{2/p},
  \end{align*}
  and solving for $\norm{m^*-q^*}_2$, gives
  \begin{align*}
    \norm{m^*-q^*}_2 \le (C_A Np)^{\frac{p}{p-2}}\alpha^p.
  \end{align*}

  Next, we consider the case $p=2$. Let $M = \frac{2e\sqrt N}{\sigma_{\min}(A)}\norm{y}_2$, then we know from \Cref{q_infty_bound} that $\norm{m^*}_\infty, \norm{q^*}_\infty \le \frac{M}{2e}$. Assume that $2e\alpha^2 \le M$, and let $u,v \in \R^N$ be given by 
  \begin{align*}
    u_i &= \begin{cases}
      \frac{|m^*_i|}{M}+\frac{\alpha^2}{M} & \text{ if } |m^*_i| \ge |q^*_i|\\
      \frac{|m^*_i|}{M} & \text{ if } |m^*_i| < |q^*_i|
    \end{cases},\quad\text{ and }\quad
    v_i = \begin{cases}
      \frac{|q^*_i|}{M}+\frac{\alpha^2}{M} & \text{ if } |m^*_i| \ge |q^*_i|\\
      \frac{|q^*_i|}{M} & \text{ if } |m^*_i| < |q^*_i|
    \end{cases}
  \end{align*}
    for $i \in \{1,\ldots,N\}$. 
    We insert $u$ and $v$ into \eqref{qg'_sums} and simplify:
  \begin{align}
    &\alpha^2 \sum_{\substack{i=1 \\ |m^*_i| \ge |q^*_i|}}^N g_2\left(\frac{|m^*_i|}{\alpha^2}+1\right)-g_2\left(\frac{|q^*_i|}{\alpha^2}+1\right)
    + \alpha^2 \sum_{\substack{i=1 \\ |m^*_i| < |q^*_i|}}^N g_2\left(\frac{|m^*_i|}{\alpha^2}\right)-g_2\left(\frac{|q^*_i|}{\alpha^2}\right)\nonumber\\
    =\ &\alpha^2 \sum_{i=1}^N g_2\left(\frac{Mu_i}{\alpha^2}\right)-g_2\left(\frac{Mv_i}{\alpha^2}\right)
    = M \sum_{i=1}^N u_i \ln\left(\frac{Mu_i}{e\alpha^2}\right)-v_i\ln\left(\frac{Mv_i}{\alpha^2}\right)\nonumber\\
    =\ &M \ln\left(\frac{M}{e\alpha^2}\right)(\norm{u}_1-\norm{v}_1) + M\sum_{i=1}^N u_i \ln(u_i)-v_i\ln(v_i). \label{eq:tmp4432}
  \end{align}
    Next, note that $(u-v)_i = \frac{|m^*_i|-|q^*_i|}{M} \le \frac{|m^*_i-q^*_i|}{M}$ for all $i\in \{1,\ldots,N\}$. This implies that  $\norm{u-v}_2 \le \frac{\norm{m^*-q^*}_2}{M}$, and that $\norm{u}_1-\norm{v}_1 = \frac{\norm{m^*}_1-\norm{q^*}_1}{M}$. Furthermore, by assumption we have that $u, v \in [0,e^{-1}]^N$. Thus, applying \Cref{entropy_diff_bound} and the above inequalities to \cref{eq:tmp4432}, now yields
  \begin{align*}
    &M \ln\left(\frac{M}{e\alpha^2}\right)(\norm{u}_1-\norm{v}_1) + M\sum_{i=1}^N u_i \ln(u_i)-v_i\ln(v_i)\\
    \le\ &M \ln\left(\frac{M}{e\alpha^2}\right)(\norm{u}_1-\norm{v}_1) + \sqrt NM\norm{u-v}_2\ln\left(\frac{N}{\norm{u-v}_2}\right) \\
    \le\ &\ln\left(\frac{M}{e\alpha^2}\right)(\norm{m^*_i}_1-\norm{q^*}_1) + \sqrt N\norm{m^*-q^*}_2\ln\left(\frac{NM}{\norm{m^*-q^*}_2}\right).
  \end{align*}
  We conclude that $\ln\left(\frac{M}{e\alpha^2}\right)(\norm{m^*_i}_1-\norm{q^*}_1) + \sqrt N\norm{m^*-q^*}_2\ln\left(\frac{NM}{\norm{m^*-q^*}_2}\right) \ge 0$ and combine this with \Cref{l2_l1_bound} to get
  \begin{align*}
      \norm{m^*-q^*}_2 \le C_A(\norm{q^*}_1-\norm{m^*}_1) \le \frac{C_A\sqrt N}{\ln\left(\frac{M}{e\alpha^2}\right)}\norm{m^*-q^*}_2\ln\left(\frac{NM}{\norm{m^*-q^*}_2}\right),
  \end{align*}
where $C_A > 0$ is the constant from \Cref{l2_l1_bound}.
Solving for $\norm{m^*-q^*}_2$ we get
  \begin{align*}
    \norm{m^*-q^*}_2 \le NM \left(\frac{e\alpha^2}{M}\right)^{\frac{1}{C_A \sqrt N}}.
  \end{align*}
  Having established a bound when $2e\alpha^2 \le M$, we now consider the case when $2e\alpha^2 > M$. Recall $\norm{m^*}_\infty, \norm{q^*}_\infty \le \frac{M}{2e}$ by \Cref{q_infty_bound}. Therefore, 
  \begin{align*}
    \norm{m^*-q^*}_2 \le \norm{m^*}_2+\norm{q^*}_2 \le NM.
  \end{align*}
  In both cases $2e\alpha^2 \le M$ and $2e\alpha^2 > M$ we see that
  \begin{align*}
    \norm{m^*-q^*}_2 \le NM \left(\frac{2e\alpha^2}{M}\right)^{\frac{1}{C_A \sqrt N}} = \left(2e\sqrt N\left(\frac{\sqrt N}{\sigma_{\min}(A)}\right)^{1-\frac{1}{C_A \sqrt N}}\right) \norm{y}_2 \left(\frac{\alpha^2}{\norm{y}_2}\right)^{\frac{1}{C_A \sqrt N}}.
  \end{align*}
\end{proof}

\subsubsection{Proof of \texorpdfstring{\Cref{thm:grad_flow}}{} part \ref{it:dep_bound}}

Now that we have bounded $\norm{q^*-m^*}_2$ and $\norm{m^*-g^*}_2$, we simply gather the lemmas and apply the triangle inequality. We separate into two cases $p = 2$ and $p > 2$.
\begin{proof}[Proof when $p > 2$]
  By \Cref{psi_infty_eq_q}, we have $\psi_{\alpha}(\infty) = q^*$. Furthermore, by \Cref{Wp_Gp_id} we have $\mathcal{W}_p(A,y) = g^*$. Next, we use \Cref{qm_bound} with $p > 2$ and \Cref{gm_bound} to find constants $C_1,C_2$, depending only on $A$ and $p$, such that $\norm{q^*-m^*}_2 \le C_1\alpha^p$ and $\norm{m^*-g^*}_2 \le C_2\alpha^p$. Then, by the triangle inequality
  \begin{align*}
    \norm{\psi_{\alpha}(\infty)-\mathcal{W}_p(A,y)}_2 = \norm{q^*-g^*}_2 \le \norm{q^*-m^*}_2+\norm{m^*-g^*}_2 \le (C_1+C_2) \alpha^p.
  \end{align*}
\end{proof}

\begin{proof}[Proof when $p = 2$]
  By \Cref{psi_infty_eq_q}, we have $\psi_{\alpha}(\infty) = q^*$. Furthermore, by \Cref{Wp_Gp_id} we have $\mathcal{W}_2(A,y) = g^*$. Next, we use \Cref{qm_bound} with $p = 2$ and \Cref{gm_bound} to find constants $C_1, C_2$ and $C_3$ depending only on $A$ such that $\norm{q^*-m^*}_2 \le C_1\norm{y}_2\left(\frac{\alpha^2}{\norm{y}_2}\right)^{C_2}$ and $\norm{m^*-g^*}_2 \le C_3\alpha^2$. 

  Next, let $C_2' =\min\{1, C_2\}$, and consider the case where $\alpha^2 \le \norm{y}_2$. Then 
  \[\max\left\{\frac{\alpha^2}{\norm{y}_2}, \left(\frac{\alpha^2}{\norm{y}_2}\right)^{C_2}\right\} \le \left(\frac{\alpha^2}{\norm{y}_2}\right)^{C_2'}\]
  so that
  \begin{align*}
      \norm{\psi_{\alpha}(\infty)-\mathcal{W}_2(A,y)}_2 &= \norm{q^*-g^*}_2 \le \norm{q^*-m^*}_2+\norm{m^*-g^*}_2\\
      &\le C_1\norm{y}_2\left(\frac{\alpha^2}{\norm{y}_2}\right)^{C_2}+C_3\alpha^2 \le (C_1+C_3)\norm{y}_2\left(\frac{\alpha^2}{\norm{y}_2}\right)^{C_2'}
  \end{align*}
    as desired.

    Next, consider the case where $\alpha^2 > \norm{y}_2$. By \Cref{q_infty_bound} we have $\norm{q^*}_\infty \le \norm{m^*}_1 \le \frac{\sqrt N}{\sigma_{\min}(A)}\norm{y}_2$. Since $m^*,g^* \in \mathcal{U}(A,y)$, we also have $\norm{g^*}_1 = \norm{m^*}_1$. Thus we may bound
  \begin{align*}
    \norm{\psi_{\alpha}(\infty)-\mathcal{W}_2(A,y)}_2 &= \norm{q^*-g^*}_2 \le \norm{q^*}_2+\norm{g^*}_2 \le \sqrt N \norm{q^*}_\infty+\norm{g^*}_1\\
    \le (\sqrt{N}+1)\norm{m^*}_1 &\le \frac{N+\sqrt N}{\sigma_{\min}(A)}\norm{y}_2 < \frac{N+\sqrt N}{\sigma_{\min}(A)}\norm{y}_2\left(\frac{\alpha^2}{\norm{y}_2}\right)^{C_2'}.
  \end{align*}
  Thus, both when $\alpha^2 \le \norm{y}_2$ and when $\alpha^2 > \norm{y}_2$, we have
  \begin{align*}
    \norm{\psi_{\alpha}(\infty)-\mathcal{W}_2(A,y)}_2 &\le \left(C_1+C_3+\frac{N+\sqrt N}{\sigma_{\min}(A)}\right)\norm{y}_2\left(\frac{\alpha^2}{\norm{y}_2}\right)^{C_2'}.
  \end{align*}
\end{proof}

\section{Discretization of the gradient flow -- Proof of \Cref{prop:main1} part \ref{it:m2}}\label{s:disc_steps}
\subsection{Preliminaries}
Up until now we have been working with gradient flow. When implemented on a computer however, the flow \eqref{eq:grad_flow} is discretized. This is typically done using (variants of) gradient descent. We pick a step length $\eta$ and define the sequence
\begin{equation}\label{eq:grad_desc}
  \begin{aligned}
    \theta_0 &= \alpha {\bf 1}_{2N}\\
    \theta_{k+1} &= \theta_k - \eta \nabla L(\theta_k), \quad k \in \{0,1,\dots\}.
  \end{aligned}
\end{equation}
In contrast to the gradient flow path, the discretized path might encounter non-positive elements in $\theta$. To resolve this, we extend $L$ to handle negative values
\begin{align}\label{eq:def_L_disc}
  L(\theta) = \frac{1}{2}\norm{A(|\theta_+|^p-|\theta_-|^p)-y}_2^2,
\end{align}
where $\theta = (\theta_+, \theta_-)$, and $|\cdot|^p$ is elementwise absolute values to elementwise powers. Note that since $p \ge 2$, $L$ is now twice continuously differentiable. Our new $L$ coincides with the old one which was defined only for non-positive $\theta$, so it may be taken as the original definition of $L$. We may also view the absolute values as a trick to handle negative values in the implementation of gradient descent for non-integer $p$.

The proposition below tells us that for fixed $A$,$y$,$p$ and $\alpha$, we may approximate the gradient flow to arbitrary accuracy using gradient descent. We just need to pick small enough step length $\eta$.
\begin{proposition}\label{grad_desc_bound}
  Let $\alpha > 0, p \ge 2, \epsilon > 0$, and $t \in [0,\infty)$. Define $M = \frac{2\sqrt N}{\sigma_{\min}(A)}\norm{y}_2+\alpha^p$, $C_1 = 40 p \sqrt N \norm{A}_{\mathrm{op}}^2 M^\frac{2p-1}{p}$ and $C_2 = 50 p^2 \sqrt N \norm{A}_{\mathrm{op}}^2 M^\frac{2p-2}{p}$. Then gradient descent \eqref{eq:grad_desc} with step length $\eta \le \min\{\epsilon, \frac{\alpha}{p}\}\frac{1}{C_1}e^{-C_2 t}$ satisfies
  \begin{align*}
    \|\theta_{\lfloor t / \eta \rfloor} &- \theta(t)\|_2 \le \epsilon.
  \end{align*}
\end{proposition}

To prove \Cref{grad_desc_bound} using \Cref{euler_method}, we will need to bound $\nabla L$ and $\nabla^2 L$ close to the gradient flow path. We will first bound $\theta(t)$, which will be used to bound $\nabla L$ and $\nabla^2 L$.

\begin{lemma}\label{theta_bound}
  Let $\theta$ and $\psi_{\alpha}$ be as defined in \eqref{eq:param}, evolving by \eqref{eq:grad_flow}, for some $p \ge 2$ and $\alpha > 0$. Then $\norm{\theta(t)}_\infty^p \le \norm{\psi_{\alpha}(t)}_\infty + \alpha^p \le \frac{2\sqrt N}{\sigma_{\min}(A)}\norm{y}_2 + \alpha^p$ for all $t\geq 0$.
\end{lemma}
\begin{proof}
  Let $t \geq 0$ and $i \in \{1,\ldots,N\}$. In the proof of \Cref{psi_integral} we derived explicit expressions \eqref{eq:p2_theta} and \eqref{eq:pg2_theta} for $\theta_+(t)$ and $\theta_-(t)$. We observe the following elementwise invariants for $i \in \{1,\dots,N\}$,
  \begin{alignat*}{2}
    [\theta_+(t)]_i\cdot [\theta_-(t)]_i &= \alpha^2 &&\text{ if } p = 2\\
    [\theta_+(t)]_i^{2-p} + [\theta_-(t)]_i^{2-p} &= 2\alpha^{2-p} \quad &&\text{ if } p > 2.
  \end{alignat*}
  Consider the case where $p = 2$, and assume (for a contradiction) that $[\theta_+(t)]_i^2 > \norm{\psi_{\alpha}(t)}_\infty + \alpha^2$ for some $i \in \{1,\dots,N\}$. Then 
  \begin{align*}
    [\psi_{\alpha}(t)]_i = [\theta_+(t)]_i^2 - [\theta_-(t)]_i^2 = [\theta_+(t)]_i^2 - \frac{\alpha^4}{[\theta_+(t)]_i^2} > \norm{\psi_{\alpha}(t)}_\infty + \alpha^2 - \frac{\alpha^4}{\alpha^2} = \norm{\psi_{\alpha}(t)}_\infty,
  \end{align*}
    which is a contradiction. A similar argument can be used to show that $[\theta_-(t)]_i^2 \leq \norm{\psi_{\alpha}(t)}_\infty + \alpha^2$.

    Next, consider the case where $p > 2$, and assume $[\theta_+(t)]_i^p > \norm{\psi_{\alpha}(t)}_\infty + \alpha^p$. Then
  \begin{align*}
    \norm{\psi_{\alpha}(t)}_\infty &\ge [\psi_{\alpha}(t)]_i = [\theta_+(t)]_i^p - [\theta_-(t)]_i^p = [\theta_+(t)]_i^p - \left(2\alpha^{2-p}-[\theta_+(t)]_i^{2-p}\right)^{-\frac{p}{p-2}}\\
                                   &> \norm{\psi_{\alpha}(t)}_\infty + \alpha^p - \left(2\alpha^{2-p}-\alpha^{2-p}\right)^{-\frac{p}{p-2}} = \norm{\psi_{\alpha}(t)}_\infty,
  \end{align*}
  which is again a contradiction. Again, a similar argument can be used to show that $[\theta_-(t)]_i^2 \leq \norm{\psi_{\alpha}(t)}_\infty + \alpha^2$. Finally, since our choices of $t$ and $i$ were arbitrary, this holds for all $t\geq0$ and $i \in \{1,\ldots,N\}$.
 Using \Cref{psi_bound} we get the final inequality $\norm{\psi_{\alpha}(t)}_\infty + \alpha^p \le \frac{2\sqrt N}{\sigma_{\min}(A)}\norm{y}_2 + \alpha^p$.
\end{proof}

We can now bound $\nabla L$ and $\nabla^2 L$ to apply \Cref{euler_method}.
\begin{lemma}\label{L_bounds}
  Let $A \in \R^{m\times N}$ with rank$(A) = m$ and $y \in \R^m$. Let $\psi_{\alpha}(t)$ be the state of the gradient flow problem \eqref{eq:grad_flow} for some $p \ge 2, \alpha > 0$ at time $t$. Suppose $\hat \theta \in \R^{2N}$ satisfies $\norm{\hat\theta-\theta(t)}_\infty \le \frac{\alpha}{p}$ for some $t \in [0,\infty)$. Then
    \begin{align*}
      \norm{\nabla L(\hat\theta)}_2 &\le 40p\sqrt N\norm{A}_\mathrm{op}^2 M^\frac{2p-1}{p},\quad\text{and}
      &\norm{\nabla^2 L(\hat\theta)}_\mathrm{op} \le 50p^2\sqrt N\norm{A}_\mathrm{op}^2 M^\frac{2p-2}{p},
    \end{align*}
    where $M = \frac{2\sqrt N}{\sigma_{\min}(A)}\norm{y}_2+\alpha^p$.
  \end{lemma}
  \begin{proof}
    Recall that $\hat\theta = (\hat\theta_+, \hat\theta_-)$, 
    \[ L(\hat\theta) = \frac{1}{2}\norm{A(|\hat\theta_+|^p-|\hat\theta_-|^p)-y}_2^2,
    \quad \text{and let }\quad r(\hat\theta) = A\T\left(A(|\hat\theta_+|^p-|\hat\theta_-|^p)-y\right).
  \]
  Then
    \begin{align*}
      \nabla L(\hat\theta) &= p
      \begin{bmatrix*}[r]
        \diag\left(\sign(\hat\theta_+) \odot |\hat\theta_+|^{p-1}\right)\\
        -\diag\left(\sign(\hat\theta_-) \odot |\hat\theta_-|^{p-1}\right)
      \end{bmatrix*}
      r(\hat\theta), 
      \\
      \nabla^2 L(\hat\theta) &= p^2
      \begin{bmatrix*}[r]
        \diag\left(\sign(\hat\theta_+) \odot |\hat\theta_+|^{p-1}\right)\\
        -\diag\left(\sign(\hat\theta_-) \odot |\hat\theta_-|^{p-1}\right)
      \end{bmatrix*} 
      A^\top A
      \begin{bmatrix*}[r]
        \diag\left(\sign(\hat\theta_+) \odot |\hat\theta_+|^{p-1}\right)\\
        -\diag\left(\sign(\hat\theta_-) \odot |\hat\theta_-|^{p-1}\right)
      \end{bmatrix*} \T + \\
      p(p-1)&\begin{bmatrix*}
        \diag\left(|\hat\theta_+|^{p-2} \odot r(\hat\theta)\right) & 0\\
        0 & -\diag\left(|\hat\theta_-|^{p-2} \odot r(\hat\theta)\right)
      \end{bmatrix*},
    \end{align*}
    where $\odot$ denotes the elementwise (Hadamard) product.
      
    We bound the norm
    \begin{align*}
      \norm{\nabla L(\hat\theta)}_2 &\le p
      \norm{\begin{matrix*}[r]
        \diag\left(\sign(\hat\theta_+) \odot |\hat\theta_+|^{p-1}\right)\\
        -\diag\left(\sign(\hat\theta_-) \odot |\hat\theta_-|^{p-1}\right)
      \end{matrix*}}_\mathrm{op}
      \norm{r(\hat\theta)}_2\\
                                    &\le p \left(\norm{\diag\left(|\hat\theta_+|^{p-1}\right)}_\mathrm{op} + 
                                    \norm{\diag\left(|\hat\theta_-|^{p-1}\right)}_\mathrm{op}\right)\norm{r(\hat\theta)}_2
      \le 2p\big\|\hat\theta\big\|_\infty^{p-1}\big\|r(\hat\theta)\big\|_2.
    \end{align*}
    By a similar argument
    \begin{align*}
      \norm{\nabla^2 L(\hat\theta)}_\mathrm{op} &\le 4p^2\norm{A}_\mathrm{op}^2\big\|\hat\theta\big\|_\infty^{2p-2} + p(p-1)\big\|\hat\theta\big\|_\infty^{p-2}\big\|r(\hat\theta)\big\|_\infty.
  \end{align*}
  To bound $r(\hat\theta)$ we will first bound $\big\|\hat\theta\big\|_\infty^p$. By the assumption on $\hat\theta$ we have $t \ge 0$ such that $\norm{\hat\theta-\theta(t)}_\infty \le \frac{\alpha}{p}$. Combining this with \Cref{theta_bound}, yields
  \begin{align*}
    \big\|\hat\theta\big\|_\infty^p &\le \left(\big\|\theta(t)\big\|_\infty + \big\|\hat\theta-\theta(t)\big\|_\infty\right)^p \le \left(\big\|\theta(t)\big\|_\infty + \frac{\alpha}{p}\right)^p
                                    \le \left(\left(\frac{2\sqrt N}{\sigma_{\min}(A)}\norm{y}_2+\alpha^p\right)^{\frac{1}{p}} + \frac{\alpha}{p}\right)^p.
  \end{align*}
  Next, we use $M \coloneqq \frac{2\sqrt N}{\sigma_{\min}(A)}\norm{y}_2+\alpha^p$ to simplify this expression
  \begin{align*}
    \left(\left(\frac{2\sqrt N}{\sigma_{\min}(A)}\norm{y}_2+\alpha^p\right)^{\frac{1}{p}} + \frac{\alpha}{p}\right)^p = \left(M^\frac{1}{p} + \frac{\alpha}{p}\right)^p \le \left(M^\frac{1}{p} + \frac{1}{p}M^\frac{1}{p}\right)^p \le eM.
  \end{align*}
  Which means $\big\|\hat\theta\big\|_\infty^p \le eM$. Using this, we may finally bound
  \begin{align*}
    \norm{r(\hat\theta)}_\infty &\le \norm{r(\hat\theta)}_2 = \norm{A\T\left(A(|\hat\theta_+|^p-|\hat\theta_-|^p)-y\right)}_2 \le \norm{A}_{\mathrm{op}}\norm{y}_2 + \norm{A}_{\mathrm{op}}^2\norm{|\hat\theta_+|^p-|\hat\theta_-|^p}_2\\
                                &\le \norm{A}_{\mathrm{op}}\norm{y}_2 + 2\sqrt N\norm{A}_{\mathrm{op}}^2\norm{\hat\theta}_\infty^p \le \norm{A}_{\mathrm{op}}\norm{y}_2 + 2e\sqrt N\norm{A}_{\mathrm{op}}^2M \le 7\sqrt N\norm{A}_{\mathrm{op}}^2M.
  \end{align*}
  For the last inequality we used that $\norm{y}_2 \le \norm{A}_{\mathrm{op}}M$.
  We may now use the bounds on $\hat\theta$ and $r(\hat\theta)$ to bound the derivatives
  \begin{align*}
    \norm{\nabla L(\hat\theta)}_2 &\le 2p\big\|\hat\theta\big\|_\infty^{p-1}\big\|r(\hat\theta)\big\|_2 \le 2p(eM)^\frac{p-1}{p} 7\sqrt N\norm{A}_{\mathrm{op}}^2 M\\
                                  &\le 40p\sqrt N\norm{A}_{\mathrm{op}}^2 M^\frac{2p-1}{p}\\
    \norm{\nabla^2 L(\hat\theta)}_\mathrm{op} &\le 4p^2\norm{A}_{\mathrm{op}}^2\big\|\hat\theta\big\|_\infty^{2p-2} + p(p-1)\big\|\hat\theta\big\|_\infty^{p-2}\big\|r(\hat\theta)\big\|_\infty\\
                                       &\le 4p^2\norm{A}_{\mathrm{op}}^2(eM)^\frac{2p-2}{p} + p(p-1)(eM)^\frac{p-2}{p}7\sqrt N\norm{A}_{\mathrm{op}}^2M\\
                                       &\le 50 p^2 \sqrt N\norm{A}_{\mathrm{op}}^2 M^\frac{2p-2}{p}.
  \end{align*}
\end{proof}

\begin{proof}[Proof of \Cref{grad_desc_bound}]
  The claim follows directly from applying \Cref{euler_method} with $f = -\nabla L$ and $\delta = \frac{\alpha}{p}$. Because $\norm{f}_2 = \norm{\nabla L}_2$ and $\norm{\nabla f}_{\mathrm{op}} = \norm{\nabla^2 L}_{\mathrm{op}}$, we get the required regularity conditions on $f$ from \Cref{L_bounds}.
\end{proof}

\subsection{Proof of \Cref{prop:main1} part \ref{it:m2}}
Now that we have proved \Cref{grad_desc_bound}, we can use it to prove \Cref{prop:grad_desc}, which is a precise version of \Cref{prop:main1} part \ref{it:m2}.
\begin{theorem}\label{prop:grad_desc}
  Let $\alpha > 0, p \ge 2, \epsilon > 0$, and $t \in [0,\infty)$ be given. Let
  \begin{align*}
    0 < \eta &\le \frac{\min\{\hat\epsilon,\ \alpha/p\}}{40p\sqrt N \norm{A}_\mathrm{op}^2 K^{2p-1}}e^{-50p^2\sqrt N\norm{A}_\mathrm{op}^2 K^{2p-2}t},
             && K = \left(\frac{2\sqrt N \norm{y}_2}{\sigma_{\min}(A)}+\alpha^p\right)^\frac{1}{p},\\
    J &= \lfloor t / \eta \rfloor, && 
    \hat\epsilon = \min\left\{K, \frac{\epsilon}{2^p p N K^{p-1}}\right\}.
  \end{align*}
  Furthermore, let $\hat\psi \coloneqq |\theta_{+,J}|^p-|\theta_{-,J}|^p$, where $\theta_{J} = (\theta_{+,J},\theta_{-,J})$ is the $J$th iterate of \eqref{eq:grad_desc} with step length $\eta$. Then
  \[ \norm{\hat\psi-\psi_\alpha(t)}_2 \le \epsilon. \]
\end{theorem}
\begin{proof}
  We choose $\eta$ such that $\norm{\theta_{\lfloor t / \eta \rfloor}-\theta(t)}_2 \le \hat\epsilon\ $ by \Cref{grad_desc_bound}. Then, by the choice of $\hat\epsilon$ and $J = \lfloor t / \eta \rfloor$, \Cref{theta_phi_bound} gives $\|\hat{\psi} - \psi_{\alpha}(t)\|_2 \leq \epsilon$.
\end{proof}

\begin{lemma}\label{theta_phi_bound}
  Let $\hat\theta = (\hat\theta_+, \hat\theta_-)  \in \R^{2N}$, $p \ge 2$, $\epsilon > 0$ and $t \ge 0$ satisfy $\norm{\hat\theta-\theta(t)}_\infty \le \min\left\{K, \frac{\epsilon}{2^p p N K^{p-1}}\right\}$ where $K \coloneqq \left(\frac{2\sqrt N}{\sigma_{\min}(A)}\norm{y}_2 + \alpha^p\right)^\frac{1}{p}$ and $\theta(t)$ is given by \eqref{eq:grad_flow}. Then $\|\hat\psi - \psi_{\alpha}(t)\|_2 \le \epsilon$ where $\hat\psi \coloneqq |\hat\theta_+|^p-|\hat\theta_-|^p$ and $\psi_{\alpha}(t)$ is given by \eqref{psi_def}.
\end{lemma}
\begin{proof}
  Using the definitions, the triangle inequality, and $\norm{\cdot}_2 \le N\norm{\cdot}_\infty$, we bound
  \begin{align}
    \|\hat\psi - \psi_{\alpha}(t)\|_2 &= \|(|\hat\theta_+|^p-|\hat\theta_-|^p) - (\theta_+(t)^p-\theta_-(t)^p)\|_2\nonumber\\
                                      &\le \||\hat\theta_+|^p-\theta_+(t)^p\|_2 + \|\theta_-(t)^p-|\hat\theta_-|^p\|_2 \le 2N\norm{|\hat\theta|^p-\theta(t)^p}_\infty.\label{eq:psi_diff_bound}
  \end{align}
  Consider the $i$th term $\big||\hat\theta|_i^p-\theta(t)_i^p\big|$. Using $\theta(t)_i > 0$, we have $\big||\hat\theta|_i-\theta(t)_i\big| \le \norm{\hat\theta-\theta(t)}_\infty$. Furthermore, by \Cref{theta_bound}, we have $\theta(t)_i \le K$.

Next, consider $u^\frac{p-1}{p}v^\frac{1}{p} \le \left(\frac{(p-1)u^p+v^p}{p}\right)^\frac{1}{p}$ for $u,v \in [0,\infty)$ by the power mean inequality. Rearranging, we get $u^p-v^p \le pu^{p-1}(u-v)$ and by symmetry $v^p-u^p \le pv^{p-1}(v-u)$. In combination, $|u^p-v^p| \le p\max\{u,v\}^{p-1}|u-v|$. Inserting $|\hat\theta|_i$ and $\theta(t)_i$, using the bounds from above, and finally the assumed bound on $\norm{\hat\theta-\theta(t)}_\infty$, we get that 
  \begin{align*}
    \big||\hat\theta|_i^p-\theta(t)_i^p\big| &\le p \max\{|\hat\theta|_i,\theta(t)_i\}^{p-1}\big||\hat\theta|_i-\theta(t)_i\big|\\
                                             &\le p \left(K+\big||\hat\theta|_i-\theta(t)_i\big|\right)^{p-1}\big||\hat\theta|_i-\theta(t)_i\big|\\
                                             &\le p \left(K+\norm{\hat\theta-\theta(t)}_\infty\right)^{p-1}\norm{\hat\theta-\theta(t)}_\infty\\
                                             &\le p (2K)^{p-1} \frac{\epsilon}{2^p pNK^{p-1}}
                                             \le \frac{\epsilon}{2N}.
  \end{align*}
  Combining with \eqref{eq:psi_diff_bound}, we proved the lemma
   $ \|\hat\psi - \psi_{\alpha}(t)\|_2 \le 2N\norm{|\hat\theta|^p-\theta(t)^p}_\infty \le \epsilon $.
\end{proof}

\section{Sharpness of the results -- Proof of \Cref{thm:grad_flow} part \ref{it:forall}}
In \Cref{thm:grad_flow} \ref{it:dep_bound}, we prove that the bound in \eqref{e:pe2} holds for a specific matrix $A$. However, it is clear from the proofs that the constant $C_1$ depends on $A$ and that this constant can get arbitrarily large for certain choices of $A$. A prominent question in this respect, is whether this is an artifact of our proof, or whether this is a sharp result. In this section, we shall see that it is indeed sharp. The proof requires some background in the framework behind the Solvability Complexity Index hierarchy which we recall in the next section. 

\subsection{Preliminaries -- Mathematical tools from the SCI hierarchy}

The Solvability Complexity Index (SCI) hierarchy is a mathematical framework designed to classify the intrinsic difficulty of computational problems found in mathematics. The theory is now comprehensive, and thus we mention only certain results \cite{opt_big, SCI, Hansen_JAMS, Hansen2016ComplexityII, CRAS, ben2022computing, paradox22, Ben_Artzi2022, Colbrook_2019, colbrook2019foundations, colbrook2021computing}.  In this section, we will introduce the parts of this framework that are needed to prove our main results. We start by defining what we mean by a computational problem.  
\begin{definition}[Computational problem]\label{d:ComputationalProblem}
Let $\Omega$ be some set, which we call the {domain},
and $\Lambda$ be a set of complex valued functions on $\Omega$ such that for $\iota_1, \iota_2 \in \Omega$, then $\iota_1 = \iota_2$ if and only if $f(\iota_1) = f(\iota_2)$ for all $f \in \Lambda$, called an {evaluation} set. Let $(\mathcal{M},d)$ be a metric space, and finally let $\Xi:\Omega\to \mathcal{M}$ be a function which we call the {problem} function. We call the collection $\{\Xi,\Omega,\mathcal{M},\Lambda\}$ a {computational problem}. When it is clear what $\mathcal{M}$ and $\Lambda$ are, we write $\{\Xi,\Omega\}$ for brevity. 
\end{definition}

\begin{remark}[Multivalued problems]
\label{mulivalued_remark}
Some computational problems, such as the optimization problem \eqref{eq:BP_argmin}, may have more than one solution. In these cases, we abuse notation, and set $d(x,\Xi(\iota))=\mathrm{dist}(x,\Xi(\iota))=\inf_{z\in\Xi(\iota)}d(x,z)$. It should be clear from the context when this is the case. 
\end{remark}

In the above definition, $\Omega$ consists of the set of objects that give rise to the computational problem, whereas $\Xi\colon \Omega\to\mathcal{M}$ is the problem function we are interested in computing. The set $\Lambda$ consists of functions which allow us to read information about the objects in $\Omega$. For example, $\Omega$ could consist of a collection of matrices $A$ and data $y$ in \eqref{eq:BP_argmin}, $\Lambda$ could consist of the pointwise entries of the vectors and matrices in $\Omega$, $\Xi$ could represent the solution set $\mathcal{U}$ in \eqref{eq:BP_argmin} (with the possibility of more than one solution as in \Cref{mulivalued_remark}) and $(\mathcal{M},d)$ could be $\mathbb{R}^N$ with the usual Euclidean metric (or any other suitable metric). In this paper, we restrict our attention to sets $\Lambda = \{f_j\}_{j\in\beta}$ whose cardinality is at most countable. 

Given the definition of a computational problem, we introduce the concept of a general algorithm. This concept was introduced in \cite{SCI, Hansen_JAMS}, and consists of conditions which any reasonable notion of a deterministic algorithm satisfies.   

\begin{definition}[General Algorithm]\label{definition:Algorithm}
Given a  computational problem $\{\Xi,\Omega,\mathcal{M},\Lambda\}$, a {general algorithm} is a mapping $\Gamma:\Omega\to\mathcal{M}$ such that for each $\iota\in\Omega$
\begin{enumerate}[label=(\roman*)]
\item There exists a non-empty finite subset of evaluations $\Lambda_\Gamma(\iota) \subset\Lambda$, \label{property:AlgorithmFiniteInput}
\item The action of $\,\Gamma$ on $\iota$ only depends on $\{\iota_f\}_{f \in \Lambda_\Gamma(\iota)}$ where $\iota_f \coloneqq  f(\iota),$\label{property:AlgorithmDependenceOnInput}
\item For every $\kappa\in\Omega$ such that $\kappa_f=\iota_f$ for every $f\in\Lambda_\Gamma(\iota)$, it holds that $\Lambda_\Gamma(\kappa)=\Lambda_\Gamma(\iota)$.\label{property:AlgorithmSameInputSameInputTaken}\end{enumerate}
\end{definition}

The first condition above, says that a general algorithm can only ask for a finite amount of information. However, the amount of information it reads is allowed to depend on the input, and can thus be chosen adaptively. The second condition says that the output of $\Gamma$ is only allowed to depend on the information it has read, whereas the final condition ensures that general algorithms behaves consistently, given the same information. These three conditions are chosen, as any reasonable definition of algorithm should satisfy the above three clauses. In particular, the above definition is general enough to encompass both Turing machines \cite{turing1937computable} and Blum--Shub--Smale (BSS) machines \cite{BCSS}, but also much more general models of computations.  

\begin{remark}
The generality in \Cref{definition:Algorithm}, serves two purposes. First, it provides the strongest possible impossibility bounds. That is, any statement saying that no algorithm can solve a given problem, holds in any model of computation, including the Turing and BSS models. Second, it simplifies the proofs, as general algorithms have no restrictions on the operations involved.  
\end{remark}

Now let $\{\Xi, \Omega, \mathcal{M}, \Lambda\}$ be a given computational problem. In many areas of mathematics, it is a computational task on it own to obtain the complex numbers $f(\iota)$, for $\iota \in \Omega$ and $f\in \Lambda$. This is for example the case for $\sqrt{\pi}$, $e^{5\pi i}$, or $\sin(2)$. Thus, typically we do not work with the exact numbers $f_j(\iota)$ on a computer, but rather approximations $f_{j,n}(\iota)$, where $f_{j,n}(\iota)\to f_j(\iota)$ as $n\to \infty$. This idea is formalized in the following definition.

\begin{definition}[$\Delta_{1}$-information \cite{opt_big, SCI} ]\label{definition:Lambda_limits}
	Let $\{\Xi, \Omega, \mathcal{M}, \Lambda\}$ be a computational problem. We say that $\Lambda$ has $\Delta_{1}$-information if each $f_j \in \Lambda$ is not available, however, there are mappings $f_{j,n}: \Omega \rightarrow \mathbb{Q} + i \mathbb{Q}$ such that $|f_{j,n}(\iota)-f_j(\iota)|\leq 2^{-n}$ for all $\iota\in\Omega$. Finally, if $\widehat \Lambda$ is a collection of such functions described above such that $\Lambda$ has $\Delta_1$-information, we say that $\widehat \Lambda$ provides $\Delta_1$-information for $\Lambda$. Moreover, we denote the family of all such $\widehat \Lambda$ by $\mathcal{L}^1(\Lambda)$. 
\end{definition}

We typically want to develop algorithms that work for any choice of $\Delta_1$-information. The following definition clarifies how a computational problem with this type of information is defined. 

\begin{definition}[Computational problem with $\Delta_1$-information]
Given $\{\Xi, \Omega, \mathcal{M},\Lambda\}$ with $\Lambda = \{f_j\}_{j\in \beta}$, the corresponding computational problem with $\Delta_1$-information is defined as 
	$
	\{\Xi,\Omega,\mathcal{M},\Lambda\}^{\Delta_1} \coloneqq  \{\widetilde \Xi,\widetilde \Omega,\mathcal{M},\widetilde \Lambda\},
	$ 
where 
\[
\widetilde \Omega = \left\{ \widetilde \iota = \{f_{j,n}(\iota)\}_{j, n\in \beta \times \mathbb{N}} \, : \, \iota \in \Omega, \{f_j\}_{j \in \beta} = \Lambda, |f_{j,n}(\iota)-f_j(\iota)|\leq 2^{-n}\right\},
\]
$\tilde \Xi (\tilde \iota) = \Xi(\iota)$, and $\tilde{\Lambda} = \{\tilde f_{j,n}\}_{(j,n)\in\beta \times \N}$, where $\tilde f_{j,n}(\tilde \iota) = f_{j,n}(\iota)$. Due to \Cref{d:ComputationalProblem}, we know that for each $\tilde{\iota} \in \tilde \Omega$ there is a unique $\iota \in \Omega$. We say that this $\tilde\iota$ \emph{corresponds to} $\iota$.  
\end{definition}

A few comments are in order. First, note that $\tilde \Xi$ is well-defined, since $\tilde \iota$ uniquely identifies $\iota$. Furthermore, note that $\tilde \Omega$ includes all possible instances of $\Delta_1$-information $\widehat \Lambda \in \mathcal{L}^1(\Lambda)$. 
In other words, there could be several $\tilde \iota \in \tilde \Omega$ which correspond to a given $\iota \in \Omega$. Moreover, as we shall see below, we will require an algorithm to work for any $\tilde \iota$, that is, any sequence approximating $\iota$, and not just one. First, however, we discuss how information is read by different algorithms.  We start with Turing machines, whose definition is rather lengthy and can be found in any standard text. See e.g., \cite{arora2009computational}. For readers not familiar with Turing machines, one can think of these as digital computers with no restrictions on their memory. In particular, any computer program that runs on a digital computer can be executed on a Turing machine.

\begin{remark}[General algorithms include oracle Turing machines]
\label{d:OracleTM}
    Given the definition of a Turing machine, an oracle Turing machine for $\{\Xi,\Omega,\mathcal{M}, \Lambda\}^{\Delta_1} = \{\tilde\Xi, \tilde \Omega, \mathcal{M}, \tilde \Lambda\}$, where $\Lambda = \{f_j\}_{j\in \beta\subset \N}$, is a Turing machine that has an oracle input tape that, on input $(j,n) \in \beta\times \N$ (where $j$ and $n$ may be encoded in a finite alphabet), prints the unique finite string representing the rational number $\tilde f_{j,n}(\tilde \iota)$. This is obviously a general algorithm. 
\end{remark}
As general algorithms are not tied down to a particular computational model, we do not need to specify how such algorithms read the information. However, in some cases, to simplify the proofs, we will let our general algorithms utilize Turing machines for parts of the computations. In these cases, if a general algorithm needs to execute a given Turing machine $\mathrm{TM}$ on an irrational input $c$, it is implicitly assumed that the general algorithm executes the Turing machine with an oracle tape as input, consisting of a rational sequence $\{c_k\}_{k\in \N}$, satisfying $|c-c_k|\leq 2^{-k}$ for each $k\in \N$.

A central pillar in computability theory is the concept of a quantity being \emph{computable}. Informally speaking, we would say that a real number $x$ or a function $f$ is computable if it can be approximated to arbitrary accuracy with control of the error. 
Next, we define this concept for oracle Turing machines and general algorithms.  
\begin{remark}[Computable computational problem]\label{def:comp_prob}
    Given a computational problem $\{\tilde \Xi, \tilde \Omega, \mathcal{M}, \tilde \Lambda\}$ with $\Delta_1$-information $\tilde \Lambda $, where $\mathcal{M} \subset \R^s$ for some $s\in \N$ and the metric $d$ is induced by an $\ell^r$-norm for some $r \in [1,\infty]$. We say that the problem function $\tilde \Xi \colon \Omega \to \mathcal{M}$ is:  
\begin{itemize}
    \item \emph{computable in the Turing sense}, if there exists an oracle Turing machine $\Gamma\colon \tilde \Omega \times \N \to \mathbb{Q}^s$ which upon input $(\tilde \iota, n)$ computes an approximation satisfying $d(\Gamma(\tilde \iota, n), \tilde \Xi(\tilde \iota))\leq 2^{-n}$.
    \item \emph{computable in the general sense}, if there exists a general algorithm $\Gamma\colon \tilde \Omega \times \N \to \mathbb{R}^s$ which upon input $(\tilde \iota, n)$ computes an approximation satisfying $d(\Gamma(\tilde \iota, n), \tilde \Xi(\tilde \iota))\leq 2^{-n}$.
\end{itemize}
Note that the former is obviously stronger than the latter. As we will use Theorem \ref{GHA:BP} -- which is true even for general algorithms -- we only need to consider latter case. Thus, we will with slight abuse of terminology -- for simplicity -- refer to the latter as computable.

\end{remark}
\begin{remark}[Relation to the SCI hierarchy]
     The above definition resembles the definition of being computable found in standard textbooks on computability theory \cite{ko1991computational}. In the language of the SCI hierarchy we would say that the computational problem $\{\tilde \Xi, \tilde \Omega, \mathcal{M}, \tilde \Lambda\} \in \Delta_{1}^{\mathrm{A}}$ if it is computable in a Turing sense, and that $\{\tilde \Xi, \tilde \Omega, \mathcal{M}, \tilde \Lambda\} \in \Delta_{1}^{\mathrm{G}}$ if it is computable \cite{opt_big, SCI}.
\end{remark}

\begin{remark}[Compositions of computable functions] \label{rem:comp_func}
When arguing that a given problem function $\Xi \colon \Omega \to \mathcal{M}$ is computable (with $\Delta_1$-information), it is often useful to check whether each of the operations needed to compute $f$ is itself computable, as computability is closed under function composition. Indeed, since the output of a computable function can be computed with error control, we can use the output of an algorithm approximating such a function as $\Delta_1$-information, approximating the input to another algorithm. See e.g. \cite{ko1991computational, weihrauch2000computable} for how this is done in the Turing model.
\end{remark}

The motivation for the above definition is to extend the concept of Turing computability to general algorithms. This is needed, as we will design general algorithms which can solve computational problems with $\Delta_1$-information with error control. 
\subsection{Computing basic functions}
\label{ss:Ex_non_comp}
Most functions known from calculus are computable.  
\begin{lemma}[{\cite[Sec.\ 4.3]{weihrauch2000computable}}]\label{l:comp_func}
    The following functions are computable: 
    \begin{enumerate}[label=(\roman*)]
        \item $x\mapsto c$ (where $c \in \R$ is a Turing computable constant),
        \item\label{it:arit} $+,-,\cdot \colon \R \times \R \to \R$, and $/ \colon \R\times \R\setminus \{0\} \to \R$,
        \item\label{it:max_min} $\max, \min \colon \R\times \R \to \R$,
        \item $\exp, \sin, \cos \colon \R\to\R$, $\log \colon (0,\infty)\to \R$,
        \item $\|\cdot\|_2 \colon \R^n \to \R$,
        \item\label{it:pow_pos} $(x,p)\mapsto x^p$ for any $p\in \R$ and $x > 0$ (via the identity $x^p = \exp(p\log(x))$),
        \item\label{it:abs_comp} $x\mapsto |x|$ for $x \in \R$. 
    \end{enumerate}
\end{lemma}
Next, we extend the domain of $x^p$ to non-negative values of $x$, and $p \geq 1$. We also show that when the discontinuous $\sign$-function is multiplied with $|x|^p$, the result becomes computable. 

\begin{lemma}\label{lem:pow} 
The following functions are computable.
 \begin{enumerate}[label=(\roman*)]
     \item\label{it:comp_pow} $x\mapsto x^p$ for $x\geq 0$ and $p\geq 1$.
     \item\label{it:sign_pow} $(x,p)\mapsto \mathrm{sign}(x)|x|^{p}$, for $p\geq 1$ and $x\in \R$. 
 \end{enumerate}
\end{lemma}
\begin{proof} We start the proof of \ref{it:comp_pow}.
    Let $x \in [0,\infty)$ and $p \geq 1$ be given. We denote the input to the algorithm by $(\{x_k\}_{k\in \N}, n)$, where $|x_k - x|\leq 2^{-k}$ for each $k$. Next, we design a general algorithm, which upon this input, outputs an approximation to $x^p$ with accuracy $2^{-n}$. The algorithm will, for certain inputs, apply the oracle Turing machine from \Cref{l:comp_func} \ref{it:pow_pos}, denoted by $\mathrm{TM}_{\mathrm{pow}}$. As input to this Turing machine we need a rational sequence $\{p_k\}_{k\in \N}\subset \mathbb{Q}$ of approximations to $p$, satisfying $|p_k-p|\leq 2^{-k}$ for each $k$.

The algorithm works as follows.
  \begin{algorithmic}
    \State \textbf{Input:} $(\{x_k\}_{k\in \N}, n)$
    \For{$k = 1,2,3,\ldots$}
        \State Read $x_k$.
        \If{$x_k - 2^{-k} > 0$}
            \State \Return $\mathrm{TM}_{\mathrm{pow}}(\{x_k\}, \{p_k\}, n)$. 
        \ElsIf{$x_k + 2^{-k} < 2^{-n}$}
            \State \Return $0$
        \EndIf
    \EndFor
    \end{algorithmic}
    A few comments are in order. Observe that if $x > 0$, there exists a $k\in \N$, such that $x_k - 2^{-k} > 0$. Thus, if the first if-test is true, we know that $x > 0$ and we can use the oracle Turing machine $\mathrm{TM}_{\mathrm{pow}}$ to compute an approximation to $x^p$ to accuracy $2^{-n}$. 
Furthermore, if $0\leq x < 2^{-n}$, then there exists a $k\in \N$ for which $x_k+2^{-k} < 2^{-n}$ such that the second if-test is true. Now, since we consider $p\geq 1$, we have that  
$
0-x^{p}| \leq |0 - x| \leq 2^{-n}
$
for $x < 2^{-n}$. 
It follows that $0$ approximates $x^p$ to accuracy $2^{-n}$.

Next consider \ref{it:sign_pow}, and notice that $\mathrm{sign}(x)|x|^p = \max\{0, x\}^p - \max\{0,-x\}^p$. Now, from \Cref{l:comp_func} \ref{it:max_min} we know $\max$ is computable, and from \ref{it:comp_pow} we know that $a\mapsto a^p$ is computable for $a \geq 0$ and $p\geq 1$. The result now follows since the composition of computable functions is computable.    
\end{proof}

The concept of computability is very strict in the sense that it requires the algorithm (Turing or general) to compute an approximation to arbitrary accuracy. However, in most applications we might just need approximations that are accurate to $5, 10$ or $16$ digits.  The concept of the Strong breakdown epsilon allows us to classify which (non-computable) problems that can be approximated and to which accuracy. 

\begin{definition}[{Strong breakdown epsilon \cite{opt_big}}]
    Given a computational problem $\{\Xi, \Omega, \mathcal{M}, \Lambda\}$, the strong breakdown epsilon for this problem is
\[\epsilon_{\mathrm{B}}^{\mathrm{s}} \coloneqq \sup\{\epsilon \geq 0  \,|\, \forall \text{ general algorithms }\Gamma, \exists \, \iota \in \Omega \text{ such that }\mathrm{dist}_{\mathcal{M}}(\Gamma(\iota), \Xi(\iota)) > \epsilon \}.\] 
\end{definition}

The next result from \cite[Prop.\ 8.33]{opt_big} shows that the basis pursuit problem \eqref{eq:BP_argmin} is not computable. However, as the theorem reveals, this does not mean that it is impossible to compute approximations to solutions of the basis pursuit problem, but we cannot compute such solutions to an accuracy smaller than the strong breakdown epsilon. Note that we state a simplified version of \cite[Prop.\ 8.33]{opt_big} below. 
\begin{theorem}[{\cite[Prop.\ 8.33]{opt_big}}]\label{thm:extended_smale}
    Let $\mathcal{U}$ be the solution map from \eqref{eq:BP_argmin} and let the metric on $\mathcal{M}=\R^N$ be induced by the $\|\cdot\|_r$-norm for an arbitrary $r\in[1,\infty]$. Let $K \geq 1$ be an integer.
	Then there exist a set of inputs $\Omega_{K}\subset  \R^{m\times N} \times \R^m$ for the map 
	\begin{equation}\label{eq:Omega_set}
		\mathcal{U}:  \Omega_{K}\rightrightarrows \mathcal{M},
	\end{equation}
as well as sets of $\Delta_1$-information $\hat \Lambda \in \mathcal{L}^1(\Lambda)$ such that for the computational problem $\{\mathcal{U},\Omega_{K},\mathcal{M}, \hat \Lambda\}$ we have 
 $\epsilon_{\mathrm{B}}^{\mathrm{s}} > 10^{-K}$. 
	The statement above is true even when we require the inputs in $\Omega_{K}$ to be well-conditioned and bounded from above and below. In particular, for any input $\iota = (A,y) \in \Omega_{K}$ we have 
	$\mathrm{Cond}(AA^*)\leq 3.2$, $\|y\|_\infty\leq 2$, and $\|A\|_{\max} = 1$.
\end{theorem}

\subsection{Proof of \Cref{thm:grad_flow} part \ref{it:forall}}
Throughout this section let
\begin{equation}\label{eq:def_Omega}
     \Omega = \left\{(A,y) \in \R^{m\times N}\times \R^{m} : \mathrm{rank}(A) = m, y\neq 0\right\}.
\end{equation}
Furthermore, observe that the gradient descent step for \eqref{eq:def_L_disc} with step length $\eta$ can be written as  
\[(\theta^{(k+1)}_{+},\theta^{(k+1)}_{-}) = \mathrm{GD}_{\eta,p}(\theta^{(k)}_{+}, \theta^{(k)}_{-}) \] 
where 
\begin{equation}\label{eq:GD}
    \mathrm{GD}_{\eta,p} (\theta_+,\theta_-) = 
    \begin{bmatrix} 
        \theta_+ - \eta p \left(\mathrm{sign}(\theta_+)\odot |\theta_+|^{p-1}\odot A^{\top}(A(|\theta_+|^p - |\theta_-|^p) - y )\right) \\
        \theta_- + \eta p \left(\mathrm{sign}(\theta_-)\odot |\theta_-|^{p-1}\odot A^{\top}(A(|\theta_+|^p - |\theta_-|^p) - y )\right) \\
    \end{bmatrix}.
\end{equation}
The full algorithm is presented in \Cref{alg:l1_min2}.
\begin{algorithm}
  \caption{The gradient descent algorithm with step length $\eta$ for $L(\theta)$ given by \eqref{eq:def_L_disc}.}
\label{alg:l1_min2}
  \begin{algorithmic}
 \State \textbf{Input:} $(A, y, \alpha, \eta, p, J)$.
    \State $\theta^0_+ \gets \alpha {\bf 1}_N$
    \State $\theta^0_- \gets \alpha {\bf 1}_N$
    \For{$k = 0,1,\dots,J-1 $}
      \State $(\theta^{(k+1)}_+,\theta^{(k+1)}_{-}) \gets \mathrm{GD}_{\eta,p}(\theta^{(k)}_+, \theta^{(k)}_{-})$, where $\mathrm{GD}_{\eta,p}$ is given by \eqref{eq:GD}
    \EndFor
    \State $\hat\psi \gets \left|\theta^{(J)}_+\right|^p-\left|\theta^{(J)}_-\right|^p$
    \State\Return $\hat\psi$
  \end{algorithmic}
\end{algorithm}

Next, we clarify the assumptions of \Cref{thm:grad_flow} \ref{it:forall} and write the theorem in the language of the SCI hierarchy.
\begin{assumption}\label{assump:C_bounded2}
   Let $\Omega$ be given by \eqref{eq:def_Omega}  and assume that there exists a $p\geq 2$ and constants $C_1, C_2 > 0$ such that for any $\alpha > 0$ we have that 
    \[
    \norm{\psi_{\alpha}(\infty)-\mathcal{W}_p(A,y)}_2 \le C_1\norm{y}_2\left(\frac{\alpha^p}{\norm{y}_2}\right)^{C_2} \quad \forall (A,y) \in \Omega.
    \]
\end{assumption}
\begin{proposition}
    Let $\mathcal{U}$ be the solution map from \eqref{eq:BP_argmin}, let $\Omega$ be given by \eqref{eq:def_Omega} and let $\{\mathcal{U}, \Omega\}^{\Delta_1} = \{\tilde{\mathcal{U}}, \tilde \Omega\}$ denote the computational problem with $\Delta_1$-information. Furthermore, suppose that \Cref{assump:C_bounded2} is true.  
    \begin{enumerate}[label=(\roman*)]
        \item\label{it:no_alg1} 
Then for any $\delta > 0$ there exists a general algorithm $\Gamma_{\delta}$ which satisfies 
\begin{equation}\label{eq:imp_bound}
 \inf_{x^* \in \tilde{\mathcal{U}}(\tilde A,\tilde y)}\|\Gamma_{\delta}(\tilde A,\tilde y) - x^*\|_2 \leq \delta, \quad\forall (\tilde A,\tilde y) \in \tilde \Omega. 
\end{equation}
        \item\label{it:no_alg2} The statement in \ref{it:no_alg1} above contradicts \Cref{thm:extended_smale}, which implies that \Cref{assump:C_bounded2} does not hold.
    \end{enumerate}
\end{proposition}

\begin{proof}
    We start with \ref{it:no_alg1}. Let $\delta > 0$ be given, set $\epsilon = \delta/2$ and let $p$, $C_1$ and $C_2$ be the constants from \Cref{assump:C_bounded2}. We will design a general algorithm $\Gamma_{\delta}$ that for any input $(\tilde A,\tilde y) \in \tilde \Omega$ outputs a vector that satisfies \eqref{eq:imp_bound}. This is done as follows.

  Let $U_{\alpha^{\sharp}}$, $U_{\eta^{\sharp}}$ and $L_{ t^{\sharp}}$ denote the constants from \Cref{p:a1_prop1}. From \Cref{lem:comp_const}, we know that the constant $U_{\eta^{\sharp}}$ -- with the choices of $p$, $C_1$, $C_2$ and $\epsilon$ listed above -- is computable on inputs $(\tilde A,\tilde y) \in \tilde \Omega$. Thus, for a given input $(\tilde A,\tilde y) \in \tilde \Omega$ our general algorithm can compute a non-zero lower bound $0 < u_{\alpha^{\sharp}} \leq  \tilde U_{\alpha^{\sharp}}(\tilde A,\tilde y)$ and pick a $\alpha \in (0,u_{\alpha^{\sharp}})$. With this choice of $\alpha$, we can repeat the argument for $L_{ t^{\sharp}}$, and compute an upper bound $l_{ t^{\sharp}} \geq \tilde L_{ t^{\sharp}}(\tilde A,\tilde y)$ for the given input $(\tilde A,\tilde y) \in \tilde \Omega$. Next, our general algorithm computes such an upper bound and then picks a $t \geq l_{ t^{\sharp}}$. Next, for these choices of $t$ and $\alpha$, the algorithm computes a non-zero lower bound $0 < u_{\eta^{\sharp}} \leq \tilde U_{\eta^{\sharp}}(\tilde A,\tilde y)$ (invoking \Cref{lem:comp_const} once more) and then picks a rational $\hat\eta \in (0,u_{\eta^{\sharp}})$. Afterwards, the algorithm picks an integer $J \in (t/\hat\eta, t/\hat \eta+2)$, and finally $\eta \in \left(\frac{t}{J+1}, \frac{t}{J}\right)$.

It now follows from \Cref{l:l1_computable} that \Cref{alg:l1_min2} represents a function that is computable on $\tilde \Omega$ for these choices of constants $\alpha,\eta > 0$, $p \geq 2$ and $J \in \N$. We can, therefore, compute an approximation $\xi$ to $\hat\psi$ which satisfies $\|\xi - \hat\psi\|_2 < \delta/2$. Furthermore, since $\alpha,\eta$ and $J$ all satisfies the conditions of \Cref{p:a1_prop1}, we know that 
\[\inf_{x^* \in \tilde{\mathcal{U}}(\tilde A,\tilde y)} \|\hat\psi - x^{*}\|_2 \le \delta/2.\]
It follows that the computed approximation $\xi$ satisfies
\[\inf_{x^* \in \tilde{\mathcal{U}}(\tilde A,\tilde y)} \|\xi - x^{*}\|_2 \leq   \|\xi -\hat\psi\|_2 +\inf_{x^* \in \tilde{\mathcal{U}}(\tilde A,\tilde y)}\|\hat\psi - x^{*}\|_2 \leq \delta,\]
as desired.

Consider \ref{it:no_alg2} and pick an integer $K\geq 1$. From \Cref{thm:extended_smale} we know that $\mathrm{cond}(AA^*) \leq 3.2$ for each $(A,y) \in \Omega_K$. This implies that $(AA^*)^{-1}$ exists for all  $(A,y) \in \Omega_K$, and thus that $\mathrm{rank}(A) = m$ for each of the matrices in $\Omega_{K}$. Furthermore, it is clear from \cite[{Eq.\ (11.4)}]{opt_big} that $y\neq 0$. From this we conclude that $\Omega_{K}\subset \Omega$, and thus that $\epsilon_{\mathrm{B}}^{\mathrm{s}} > 10^{-K}$ for the computational problem $\{\tilde{\mathcal{U}}, \tilde \Omega\}$. This contradicts the existence of the general algorithm in \ref{it:no_alg1}. We conclude that \Cref{assump:C_bounded2} does not hold.
\end{proof}

\begin{proposition}\label{p:a1_prop1}
  Assume that \Cref{assump:C_bounded2} holds. Let $(A,y)\in \Omega$ and let $ \epsilon > 0$ be given. Furthermore, let $\alpha, t > 0$ and let  
  \begin{align*}
    U_{\alpha^{\sharp}} &= \left(\frac{\epsilon}{3C_1\|y\|_2}\right)^{1/(pC_2)} \|y\|_{2}^{1/p}, && 
    K = \left(\frac{2\sqrt N \norm{y}_2}{\sigma_{\min}(A)}+\alpha^p\right)^\frac{1}{p},  \\
    L_{ t^{\sharp}} &= \max\left\{\frac{\ln(24K^{3p-2})-\ln(\epsilon\alpha^{2p-2})}{2p^2\sigma_{min}^2(A)\alpha^{2p-2}}, 1\right\}, && 
    \hat\epsilon = \min\left\{K, \frac{\epsilon}{3\cdot 2^p p N K^{p-1}}\right\}, \\
    U_{\eta^{\sharp}} &= \frac{\min\{\hat\epsilon,\ \alpha/p\}}{40p\sqrt N \norm{A}_\mathrm{op}^2 K^{2p-1}}\exp\left(-50p^2\sqrt N\norm{A}_\mathrm{op}^2 K^{2p-2}t\right).
  \end{align*}
  Then if $\alpha \in (0,U_{\alpha^{\sharp}})$, $t\geq L_{ t^{\sharp}}$, $\hat \eta \in (0,U_{\eta^{\sharp}})$, $J \in (t / \hat \eta, t / \hat \eta+2)$, $\eta \in \left(\frac{t}{J+1}, \frac{t}{J}\right)$, where $J$ is a positive integer, we have that
    \begin{equation}\label{eq:alg_bound}
     \|\hat \psi - \mathcal{W}(A,y)_p\|_2 \leq \epsilon,
    \end{equation}
    where $\hat\psi$ is the output of \Cref{alg:l1_min2}. 
\end{proposition}
\begin{proof} 
  We start by noticing that since $(A,y) \in \Omega$, we have that $\rank(A) = m$ and $y\neq 0$. This implies that $\norm{A}_\mathrm{op} \geq \sigma_{min}(A) > 0$ and that $\|y\|_2 > 0$, which ensures that $U_{\alpha^{\sharp}}$, $U_{\eta^{\sharp}}$ and $L_{ t^{\sharp}}$ are well-defined. 
   Next, we comment on how the three quantities $U_{\alpha^{\sharp}}, L_{ t^{\sharp}}$ and $U_{\eta^{\sharp}}$ are chosen and the implications of these choices. The constant $U_{\alpha^{\sharp}}$ is chosen so that if $\alpha \in (0,U_{\alpha^{\sharp}})$, then we know from \Cref{assump:C_bounded2} that we have $\norm{\psi_\alpha(\infty)-\mathcal{W}_p(A,y)}_2 \le \frac{\epsilon}{3}$. Furthermore, if $t \geq L_{ t^{\sharp}}$, then we know from \Cref{psi_convergence} that $\norm{\psi_{\alpha}(t)-\psi_{\alpha}(\infty)}_2 \le \frac{\epsilon}{3}$.
    Next, $U_{\eta^{\sharp}}$ is chosen to ensure that if $\eta \in (0,U_{\eta^{\sharp}})$ and $J = \lfloor t / \eta \rfloor$, we have $\|\hat{\psi} - \psi_{\alpha}(t)\|_2 \leq \frac{\epsilon}{3}$ by \Cref{prop:grad_desc}. Note that $\eta$ was chosen so that $\lfloor t / \eta \rfloor = J$. Furthermore, by the choices of $\eta$, $J$ and $\hat \eta$, we have $0 < \eta < \frac{t}{J} < \hat \eta < U_{\eta^{\sharp}})$. Hence, $\|\hat{\psi} - \psi_{\alpha}(t)\|_2 \leq \frac{\epsilon}{3}$.
Finally, using the triangle inequality twice on the derived inequalities gives the bound in \eqref{eq:alg_bound}.
\end{proof}

\begin{lemma}\label{lem:comp_const}
    Let $\epsilon, \alpha, t, C_1,C_2 > 0$ and $p\geq 2$ be given constants and let $\Omega$ be as in \eqref{eq:def_Omega}. For these constants, view the real-valued numbers 
$U_{\alpha^{\sharp}}$, $U_{\eta^{\sharp}}$ and $L_{ t^{\sharp}}$ from \Cref{p:a1_prop1} as functions, mapping $(A,y) \in \Omega$ to $\R$. Then the corresponding computational problems with $\Delta_1$-information $\{\tilde U_{\alpha^{\sharp}}, \tilde\Omega\}$, $\{\tilde U_{\eta^{\sharp}}, \tilde\Omega\}$, $\{\tilde L_{ t^{\sharp}},\tilde \Omega\}$ are computable. 
\end{lemma}
\begin{proof}
We will argue that for each of these functions, there exists a general algorithm which takes $(\tilde A,\tilde y) \in \tilde{\Omega}$ as input, generates the constants listed in the first part of the proposition and executes the corresponding computations with error control. To achieve the desired error control, we will argue that the functions involved are all computable. Thus, as computability is closed under function compositions, the functions $\tilde U_{\alpha^{\sharp}}$, $\tilde U_{\eta^{\sharp}}$, and $\tilde L_{ t^{\sharp}}$ are computable. We refer to \Cref{rem:comp_func} for how to compose oracle Turing machines with general algorithms. 

\begin{claim}\label{sigma_claim}
  The two mappings $A \mapsto \norm{A}_\mathrm{op}$ and $A \mapsto \sigma_{min}(A)$ are computable with $\Delta_1$-information.
\end{claim}
Assume that the claim is true. Then the result follows from the claim, \Cref{l:comp_func} and the discussion above, as each of the three functions $U_{\alpha^{\sharp}}$, $U_{\eta^{\sharp}}$ and $L_{ t^{\sharp}}$ only consists of compositions of computable functions. 

We proceed to prove the claim. From Proposition 17 in \cite{ziegler2004computability}, we know that the eigenvalues of any real-valued symmetric $m\times m$ matrix is computable in a Turing sense. Furthermore, from \Cref{l:comp_func} \ref{it:arit} it is clear that $A\mapsto AA^{\top}$ is computable in a Turing sense. Note that $\norm{A}_\mathrm{op} = \sigma_1(A)$ and $\sigma_{min}(A) = \sigma_m(A)$. Now, since $\sigma_k(A) = \sqrt{\lambda_k(AA^{\top})}$ for any $k\in \{1,\ldots, m\}$, where $\lambda_{k}(AA^{\top})$ denotes the $k$'th eigenvalue of the symmetric matrix $AA^{T}$, the claim follows by composing computable functions.
\end{proof}

\begin{lemma}\label{l:l1_computable}
Let $\alpha, \eta > 0$, $p\geq 2$ and $J \in \N$ be given, and let $\Omega$ be the set in \eqref{eq:def_Omega}. For any $(A,y) \in \Omega$, and any set of $\Delta_1$-information $\hat{\Lambda}\in\mathcal{L}^1(\Lambda)$ for $(A,y)$ the output vector $\hat{\psi}$ of \Cref{alg:l1_min2} is computable. 
\end{lemma}
\begin{proof}
The proof uses the same line of arguments as in the proof of \Cref{lem:comp_const}, and utilizes that the composition of computable functions is computable. The main hurdle is to show that the function $\mathrm{GD}_{\eta,p}$ is computable for inexact input $A,y$ and $\theta$ (the $\theta$ will be inexact after the first iteration). Now, recall from \Cref{lem:pow} \ref{it:sign_pow} that $(z,q) \mapsto \sign(z)|z|^q$ for $q\geq 1$ and $z \in \R$ is computable. Furthermore, from the same lemma we have that $(z,q)\mapsto z^q$ is computable for $z\geq 0$ and $q\geq 1$. It is clear from \Cref{l:comp_func} \ref{it:abs_comp} that also $x\mapsto |x|$ is computable. Thus, as $p-1\geq 1$ for $p\geq 2$, and all the other operations in $\mathrm{GD}_{\eta,p}$ are arithmetic operations, it is clear that $\mathrm{GD}_{\eta,p}$ is computable for inexact input $A,y$ and $\theta$. This implies that $\hat\psi$ is computable.     
\end{proof}

\appendix
\section{Appendix}
\subsection{If \texorpdfstring{$A$}{} has linearly dependent rows}\label{rank_A_l_m}
Throughout this manuscript we have assumed that $\mathrm{rank}(A) =m$. If this is not the case, then we can still analyze the flow of the pair $(A,y)$ via some rather straightforward observations. First, observe that if $A = 0$, then the dynamics studied in this manuscript become trivial. So assume that  $0 < \rank(A) = m' < m$. 

Now, let $QR=A$ be a thin QR decomposition of $A$, where $Q \in \R^{m\times m'}$ is semi-orthogonal and $R \in \R^{m'\times N}$. Next, let $\tilde A = Q\T A$ and $\tilde y = Q\T y$, and 
observe that the gradient flow (and gradient descent) with the data pair $(A,y)$ is identical to the flow of the pair $(\tilde A,\tilde y)$. Indeed, we may rewrite the loss in \eqref{eq:flow1} as follows
{ \allowdisplaybreaks
\begin{align*}
  \frac{1}{2}\norm{A\psi_\theta-y}_2^2 &= \frac{1}{2}\norm{(QQ\T+I-QQ\T)(A\psi_\theta-y)}_2^2\\
                                       &= \frac{1}{2}\norm{QQ\T(A\psi_\theta-y)}_2^2 + \frac{1}{2}\norm{(I-QQ\T)(A\psi_\theta-y)}_2^2\\
                                       &= \frac{1}{2}\norm{Q\T(A\psi_\theta-y)}_2^2 + \frac{1}{2}\norm{(I-QQ\T)y}_2^2\\
                                       &= \frac{1}{2}\norm{\tilde A\psi_\theta-\tilde y)}_2^2 + \frac{1}{2}\norm{(I-QQ\T)y}_2^2.
\end{align*}
}
Now, since the expression $\frac{1}{2}\norm{(I-QQ\T)y}_2^2$ is constant with respect to the parameters $\theta$, it does not change the gradient flow \eqref{eq:grad_flow} or gradient descent \eqref{eq:grad_desc}. We, therefore, conclude that $(A,y)$ and $(\tilde A,\tilde y)$ have identical flows. 
This observation can be used to extend the convergence result in \Cref{thm:grad_flow}\ref{it:dep_bound} to cases where $\rank(A) < m$, in which case we would find that $\lim_{\alpha\to 0}\psi_{\alpha}(\infty) = \mathcal{W}_p(\tilde A, \tilde y)$.

\subsection{Details for \Cref{ex:positive}}\label{s:ex_details}
{ \allowdisplaybreaks
\begin{alignat*}{2}
  A_1 &= 
  \begin{bmatrix*}[r]
  -0.111 & 0.120 & -0.370 & -0.240 & -1.197\\
0.209 & -0.972 & -0.755 & 0.324 & -0.109\\
0.210 & -0.391 & 0.235 & 0.665 & 0.353
  \end{bmatrix*}, \quad
  &&y_1 = 
  \begin{bmatrix*}[r]
  0.973\\
-0.039\\
-0.886
  \end{bmatrix*},\\
  A_2 &= 
  \begin{bmatrix*}[r]
  1 & 1 & 1\\
3 & 0 & 1
  \end{bmatrix*}, \quad
  &&y_2 = 
  \begin{bmatrix*}[r]
  3\\3
  \end{bmatrix*},\\
  A_3 &= 
  \begin{bmatrix*}[r]
  2 & -1 & 0 & 1\\
0 & 3 & 2 & 0
  \end{bmatrix*}, \quad
  &&y_3 = 
  \begin{bmatrix*}[r]
  0\\
6
  \end{bmatrix*}.
\end{alignat*}
}
$A_1$ and $y_1$ were sampled randomly. The rest were chosen to showcase how $\mathcal{W}_{p}$ depends on $p$. Note that $\mathcal{U}(A_2,y_2) = \{(1-\mu,2-2\mu,3\mu)\T \colon \mu \in [0,1]\}$ and $\mathcal{U}(A_3,y_3) = \{(1-\mu,2-2\mu,3\mu,0)\T \colon \mu \in [0,1]\}$. Furthermore, let $\mu_2 = \frac{4-6\sqrt[3]{2}+9\sqrt[3]{4}}{31}$ and $\mu_p = 1/\left(1+\left(\frac{3^{2/p}}{2^{2/p}+1}\right)^{-\frac{p}{p-2}}\right)$ for $p > 2$. Then for $p \ge 2$, $\mathcal{W}_{p}(A_2,y_2) = (1-\mu_p,2-2\mu_p,3\mu_p)\T$ and $\mathcal{W}_{p}(A_3,y_3) = (1-\mu_p,2-2\mu_p,3\mu_p, 0)\T$.

\subsection{Precise statements of well-known theorems}
Since it is challenging to find references in the literature with the exact statements we need, we include the precise statement of some well-known theorems to ease the referencing. Our first theorem is the well-known KKT-conditions, see e.g., \cite[Sec.\ 3.7]{beck2017first}, \cite[Sec.\ 5.5]{boydVand} or \cite[Sec. 5.5--5.7]{andreasson2005introduction}. 
\begin{theorem}[Karush–Kuhn–Tucker conditions with linearity constraint qualification]\label{KKT}
  Let $A \in \R^{m \times N}$, $y \in \R^m$, $B \in \R^{k\times N}$, and $z \in \R^k$. Let $f \colon \R^N \to \R$ be continuously differentiable and convex, when restricted to the set $\{x \in \R^N : Ax = y, Bx \ge z\}$. Here $\ge$ denotes elementwise inequalities. Consider the optimization problem
  \begin{align*}
    \min_{x \in \R^N} f(x) \text{ subject to } Ax = y \text{ and } Bx \ge z.
  \end{align*}

  Then $x^* \in \R^N$ is a minimizer of the above optimization problem if and only if there exist $\lambda \in \R^m, \mu \in [0,\infty)^k$ such that
  \begin{align*}
    \nabla f(x^*) &= A\T\lambda + B\T\mu,\\
    Ax^* &= y,\\
    Bx^* &\ge z,\\
    [Bx^*-z]_i \mu_i &= 0 \quad \forall i \in \{1,\dots,k\}.
  \end{align*}
\end{theorem}
\begin{proof}
  An introduction to and proof of the KKT optimality conditions can be found in sections 5.5-5.7 in \cite{andreasson2005introduction}. Specifically, necessity is proven in \cite[Thm.\ 5.33]{andreasson2005introduction} with linearity constraint qualification \cite[Prop.\ 5.44]{andreasson2005introduction}, while sufficiency holds by \cite[Thm.\ 5.45]{andreasson2005introduction}.
\end{proof}

Our next theorem looks at the global truncation error of the forward Eulers method. The crucial part here is the bound on the step size $\eta$. In e.g, \cite[Thm. 7.4]{hairer1993solving} one state the result for a \enquote{sufficiently small} choice of $\eta$, and in \cite[Thm.\ 1.1]{iserles2009first} one states that Euler's method converges.  
\begin{theorem}[Global truncation error of the forward Euler method]\label{euler_method}
  Let $y \colon [0,\infty) \to \R^N$ satisfy the differential equation
  \begin{align*}
    \frac{d}{dt}y = f(y).
  \end{align*}
  where the differentiable function $f \colon \R^N \to \R^N$ satisfies the following regularity conditions. There are positive constants $C_1,C_2$ and $\delta$ such that for $t \in [0,\infty)$ and $x \in \R^N$, if $\norm{x-y(t)}_2 \le \delta$, then
  \begin{align}
    \norm{f(x)}_2 &\le C_1\label{f_bound_assump}\\
    \norm{\nabla f(x)}_{\mathrm{op}} &\le C_2.\label{f_lip_assump}
  \end{align}
  Fix $T \in [0,\infty)$ and pick $\epsilon > 0$. Let $0 < \eta \le \frac{\min\{\epsilon,\delta\}}{C_1}e^{-C_2T}$ and
  \begin{align*}
    y_0 &= y(0)\\
    y_{k+1} &= y_k+\eta f(y_k)\ \text{ for } k \in \{0,\dots,\lfloor T / \eta \rfloor-1\}.
  \end{align*}
  Then
$
    \norm{y(T)-y_{\lfloor T / \eta \rfloor}}_2 \le \epsilon.
$
\end{theorem}
\begin{proof}
  We first reformulate the first regularity condition on $f$. Let $a, b \in [0,\infty),\ a \le b$, then 
  \begin{align}\label{y_lip}
    \norm{y(b)-y(a)}_2 = \norm{\int_a^b f(y(t))\ dt}_2 &\le \int_a^b \norm{f(y(t))}_2\ dt \le C_1(b-a).
  \end{align}
  The last inequality uses \eqref{f_bound_assump}.

  We will now use induction to prove 
  \begin{align}\label{induction}
    \norm{y(k\eta)-y_k}_2 \le \eta\frac{C_1}{2}\left(e^{C_2\eta k}-1\right) \text{ for } k \in \{0,\dots,\lfloor T / \eta \rfloor\}.
  \end{align}
  The base case $\norm{y(0)-y_0}_2 = 0$ holds. Next, assume \eqref{induction} is true up to some $k\in \{0,\dots,\lfloor T / \eta \rfloor-1\}$, we will prove it holds for $k+1$. First, we bound the local truncation error using \eqref{f_lip_assump} and \eqref{f_bound_assump}
  \begin{align*}
    \norm{y(k\eta+\eta)-y(k\eta)-\eta f(y(k\eta))}_2 &= \norm{\int_{k\eta}^{k\eta+\eta} f(y(t))-f(y(k\eta))\ dt}_2\\
    = \norm{\int_{k\eta}^{k\eta+\eta}\int_{k\eta}^t \nabla f(y(s))f(y(s))\ ds\ dt}_2
    &\le \int_{k\eta}^{k\eta+\eta}\int_{k\eta}^t \norm{\nabla f(y(s))}_\mathrm{op}\norm{f(y(s))}_2\ ds\ dt\\
    &\le \frac{C_1 C_2}{2}\eta^2.
  \end{align*}
  Next, by \eqref{induction} and the assumption on $\eta$, we have $\norm{y(k\eta)-y_k}_2 \le \eta\frac{C_1}{2}\left(e^{C_2\eta k}-1\right) \le \eta\frac{C_1}{2}\left(e^{C_2 T}-1\right) \le \delta$, so we may use the assumption \eqref{f_lip_assump} to get 
  \begin{align}\label{f_lip}
    \begin{split}
  &\norm{f(y(k\eta))-f(y_k)}_2 = \norm{\int_0^1 \nabla f\big(y_k+t\cdot(y(k\eta)-y_k)\big)(y(k\eta)-y_k)\ dt}_2\\
      \le\ &\int_0^1 \norm{\nabla f\big(y_k+t(y(k\eta)-y_k)\big)}_\mathrm{op}\norm{y(k\eta)-y_k}_2 dt \le C_2 \norm{y(k\eta)-y_k}_2.
    \end{split}
  \end{align}
  Finally, we can prove the next induction hypothesis using the triangle inequality, \eqref{f_lip}, and \eqref{induction}
  \begin{align*}
    & \norm{y\big((k+1)\eta\big)-y_{k+1}}_2 = \norm{y(k\eta+\eta)-y_k-\eta f(y_k)}_2\\
    \le\ &\norm{y(k\eta+\eta)-y(k\eta)-\eta f(y(k\eta))}_2+\norm{y(k\eta)-y_k}_2+\norm{\eta f(y(k\eta))-\eta f(y_k)}_2\\
    \le\ &\frac{C_1C_2}{2}\eta^2+\norm{y(k\eta)-y_k}_2+\eta \norm{f(y(k\eta))-f(y_k)}_2
    \le\ \frac{C_1C_2}{2}\eta^2+(1+C_2\eta)\norm{y(k\eta)-y_k}_2\\
    \le\ &\frac{C_1C_2}{2}\eta^2+(1+C_2\eta)\eta\frac{C_1}{2}\left(e^{C_2\eta k}-1\right)
    = \eta\frac{C_1}{2}\big((1+C_2\eta)e^{C_2\eta k}-1\big) \le \eta\frac{C_1}{2}\left(e^{C_2\eta (k+1)}-1\right).
  \end{align*}
  By induction, we hence have $\norm{y(\lfloor T / \eta \rfloor\eta) - y_{\lfloor T / \eta \rfloor}}_2 \le \eta\frac{C_1}{2}\left(e^{C_2\eta \lfloor T / \eta \rfloor}-1\right) \le \eta C_1\left(e^{C_2T}-1\right)$. Using \eqref{y_lip} to bound the last part of the path $y$, we get the desired bound
  \begin{align*}
    &\norm{y(T)-y_{\lfloor T / \eta \rfloor}}_2 \le \norm{y(T)-y(\lfloor T / \eta \rfloor\eta)}_2 + \norm{y(\lfloor T / \eta \rfloor\eta)-y_{\lfloor T / \eta \rfloor}}_2\\
    \le &C_1\big(T-\lfloor T / \eta \rfloor\eta\big)+\eta C_1\left(e^{C_2T}-1\right) \le C_1\eta + \eta C_1\left(e^{C_2T}-1\right) \le \epsilon.
  \end{align*}
\end{proof}

{\small
\bibliographystyle{abbrv}
\bibliography{references}

\begin{thebibliography}{10}

\bibitem{adcock2023restarts}
B.~Adcock, M.~J. Colbrook, and M.~Neyra-Nesterenko.
\newblock Restarts subject to approximate sharpness: A parameter-free and
  optimal scheme for first-order methods.
\newblock {\em arXiv preprint arXiv:2301.02268}, 2023.

\bibitem{CSBook}
B.~Adcock and A.~C. Hansen.
\newblock {\em Compressive Imaging: Structure, Sampling, Learning}.
\newblock Cambridge University Press, 2021.

\bibitem{andreasson2005introduction}
N.~Andr{\'e}asson, A.~Evgrafov, and M.~Patriksson.
\newblock {\em An Introduction to Continuous Optimization}.
\newblock Professional Publishing Svc., 2005.

\bibitem{paradox22}
V.~Antun, M.~J. Colbrook, and A.~C. Hansen.
\newblock Proving existence is not enough: Mathematical paradoxes unravel the
  limits of neural networks in artificial intelligence.
\newblock {\em SIAM News}, 55(04):1--4, May 2022.

\bibitem{Arora_PTAS}
S.~Arora.
\newblock Polynomial time approximation schemes for {Euclidean} traveling
  salesman and other geometric problems.
\newblock {\em J.\ ACM}, 45(5):753–782, sep 1998.

\bibitem{arora2009computational}
S.~Arora and B.~Barak.
\newblock {\em Computational complexity: a modern approach}.
\newblock Cambridge University Press, 2009.

\bibitem{arora2018optimization}
S.~Arora, N.~Cohen, and E.~Hazan.
\newblock On the optimization of deep networks: Implicit acceleration by
  overparameterization.
\newblock In {\em International Conference on Machine Learning}, pages
  244--253, 2018.

\bibitem{arora2019implicit}
S.~Arora, N.~Cohen, W.~Hu, and Y.~Luo.
\newblock Implicit regularization in deep matrix factorization.
\newblock {\em Advances in Neural Information Processing Systems}, 32, 2019.

\bibitem{arora1998proof}
S.~Arora, C.~Lund, R.~Motwani, M.~Sudan, and M.~Szegedy.
\newblock Proof verification and the hardness of approximation problems.
\newblock {\em J.\ ACM}, 45(3):501--555, 1998.

\bibitem{arora1998probabilistic}
S.~Arora and S.~Safra.
\newblock Probabilistic checking of proofs: A new characterization of {NP}.
\newblock {\em J.\ ACM}, 45(1):70--122, 1998.

\bibitem{azulay2021implicit}
S.~Azulay, E.~Moroshko, M.~S. Nacson, B.~E. Woodworth, N.~Srebro, A.~Globerson,
  and D.~Soudry.
\newblock On the implicit bias of initialization shape: Beyond infinitesimal
  mirror descent.
\newblock In {\em International Conference on Machine Learning}, pages
  468--477, 2021.

\bibitem{bah2022learning}
B.~Bah, H.~Rauhut, U.~Terstiege, and M.~Westdickenberg.
\newblock Learning deep linear neural networks: {Riemannian} gradient flows and
  convergence to global minimizers.
\newblock {\em Inform. Inference: J. IMA}, 11(1):307--353, 2022.

\bibitem{opt_big}
A.~Bastounis, A.~C. Hansen, and V.~{Vla\v{c}i\'{c}}.
\newblock The extended {S}male's 9th problem -- {O}n computational barriers and
  paradoxes in estimation, regularisation, computer-assisted proofs and
  learning.
\newblock {\em arXiv:2110.15734}, 2021.

\bibitem{beck2017first}
A.~Beck.
\newblock {\em First-order methods in optimization}.
\newblock SIAM, 2017.

\bibitem{bellare1998free}
M.~Bellare, O.~Goldreich, and M.~Sudan.
\newblock Free bits, {PCPs}, and nonapproximability -- towards tight results.
\newblock {\em SIAM J.\ Computing}, 27(3):804--915, 1998.

\bibitem{SCI}
J.~Ben-Artzi, M.~J. Colbrook, A.~C. Hansen, O.~Nevanlinna, and M.~Seidel.
\newblock Computing spectra -- {O}n the solvability complexity index hierarchy
  and towers of algorithms.
\newblock {\em arXiv:1508.03280}, 2020.

\bibitem{CRAS}
J.~Ben-Artzi, A.~C. Hansen, O.~Nevanlinna, and M.~Seidel.
\newblock New barriers in complexity theory: On the solvability complexity
  index and the towers of algorithms.
\newblock {\em Comptes Rendus Mathematique}, 353(10):931 -- 936, 2015.

\bibitem{ben2022computing}
J.~Ben-Artzi, M.~Marletta, and F.~R{\"o}sler.
\newblock Computing the sound of the sea in a seashell.
\newblock {\em Found. Comput. Math.}, 22(3):697--731, 2022.

\bibitem{Ben_Artzi2022}
J.~Ben-Artzi, M.~Marletta, and F.~R\"osler.
\newblock Computing scattering resonances.
\newblock {\em J. Eur. Math. Soc.}, (to appear).

\bibitem{Nemirovski_robust}
A.~Ben-Tal, L.~El~Ghaoui, and A.~Nemirovski.
\newblock {\em Robust Optimization}.
\newblock Princeton Series in Applied Mathematics. Princeton University Press,
  October 2009.

\bibitem{Nemirovski_robust2}
A.~Ben-Tal and A.~Nemirovski.
\newblock Robust solutions of linear programming problems contaminated with
  uncertain data.
\newblock {\em Mathematical Programming}, 88(3):411--424, 2000.

\bibitem{BCSS}
L.~Blum, F.~Cucker, M.~Shub, and S.~Smale.
\newblock {\em Complexity and Real Computation}.
\newblock Springer-Verlag New York, Inc., 1998.

\bibitem{bolte2022iterates}
J.~Bolte, C.~W. Combettes, and E.~Pauwels.
\newblock The iterates of the {Frank-Wolfe} algorithm may not converge.
\newblock {\em arXiv preprint arXiv:2202.08711}, 2022.

\bibitem{bolte2022curiosities}
J.~Bolte and E.~Pauwels.
\newblock Curiosities and counterexamples in smooth convex optimization.
\newblock {\em Mathematical Programming}, 195(1-2):553--603, 2022.

\bibitem{boydVand}
S.~P. Boyd and L.~Vandenberghe.
\newblock {\em Convex optimization}.
\newblock Cambridge university press, 2004.

\bibitem{Chambolle_2011}
A.~Chambolle and T.~Pock.
\newblock A first-order primal-dual algorithm for convex problems with
  applications to imaging.
\newblock {\em J. Math. Imaging Vis.}, 40(1):120--145, May 2011.

\bibitem{chambolle_pock_2016}
A.~Chambolle and T.~Pock.
\newblock An introduction to continuous optimization for imaging.
\newblock {\em Acta Numerica}, 25:161?319, 2016.

\bibitem{chou2020gradient}
H.-H. Chou, C.~Gieshoff, J.~Maly, and H.~Rauhut.
\newblock Gradient descent for deep matrix factorization: Dynamics and implicit
  bias towards low rank.
\newblock {\em arXiv:2011.13772}, 2020.

\bibitem{chou2021more}
H.-H. Chou, J.~Maly, and H.~Rauhut.
\newblock More is less: Inducing sparsity via overparameterization.
\newblock {\em arXiv:2112.11027}, 2021.

\bibitem{Colbrook_2019}
M.~Colbrook.
\newblock On the computation of geometric features of spectra of linear
  operators on hilbert spaces.
\newblock {\em Found. Comp. Math.}, (to appear).

\bibitem{colbrook2019foundations}
M.~Colbrook and A.~C. Hansen.
\newblock The foundations of spectral computations via the solvability
  complexity index hierarchy.
\newblock {\em J. Eur. Math. Soc.}, (to appear).

\bibitem{colbrook2021computing}
M.~Colbrook, A.~Horning, and A.~Townsend.
\newblock Computing spectral measures of self-adjoint operators.
\newblock {\em SIAM Rev.}, 63(3):489--524, 2021.

\bibitem{colbrook2022warpd}
M.~J. Colbrook.
\newblock {WARPd}: A linearly convergent first-order primal-dual algorithm for
  inverse problems with approximate sharpness conditions.
\newblock {\em SIAM Journal on Imaging Sciences}, 15(3):1539--1575, 2022.

\bibitem{comp_stable_NN22}
M.~J. Colbrook, V.~Antun, and A.~C. Hansen.
\newblock The difficulty of computing stable and accurate neural networks: On
  the barriers of deep learning and smale's 18th problem.
\newblock {\em Proc.\ Natl.\ Acad.\ Sci.\ USA}, 119(12):e2107151119, 2022.

\bibitem{doyle1989solving}
P.~Doyle and C.~T. McMullen.
\newblock Solving the quintic by iteration.
\newblock {\em Acta Math.}, 163, 1989.

\bibitem{DRnonvacuous17}
G.~K. Dziugaite and D.~M. Roy.
\newblock Computing nonvacuous generalization bounds for deep (stochastic)
  neural networks with many more parameters than training data.
\newblock In {\em Proceedings of the 33rd Annual Conference on Uncertainty in
  Artificial Intelligence (UAI)}, 2017.

\bibitem{even2023s}
M.~Even, S.~Pesme, S.~Gunasekar, and N.~Flammarion.
\newblock {(S) GD} over diagonal linear networks: Implicit regularisation,
  large stepsizes and edge of stability.
\newblock {\em arXiv:2302.08982}, 2023.

\bibitem{feige1996interactive}
U.~Feige, S.~Goldwasser, L.~Lov{\'a}sz, S.~Safra, and M.~Szegedy.
\newblock Interactive proofs and the hardness of approximating cliques.
\newblock {\em J.\ ACM}, 43(2):268--292, 1996.

\bibitem{Mario}
M.~A.~T. Figueiredo, R.~D. Nowak, and S.~J. Wright.
\newblock Gradient projection for sparse reconstruction: Application to
  compressed sensing and other inverse problems.
\newblock {\em IEEE Journal of Selected Topics in Signal Processing},
  1(4):586--597, 2007.

\bibitem{Foucart13}
S.~Foucart and H.~Rauhut.
\newblock {\em A Mathematical Introduction to Compressive Sensing}.
\newblock Birkh\"{a}user Basel, 2013.

\bibitem{gazdag2022generalised}
L.~E. Gazdag and A.~C. Hansen.
\newblock Generalised hardness of approximation and the {SCI} hierarchy--{On}
  determining the boundaries of training algorithms in {AI}.
\newblock {\em arXiv:2209.06715}, 2022.

\bibitem{geyer2020low}
K.~Geyer, A.~Kyrillidis, and A.~Kalev.
\newblock Low-rank regularization and solution uniqueness in over-parameterized
  matrix sensing.
\newblock In {\em International Conference on Artificial Intelligence and
  Statistics}, pages 930--940, 2020.

\bibitem{gissin2019implicit}
D.~Gissin, S.~Shalev-Shwartz, and A.~Daniely.
\newblock The implicit bias of depth: How incremental learning drives
  generalization.
\newblock {\em International Conference on Learning Representations}, 2020.

\bibitem{gunasekar2018implicit}
S.~Gunasekar, J.~D. Lee, D.~Soudry, and N.~Srebro.
\newblock Implicit bias of gradient descent on linear convolutional networks.
\newblock {\em Advances in Neural Information Processing Systems}, 31, 2018.

\bibitem{gunasekar2017implicit}
S.~Gunasekar, B.~E. Woodworth, S.~Bhojanapalli, B.~Neyshabur, and N.~Srebro.
\newblock Implicit regularization in matrix factorization.
\newblock {\em Advances in Neural Information Processing Systems}, 30, 2017.

\bibitem{hairer1993solving}
E.~Hairer, S.~P. N{\o}rsett, and G.~Wanner.
\newblock {\em Solving ordinary differential equations. 1, Nonstiff problems}.
\newblock Springer-Vlg, 1993.

\bibitem{Hansen_JAMS}
A.~C. Hansen.
\newblock On the solvability complexity index, the {$n$}-pseudospectrum and
  approximations of spectra of operators.
\newblock {\em J. Amer. Math. Soc.}, 24(1):81--124, 2011.

\bibitem{Hansen2016ComplexityII}
A.~C. Hansen and O.~Nevanlinna.
\newblock Complexity issues in computing spectra, pseudospectra and resolvents.
\newblock {\em Banach Center Publications}, 112:171--194, 2016.

\bibitem{johan1999clique}
J.~H{\aa}stad.
\newblock Clique is hard to approximate within $n^{1- \varepsilon}$.
\newblock {\em Acta Math.}, 182(1):105--142, 1999.

\bibitem{haastad2001some}
J.~H{\aa}stad.
\newblock Some optimal inapproximability results.
\newblock {\em J.\ ACM}, 48(4):798--859, 2001.

\bibitem{hoff2017lasso}
P.~D. Hoff.
\newblock Lasso, fractional norm and structured sparse estimation using a
  {Hadamard} product parametrization.
\newblock {\em Comput.\ Stat.\ Data Anal.}, 115:186--198, 2017.

\bibitem{iserles2009first}
A.~Iserles.
\newblock {\em A first course in the numerical analysis of differential
  equations}.
\newblock Number~44. Cambridge university press, 2009.

\bibitem{Juditsky_2012}
A.~Juditsky, F.~Kilin{\c{c}}{-}Karzan, A.~Nemirovski, and B.~Polyak.
\newblock {Accuracy guaranties for $\ell_{1}$ recovery of block-sparse
  signals}.
\newblock {\em The Annals of Statistics}, 40(6):3077 -- 3107, 2012.

\bibitem{Juditsky_2011}
A.~B. Juditsky, F.~Kilin{\c{c}}{-}Karzan, and A.~Nemirovski.
\newblock Verifiable conditions of $\ell_{1}$-recovery for sparse signals with
  sign restrictions.
\newblock {\em Math. Program.}, 127(1):89--122, 2011.

\bibitem{kaplan2020scaling}
J.~Kaplan, S.~McCandlish, T.~Henighan, T.~B. Brown, B.~Chess, R.~Child,
  S.~Gray, A.~Radford, J.~Wu, and D.~Amodei.
\newblock Scaling laws for neural language models.
\newblock {\em arXiv:2001.08361}, 2020.

\bibitem{Khot}
S.~Khot.
\newblock On the power of unique 2-prover 1-round games.
\newblock In {\em Proceedings of the Thiry-Fourth Annual ACM Symposium on
  Theory of Computing}, STOC '02, page 767–775, New York, NY, USA, 2002.
  Association for Computing Machinery.

\bibitem{ko1991computational}
K.-I. Ko.
\newblock {\em Computational complexity of real functions}.
\newblock Birkh\"{a}user, 1991.

\bibitem{lecun2002efficient}
Y.~LeCun, L.~Bottou, G.~B. Orr, and K.-R. M{\"u}ller.
\newblock Efficient backprop.
\newblock In {\em Neural networks: Tricks of the trade}, pages 9--50. Springer,
  2002.

\bibitem{li2021implicit}
J.~Li, T.~Nguyen, C.~Hegde, and K.~W. Wong.
\newblock Implicit sparse regularization: The impact of depth and early
  stopping.
\newblock {\em Advances in Neural Information Processing Systems},
  34:28298--28309, 2021.

\bibitem{mcmullen1985families}
C.~T. McMullen.
\newblock {\em Families of rational maps and iterative root-finding algorithms
  (dynamics, complex analysis, newton's method)}, volume 125.
\newblock Harvard University, 1985.

\bibitem{Mitchell99}
J.~S.~B. Mitchell.
\newblock Guillotine subdivisions approximate polygonal subdivisions: A simple
  polynomial-time approximation scheme for geometric {TSP, k-MST,} and related
  problems.
\newblock {\em SIAM J.\ Computing}, 28(4):1298--1309, 1999.

\bibitem{moroshko2020implicit}
E.~Moroshko, B.~E. Woodworth, S.~Gunasekar, J.~D. Lee, N.~Srebro, and
  D.~Soudry.
\newblock Implicit bias in deep linear classification: Initialization scale vs
  training accuracy.
\newblock {\em Advances in neural information processing systems},
  33:22182--22193, 2020.

\bibitem{NemirovskiLRob}
A.~Nemirovski.
\newblock { Lectures on Robust Convex Optimization}.
\newblock Available online at \url{https://www2.isye.gatech.edu/~nemirovs/},
  2009.

\bibitem{Nemirovski_NPhard_Stable}
A.~Nemirovskii.
\newblock Several {NP}-hard problems arising in robust stability analysis.
\newblock {\em Mathematics of Control, Signals and Systems}, 6(2):99--105,
  1993.

\bibitem{nesterov2018lectures}
Y.~Nesterov.
\newblock {\em Lectures on convex optimization}, volume 137.
\newblock Springer, 2018.

\bibitem{Nesterov_Nemirovski_Acta}
Y.~E. Nesterov and A.~Nemirovski.
\newblock On first-order algorithms for l1/nuclear norm minimization.
\newblock {\em Acta Numer.}, 22:509--575, 2013.

\bibitem{neyra2023nestanets}
M.~Neyra-Nesterenko and B.~Adcock.
\newblock {NESTANets}: Stable, accurate and efficient neural networks for
  analysis-sparse inverse problems.
\newblock {\em Sampling Theory, Signal Processing, and Data Analysis}, 21(1):4,
  2023.

\bibitem{neyshabur2017geometry}
B.~Neyshabur, R.~Tomioka, R.~Salakhutdinov, and N.~Srebro.
\newblock Geometry of optimization and implicit regularization in deep
  learning.
\newblock {\em arXiv:1705.03071}, 2017.

\bibitem{neyshabur2014search}
B.~Neyshabur, R.~Tomioka, and N.~Srebro.
\newblock In search of the real inductive bias: On the role of implicit
  regularization in deep learning.
\newblock {\em International Conference on Learning Representations}, 2015.

\bibitem{pesme2020online}
S.~Pesme and N.~Flammarion.
\newblock Online robust regression via {SGD} on the $l_1$ loss.
\newblock {\em Advances in Neural Information Processing Systems},
  33:2540--2552, 2020.

\bibitem{pesme2023saddle}
S.~Pesme and N.~Flammarion.
\newblock Saddle-to-saddle dynamics in diagonal linear networks.
\newblock {\em arXiv:2304.00488}, 2023.

\bibitem{pesme2021implicit}
S.~Pesme, L.~Pillaud-Vivien, and N.~Flammarion.
\newblock Implicit bias of {SGD} for diagonal linear networks: a provable
  benefit of stochasticity.
\newblock {\em Advances in Neural Information Processing Systems},
  34:29218--29230, 2021.

\bibitem{Poon_NeurIPS21}
C.~Poon and G.~Peyr\'{e}.
\newblock Smooth bilevel programming for sparse regularization.
\newblock In M.~Ranzato, A.~Beygelzimer, Y.~Dauphin, P.~Liang, and J.~W.
  Vaughan, editors, {\em Advances in Neural Information Processing Systems},
  volume~34, pages 1543--1555. Curran Associates, Inc., 2021.

\bibitem{poon2023smooth}
C.~Poon and G.~Peyr{\'e}.
\newblock Smooth over-parameterized solvers for non-smooth structured
  optimization.
\newblock {\em Math.\ Program.}, pages 1--56, 2023.

\bibitem{razin2020implicit}
N.~Razin and N.~Cohen.
\newblock Implicit regularization in deep learning may not be explainable by
  norms.
\newblock {\em Advances in neural information processing systems},
  33:21174--21187, 2020.

\bibitem{razin2021implicit}
N.~Razin, A.~Maman, and N.~Cohen.
\newblock Implicit regularization in tensor factorization.
\newblock In {\em International Conference on Machine Learning}, pages
  8913--8924, 2021.

\bibitem{razin2022implicit}
N.~Razin, A.~Maman, and N.~Cohen.
\newblock Implicit regularization in hierarchical tensor factorization and deep
  convolutional neural networks.
\newblock {\em arXiv:2201.11729}, 2022.

\bibitem{smale1981fundamental}
S.~Smale.
\newblock The fundamental theorem of algebra and complexity theory.
\newblock {\em Am.\ Math.\ Soc.\ Bull.}, 4:1--36, 1981.

\bibitem{smale1997complexity}
S.~Smale.
\newblock Complexity theory and numerical analysis.
\newblock {\em Acta Numer.}, 6:523--551, 1997.

\bibitem{soudry2018implicit}
D.~Soudry, E.~Hoffer, M.~S. Nacson, S.~Gunasekar, and N.~Srebro.
\newblock The implicit bias of gradient descent on separable data.
\newblock {\em The Journal of Machine Learning Research}, 19(1):2822--2878,
  2018.

\bibitem{Stewart89}
G.~V. Stewart.
\newblock On scaled protections and pseudoinvcrses.
\newblock {\em Linear Algebra Appl.}, 112:189--193, 1989.

\bibitem{stoger2021small}
D.~St{\"o}ger and M.~Soltanolkotabi.
\newblock Small random initialization is akin to spectral learning:
  Optimization and generalization guarantees for overparameterized low-rank
  matrix reconstruction.
\newblock {\em Advances in Neural Information Processing Systems},
  34:23831--23843, 2021.

\bibitem{sudan2009probabilistically}
M.~Sudan.
\newblock Probabilistically checkable proofs.
\newblock {\em Commun.\ ACM}, 52(3):76--84, 2009.

\bibitem{tan2019efficientnet}
M.~Tan and Q.~Le.
\newblock {EfficientNet}: Rethinking model scaling for convolutional neural
  networks.
\newblock In {\em International conference on machine learning}, pages
  6105--6114, 2019.

\bibitem{todd1990dantzig}
M.~J. Todd.
\newblock A {Dantzig-Wolfe-like} variant of {Karmarkar's} interior-point linear
  programming algorithm.
\newblock {\em Oper.\ Res.}, 38(6):1006--1018, 1990.

\bibitem{turing1937computable}
A.~Turing.
\newblock On {C}omputable {N}umbers, with an {A}pplication to the
  {E}ntscheidungsproblem.
\newblock {\em Proc. London Math. Soc. (2)}, 42(3):230--265, 1936.

\bibitem{vaskevicius2019implicit}
T.~Vaskevicius, V.~Kanade, and P.~Rebeschini.
\newblock Implicit regularization for optimal sparse recovery.
\newblock {\em Advances in Neural Information Processing Systems}, 32, 2019.

\bibitem{vavasis1994stable}
S.~A. Vavasis.
\newblock Stable numerical algorithms for equilibrium systems.
\newblock {\em SIAM J.\ Matrix Anal.\ Appl.}, 15(4):1108--1131, 1994.

\bibitem{vavasis1996stable}
S.~A. Vavasis.
\newblock Stable finite elements for problems with wild coefficients.
\newblock {\em SIAM J.\ Numer.\ Anal.}, 33(3):890--916, 1996.

\bibitem{vavasis1996primal}
S.~A. Vavasis and Y.~Ye.
\newblock A primal-dual interior point method whose running time depends only
  on the constraint matrix.
\newblock {\em Math.\ Program.}, 74(1):79--120, 1996.

\bibitem{webb2021spectra}
M.~Webb and S.~Olver.
\newblock Spectra of {Jacobi} operators via connection coefficient matrices.
\newblock {\em Commun. Math. Phys.}, 382(2):657--707, 2021.

\bibitem{weihrauch2000computable}
K.~Weihrauch.
\newblock {\em Computable analysis: An introduction}.
\newblock Springer, 2000.

\bibitem{woodworth2020kernel}
B.~Woodworth, S.~Gunasekar, J.~D. Lee, E.~Moroshko, P.~Savarese, I.~Golan,
  D.~Soudry, and N.~Srebro.
\newblock Kernel and rich regimes in overparametrized models.
\newblock In {\em Conference on Learning Theory}, pages 3635--3673, 2020.

\bibitem{Mario2}
S.~J. Wright, R.~D. Nowak, and M.~A.~T. Figueiredo.
\newblock Sparse reconstruction by separable approximation.
\newblock {\em IEEE Transactions on Signal Processing}, 57(7):2479--2493, 2009.

\bibitem{zhang2017understanding}
C.~Zhang, S.~Bengio, M.~Hardt, B.~Recht, and O.~Vinyals.
\newblock Understanding deep learning requires rethinking generalization.
\newblock In {\em International Conference on Learning Representations}, 2017.

\bibitem{Zhao_IR22}
P.~Zhao, Y.~Yang, and Q.-C. He.
\newblock {High-dimensional linear regression via implicit regularization}.
\newblock {\em Biometrika}, 109(4):1033--1046, 02 2022.

\bibitem{ziegler2004computability}
M.~Ziegler and V.~Brattka.
\newblock Computability in linear algebra.
\newblock {\em Theoretical Computer Sci.}, 326(1-3):187--211, 2004.

\end{thebibliography}
}

\end{document}